\documentclass{article}

\usepackage{microtype}
\usepackage{graphicx}
\usepackage{appendix}
\usepackage{minitoc}
\usepackage{multirow}
\usepackage{subcaption}
\usepackage{booktabs} 
\usepackage{dsfont}
\usepackage[ruled,linesnumbered]{algorithm2e}
\usepackage[figuresright]{rotating}
\usepackage[T1]{fontenc}

\usepackage{amsmath,amsfonts,bm}

\def\eqref#1{equation~\ref{#1}}

\def\1{\bm{1}}

\def\vs{{\bm{s}}}

\DeclareMathAlphabet{\mathsfit}{\encodingdefault}{\sfdefault}{m}{sl}
\SetMathAlphabet{\mathsfit}{bold}{\encodingdefault}{\sfdefault}{bx}{n}

\DeclareMathOperator*{\argmax}{arg\,max}

\allowdisplaybreaks[4]
\pdfoutput=1

\usepackage{hyperref}
\usepackage[hyphenbreaks]{breakurl}

\usepackage[accepted]{icml2024}

\usepackage{amsmath}
\usepackage{amssymb}
\usepackage{mathtools}
\usepackage{amsthm}

\allowdisplaybreaks[3]
\usepackage[capitalize,noabbrev]{cleveref}

\def\eg{e.g.\xspace}

\def\ie{i.e.\xspace}

\def\vs{vs.\xspace}

\def\wrt{w.r.t.\xspace}

\usepackage{xurl}

\theoremstyle{plain}
\newtheorem{theorem}{Theorem}[section]

\newtheorem{lemma}[theorem]{Lemma}

\theoremstyle{definition}
\newtheorem{definition}[theorem]{Definition}
\newtheorem{assumption}[theorem]{Assumption}
\theoremstyle{remark}

\usepackage[textsize=tiny]{todonotes}
\icmltitlerunning{Effects of Exponential Gaussian Distribution on (Double Sampling) Randomized Smoothing}

\begin{document}
\twocolumn[
\icmltitle{	
Effects of Exponential Gaussian Distribution on \\
(Double Sampling) Randomized Smoothing}

\icmlsetsymbol{corresponding}{\dag}

\begin{icmlauthorlist}
\icmlauthor{Youwei Shu}{thu}
\icmlauthor{Xi Xiao}{thu,corresponding}
\icmlauthor{Derui Wang}{data61}
\icmlauthor{Yuxin Cao}{thu}
\icmlauthor{Siji Chen}{thu}
\icmlauthor{Minhui Xue}{data61}
\icmlauthor{Linyi Li}{uiuc,sfu}
\icmlauthor{Bo Li}{uiuc,uchi}
\end{icmlauthorlist}

\icmlaffiliation{thu}{Shenzhen International Graduate School, Tsinghua University}
\icmlaffiliation{data61}{CSIRO’s Data61}
\icmlaffiliation{uiuc}{University of Illinois Urbana-Champaign}
\icmlaffiliation{uchi}{University of Chicago}
\icmlaffiliation{sfu}{Simon Fraser University}

\icmlcorrespondingauthor{Youwei Shu}{shuyw21@mails.tsinghua.edu.com}
\icmlcorrespondingauthor{Xi Xiao}{xiaox@sz.tsinghua.edu.cn}

\icmlkeywords{randomized smoothing, certified robustness, statistical distribution, curse of dimensionality}

\vskip 0.3in]

\printAffiliationsAndNotice{\icmlCorrespondingAuthor} 

\doparttoc 
\faketableofcontents 
\thispagestyle{empty}

\begin{abstract}
Randomized Smoothing (RS) is currently a scalable certified defense method providing robustness certification against adversarial examples. 
Although significant progress has been achieved in providing defenses against $\ell_p$ adversaries,
the interaction between the smoothing distribution and the robustness certification still remains vague.
In this work, we comprehensively study the effect of two families of distributions, named Exponential Standard Gaussian (ESG) and Exponential General Gaussian (EGG) distributions, on Randomized Smoothing and Double Sampling Randomized Smoothing (DSRS). 
We derive an analytic formula for ESG's certified radius, which converges to the origin formula of RS as the dimension $d$ increases.  
Additionally, we prove that EGG can provide tighter constant factors than DSRS in providing $\Omega(\sqrt{d})$ lower bounds of $\ell_2$ certified radius, and thus further addresses the curse of dimensionality in RS. 
Our experiments on real-world datasets confirm our theoretical analysis of the ESG distributions, that they provide almost the same certification under different exponents $\eta$ for both RS and DSRS. In addition, EGG brings a significant improvement to the DSRS certification, but the mechanism can be different when the classifier properties are different.
Compared to the primitive DSRS, the increase in certified accuracy provided by EGG is prominent, up to 6.4\% on ImageNet. Our code is available at \url{https://github.com/tdano1/eg-on-smoothing}.
\end{abstract}

\section{Introduction}

Deep neural networks (DNNs) have achieved great success in various applications. 
However, DNNs are susceptible to adversarial perturbations in their inputs. 
To tackle the problem of adversarial attacks, a series of empirical defenses, such as adversarial training~\citep{goodfellow2014, kurakin2016, madry2017}, have been proposed. 
Nevertheless, this strategy quickly evolved into an arms race because no matter how robust the DNNs are, well-crafted adversarial examples are capable of bypassing the defenses~\citep{carlini2017, athalye2018, uesato2018}.
Recently, researchers proposed and developed certified defenses~\citep{wong2018a, wong2018b, raghunathan2018}, a series of methodologies that can output the bounds of perturbed inputs, and provide provable robustness for classifiers. Aligned with these exact certified defense methods, randomized smoothing (RS)~\citep{lecuyer2019, li2019, cohen2019} appears as a certifying tool based on probability, and gains in popularity since it can provide scalable robustness certifications for black-box functions.
~\citet{cohen2019} first introduced the Neyman-Pearson (NP) lemma into the certification, which provided tight $\ell_2$ certified radii for linear classifiers. 
Later, a series of attempts further extended the certification process of RS using functional optimization frameworks~\citep{zhang2020, dvijothamframework}.

However, current research on the distributions for randomized smoothing is far from sufficient. After \citet{zhang2020}, \citet{yang2020} and \citet{li2022}, this branch of study seems to be dormant. Intuitively, investigating the interrelationship of distribution and the RS method is not only beneficial for understanding and solving the limit of RS, but also liable to excavate the mathematical rules hidden deep in RS. In this work, we systematically study the ESG and EGG distributions in the RS framework, providing a detailed theoretical analysis for both the new distribution families. In addition, we conduct extensive experiments on real-world datasets, testify our theory on ESG, and complete the analysis for EGG. Overall, we conclude that the ESG distributions share almost the same certification for RS, and EGG's certification can be significantly improved using different exponents $\eta$, showing the better potential to lessen the curse of dimensionality than the SOTA solution \citep{li2022}.

\begin{figure}[t!]
	\centering
	\begin{subfigure}{0.49\linewidth}
		\centering
		\includegraphics[height=1.0\linewidth, width=1.0\linewidth]{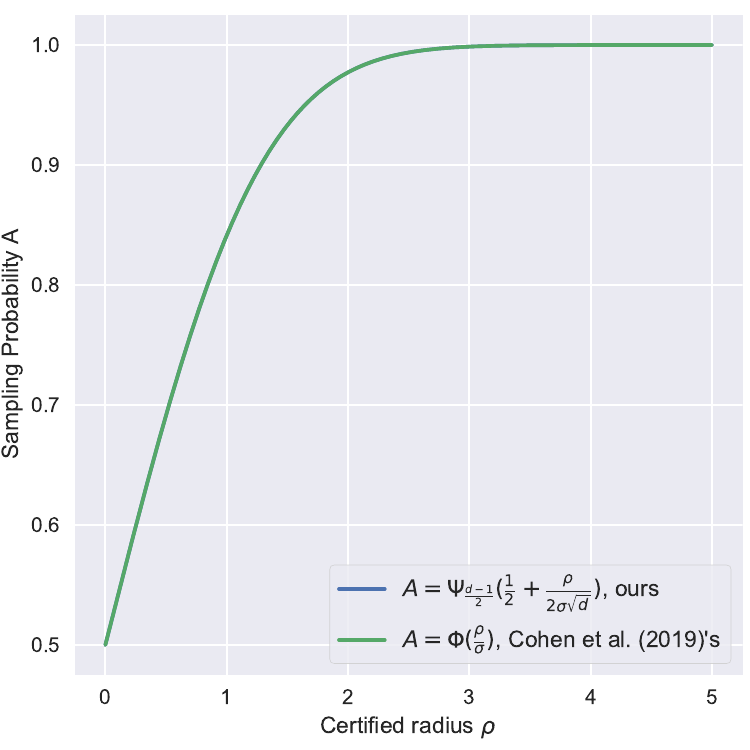}
	\end{subfigure}
	\centering
	\begin{subfigure}{0.49\linewidth}
		\centering
		\includegraphics[height=1.0\linewidth, width=1.0\linewidth]{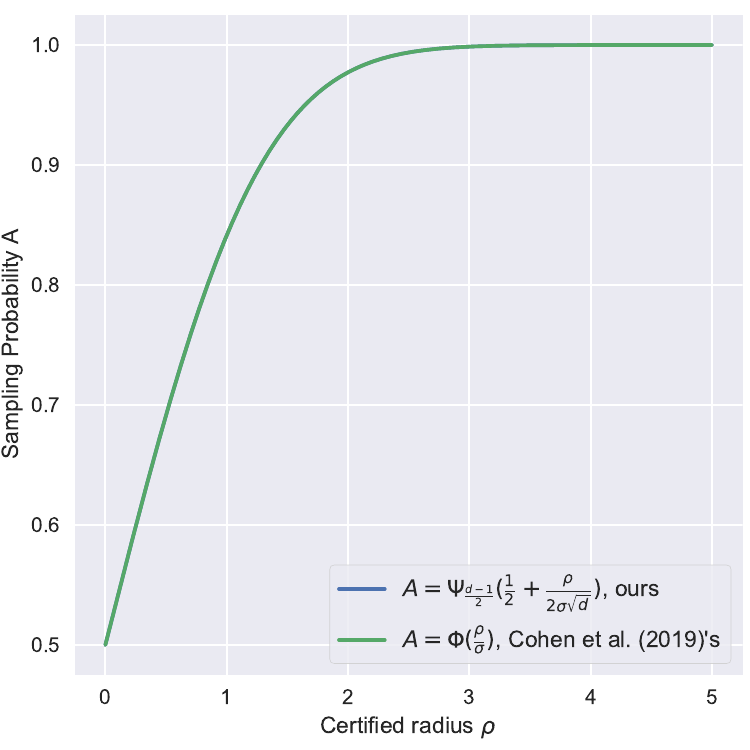}
	\end{subfigure}
	\caption{Our analytic formula for ESG highly approximates \citet{cohen2019}'s at a sufficiently large dimension. (Both $\sigma=1.0$. Left: $d=3072$, Right: $d=150224$.)}
 \label{approx_cohen}
\end{figure}
The ESG distributions are extensions of Gaussian by generalizing the exponent $\eta$ from $2$ to $\mathbb{R}_+$. As shown by \citep{yang2020}, the Gaussian distribution provides the SOTA certification for $\ell_2$ certified radius in RS. In this work, we report that their views can be augmented because in addition to Gaussian, the family of ESG distributions can provide almost identical certification compared to Gaussian. Namely, we find that ESG can tie the SOTA distribution for RS in a high-dimensional setting. Concretely, we comprehensively analyze the computational method for the certified radius in ESG. By reversing the traditional calculation procedure and proposing two asymptotically mild assumptions, we figure out the simple analytic relation between sampling probability and certified radius, which is quite rare among all the distributions. Our theoretical analysis is perfectly enhanced by numerical simulations and experiments on real-world datasets. Moreover, we find the analytic formula for ESG converges to the results derived by \citet{cohen2019}, which reveals a mathematical connection between the beta distribution and the normal distribution.

Likewise, the EGG distributions are derivations of the General Gaussian distribution from the perspective of the exponent $\eta$. General Gaussian was introduced by \citet{zhang2020}, and exploited by DSRS \citep{li2022} to treat the curse of dimensionality in randomized smoothing. In this work, we further address the problem of the curse of dimensionality by tightening the lower bounds offered by DSRS via EGG. On the whole, the effect of EGG on RS certification largely depends on whether the base classifier satisfies a concentration property. In a more ideal case, where the concentration property is almost completely satisfied, EGG with a smaller $\eta$ can provide tighter constant factors for the lower bounds provided by DSRS, which further alleviates the curse of dimensionality. But for more general cases, especially when real classifiers do not satisfy the concentration property well, EGG with a larger $\eta$ gives better certified accuracy. To sum it up, despite different mechanisms, the introduction of EGG comprehensively improves robustness certifications provided by General Gaussian, both theoretically and practically. 

Our main contributions include:
\begin{itemize}
    \item For sufficiently large dimensions, we derive the analytic relation between certified radius and sampling probability for ESG distributions in the RS framework. 
    The analytic formula obtained from ESG is in essence convergent to the formula derived by \citet{cohen2019}.

    \item Our theoretical analysis of EGG shows that the current solution to the curse of dimensionality in RS can be deepened: injecting some EGG into DSRS can tighten the constant factors of the lower bound of the certified radius.

    \item Extensive experiments on real datasets and classifiers verify our conclusion that the certification from ESG remains almost unchanged with $\eta$, and EGG provides comprehensively improved certifications compared to General Gaussian. On ImageNet, the increase in certified accuracy brought about by EGG can reach up to 6.4\% compared to the baseline.
	
 \end{itemize}
 
 \begin{figure}[t!]
	\centering
	\begin{subfigure}{0.49\linewidth}
		\centering
		\includegraphics[height=1.0\linewidth, width=1.0\linewidth]{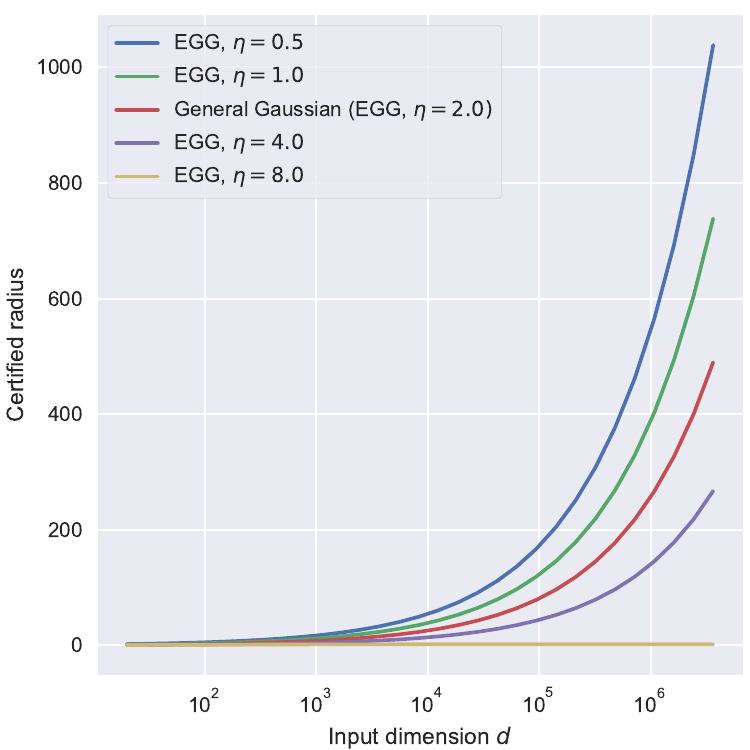}
	\end{subfigure}
	\centering
	\begin{subfigure}{0.49\linewidth}
		\centering
		\includegraphics[height=1.0\linewidth, width=1.0\linewidth]{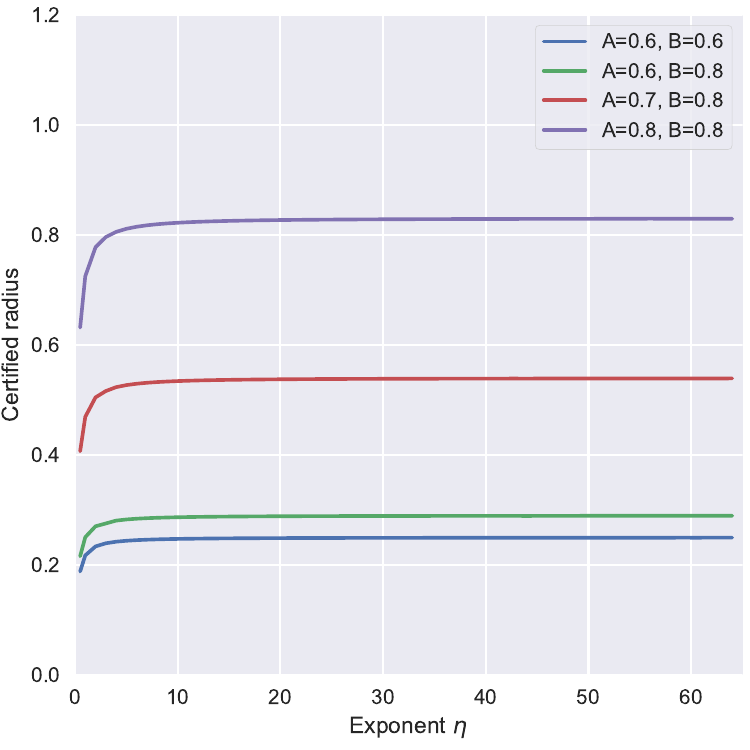}
	\end{subfigure}
	\caption{Numerical simulations for EGG in DSRS. Left: the concentration property holds ($B=1$), smaller $\eta$ provides tighter lower bounds. Right: the concentration property does not hold ($B<1$), larger $\eta$ provides better certified radius. For definitions of $A$ and $B$, see Equation~(\ref{primaldsrs}).}
 \label{main_sim}
\end{figure}

\section{Related work}
Randomized smoothing was first proposed as an extension for differential privacy, which provides a certified robustness bound for classifiers~\citep{lecuyer2019}. 
Subsequently, a series of improvements were made to obtain tighter $\ell_2$ norm certificates through the R\'enyi divergence and the Neyman-Pearson Lemma~\citep{li2019,cohen2019}.
Furthermore, there are methods modeling the certification process as functional optimization problems~\citep{zhang2020, dvijothamframework}.

A line of work focuses on extending the robustness certification from $\ell_2$-only to certifications against adversaries with other $\ell_p$ norms.
$\ell_0$ radius certification was made possible by constructing Neyman-Pearson sets for discrete random variables~\citep{lee2019, jia2020}.
Different from $\ell_2$ certified regions, the $\ell_1$ certified region is asymmetric in the space, which poses a new challenge to the certification algorithm. 
By perturbing input data under other noise distributions, such as Laplace and uniform distributions, the previous works have obtained $\ell_1$ certificates~\citep{teng2020, yang2020, levine2021}.
In addition to the asymmetry, some work discovered the phenomenon of the curse of dimensionality when trying to certify against $\ell_p$ adversaries whose $p>2$~\citep{yang2020, blum2020, kumar2020}.
To deal with this issue, a recent work offered an $\Omega(1)$ bound \wrt the input dimension for $\ell_\infty$ certified radius by introducing a supplementary smoothing distribution~\citep{li2022}, which breaks the curse of dimensionality theoretically for the first time. 
Lately, a study found that the computation of certified radius can be improved by incorporating the geometric information from adjacent decision domains of the same class~\citep{cullen2022}. 

There were also investigations using anisotropic or sample-specific smoothing noise to improve the certification~\citep{eiras2022, sukenik2022intriguing}.
In addition, a chain of work focused on improving the performance of base classifiers through adopting better training techniques~\citep{salman2019, zhai2020, jeong2020, jeong2021}, or introducing denoising modules~\citep{Hadi2020certified, carlini2022, wu2023}.
There were also attempts to adapt RS to broader application scenarios. 
RS was extended into defenses against adversarial patches~\citep{levine2020, yatsura2022} and semantic perturbations~\citep{li2021, hao2022, alfarra2022,pautov2022}. 
Moreover, RS has been shown to be capable of providing provable guarantees for tasks such as object detection, semantic segmentation, and watermarking~\citep{chiang2020,fischer2021,bansal2022}.

\section{Preliminaries}
\textbf{Problem setup.~}
We focus on the typical multi-class classification task in this work.
Let $x_i \in \mathbb{R}^d$ be the $i$-th $d$-dimensional data point and $y_i \in \mathcal{Y} = \{1, 2, \cdots, N\}$ be its corresponding ground-truth label.
We assume a dataset $\mathcal{J}$ contains data pairs $(x_i, y_i), i \in  \mathbb{N}_{\leq n}$ that are i.i.d drawn from the 
sample space $\mathbb{R}^{d} \times \mathcal{Y}$. 
A $N$-way \emph{base classifier} (neural networks in this work) $f:\mathbb{R}^d \to \mathcal{Y}$
can be trained to maximize the empirical classification accuracy $\frac{1}{\vert \mathcal{J} \vert}\sum_{(x, y) \in \mathcal{J}}\mathds{1}_{f(x) = y}$ on dataset $\mathcal{J}$.
Given an arbitrary data point $x$ and its label $y$, it is known that in practice, most classifiers trained using standard training techniques are susceptible to adversarial perturbations within a small $\epsilon$-ball. 

\textbf{Randomized smoothing.~} To mitigate adversarial perturbations, RS has been employed as a certified defense method that can provide a robustness guarantee on the correctness of the classification results from classifiers.
It provides the robustness certification for the base classifier $f$ by constructing its smoothed counterpart $\bar{f}$. Given a base classifier $f$, an input $x_0 \in \mathbb{R}^d$ and a smoothing distribution $\mathcal{P}$, the \emph{smoothed classifier} is defined as follows: 
\begin{equation}
\bar{f}_{\mathcal{P}}(x_0)={\mathop{\mathrm{argmax}}\limits_{a\in \mathcal{Y}}}\ \mathbb{P}_{z\sim\mathcal{P}}\{f(x_0+z)=a\}.
\end{equation}
With the definition of $\bar{f}$, we can evaluate its $\ell_p$ certified robustness by $\ell_p$ certified radius defined below.
\begin{definition}
Given a base classifier $f: \mathbb{R}^d \to \mathcal{Y}$, its smoothed counterpart $\bar{f}_{\mathcal{P}}:\mathbb{R}^d \to \mathcal{Y}$ under a distribution $\mathcal{P}$ and a labeled example $(x_0, y_0) \in \mathbb{R}^d \times \mathcal{Y}$. Then $r$ is called \textbf{$\ell_p$ certified radius} of $\bar{f}_{\mathcal{P}}$ if
\begin{equation}
\forall x,\ \Vert x - x_0 \Vert_p < r,\ \bar{f}_\mathcal{P}(x) = y_0.
\end{equation}
\end{definition}

To make the most conservative use of the information from smoothed classifiers (\eg, sampling probability $A$), we can consider a \emph{true binary classifier}. Given a base classifier $f:  \mathbb{R}^d\to \mathcal{Y}$, we call $\tilde{f}_{x_0}: \mathbb{R}^d\to \{0, 1\}$ a true binary classifier of $f$ if for $(x_0, y_0) \in \mathbb{R}^d\times \mathcal{Y}$ and random vector $z \in \mathbb{R}^d$:
\begin{equation}
\tilde{f}_{x_0}(z) = \mathds{1}_{f(x_0+z)=y_0}.
\end{equation}
In short, we need to construct the \emph{worst} true binary classifier to compute the certified radius in RS. For example, \citet{cohen2019} introduced the Neyman-Pearson lemma and successfully found such a worst true binary classifier. Given $\tilde{f}_{x_0}$ from the function space $\mathcal{F} = \{h(x)\mid h(x) \in [0, 1], \forall x \in \mathbb{R}^d\}$, their idea can be formulated as 
\begin{equation}\label{primal}
\small
\begin{aligned}
\min_{\tilde{f}_{x_0}\in \mathcal{F}} \quad & \mathbb{E}_{z\sim \mathcal{P}}\left(\tilde{f}_{x_0}(\delta + z)\right),\\
\textrm{s.t.} \quad & \mathbb{E}_{z\sim \mathcal{P}}\left(\tilde{f}_{x_0}(z)\right) = A.
\end{aligned}
\end{equation}

\begin{table*}[t] 
\scriptsize 
  \centering
   \caption{Properties and definitions of distributions.}\label{tb1} 
  \begin{tabular}{cccc}
    \toprule
    Distribution   & PDF     & Notation  & Formal Variance \\
    \midrule
    Standard Gaussian & $\propto \exp\left(-\frac{r^2}{2\sigma^2}\right)$ & $\mathcal{N}(\sigma)$ & $\sigma$ \\
    Exponential Standard Gaussian & $\propto \exp\left(-\frac{r^\eta}{2\sigma_s^\eta}\right)$& $\mathcal{S}(\sigma, \eta)$ & $\sigma_s=2 ^{-\frac{1}{\eta}} \sqrt{\frac{d\Gamma(\frac{d}{\eta})}{\Gamma(\frac{d+2}{\eta})}}\sigma$\\
    Truncated Exponential Standard Gaussian & $\propto \exp\left(-\frac{r^\eta}{2\sigma_s^\eta}\right)\mathds{1}_{r \leq T}$& $\mathcal{S}_t(\sigma, \eta, T)$ & $\sigma_s=2 ^{-\frac{1}{\eta}} \sqrt{\frac{d\Gamma(\frac{d}{\eta})}{\Gamma(\frac{d+2}{\eta})}}\sigma$\\
    Exponential General Gaussian & $\propto r^{-2k}\exp\left(-\frac{r^\eta}{2\sigma_g^\eta}\right)$& $\mathcal{G}(\sigma, \eta, k)$ & $\sigma_g=2 ^{-\frac{1}{\eta}} \sqrt{\frac{d\Gamma(\frac{d - 2k}{\eta})}{\Gamma(\frac{d-2k+2}{\eta})}}\sigma$\\
    Truncated Exponential General Gaussian & $\propto r^{-2k}\exp\left(-\frac{r^\eta}{2\sigma_g^\eta}\right)\mathds{1}_{r \leq T}$& $\mathcal{G}_t(\sigma, \eta, k, T)$ & $\sigma_g=2 ^{-\frac{1}{\eta}} \sqrt{\frac{d\Gamma(\frac{d - 2k}{\eta})}{\Gamma(\frac{d-2k+2}{\eta})}}\sigma$\\
    \bottomrule
  \end{tabular}
\end{table*}

\textbf{Double sampling randomized smoothing.}
Essentially, the DSRS framework is based on a generalization of the Neyman-Pearson Lemma~\citep{chernoff1952} that introduces one more subjection into the system. More specifically, DSRS provides a method to construct the worst true binary classifier based on the sampling probabilities of the base classifier under two different distributions. Like the Neyman-Pearson lemma, the problem of DSRS can be formulated as a functional optimization problem as below:
\begin{equation}\label{primaldsrs}
\begin{small}
\begin{aligned}              
\min_{\tilde{f}_{x_0}\in \mathcal{F}} \quad & \mathbb{E}_{z\sim \mathcal{P}}\left(\tilde{f}_{x_0}(\delta + z)\right),\\
\textrm{s.t.} \quad & \mathbb{E}_{z\sim \mathcal{P}}\left(\tilde{f}_{x_0}(z)\right) = A,\\ 
& \mathbb{E}_{z\sim \mathcal{Q}}\left(\tilde{f}_{x_0}(z)\right) = B.
\end{aligned}
\end{small}
\end{equation}
In the equations (\ref{primaldsrs}), $A, B\in [0,1]$ are probabilities that the base classifier $f$ outputs the right label $y_0$ for example $x_0$ under noise distributions $\mathcal{P}$, $\mathcal{Q}$, respectively. Practically, they are usually estimated by the Monte Carlo sampling. For an example $x_0$ and a given combination of $\mathcal{P}$, $\mathcal{Q}$, $A$, $B$, we are able to derive a unique $\tilde{f}_{x_0}$. Finally, by finding the maximum $\Vert\delta\Vert_2$ that satisfies $\mathbb{P}_{z\sim \mathcal{P}}\{f({x_0} + \delta + z) = y_0\} > 0.5$, we obtain the certified radius of $x_0$. 

\section{Exponential Gaussian distributions}
In this section, we show the definition of ESG and EGG distributions used throughout the paper. In brief, ESG distributions are generalizations of the Gaussian distribution, which provides the SOTA performance in providing $\ell_2$ certified radius for randomized smoothing \citep{cohen2019, yang2020}. Likewise, EGG distributions are generalizations of the General Gaussian distribution, which was introduced into the DSRS framework to provide a theoretical solution to the curse of dimensionality in randomized smoothing \citep{li2022}. To the best of our knowledge, the EGG distributions belong to the Kotz-type distribution \citep{kotz1975}, and we are the first to investigate their performance on randomized smoothing. In summary, our work systematically studies the interaction between multivariate distributions and (DS)RS through the lens of the exponent of distributions.

For conventional randomized smoothing tasks, only ESG or EGG distributions will be used.
Additionally, under the DSRS framework, when using ESG as the smoothing distribution, the Truncated Exponential Standard Gaussian (TESG) distribution is employed as the supplementary distribution.
Similarly, Truncated Exponential General Gaussian (TEGG) serves as the supplementary distribution when adopting EGG as the smoothing distribution.
We let $\mathcal{S}(\sigma, \eta)$ and $\mathcal{G}(\sigma, \eta, k)$ be the probability density functions (PDFs) of ESG and EGG, respectively. 
Table~\ref{tb1} shows the definitions and basic properties of the distributions. In the table, $r,T, \sigma, \eta \in \mathbb{R}_+$ and $d\in \mathbb{N}_+$. 
$\Gamma(\cdot)$ is the gamma function. 
Following the settings in the previous studies~\citep{yang2020,li2022}, we set the formal variance to ensure $\mathbb{E}r^2$ is a constant for all the smoothing distributions. 
We let $\sigma_s$ and $\sigma_g$ be the formal variances of EGG and ESG, respectively.
The CDFs of the beta distribution ${\rm Beta}(\alpha, \alpha)$ and the gamma distribution $\Gamma(\alpha, 1)$ are denoted respectively by $\Psi_\alpha(\cdot)$ and $\Lambda_\alpha(\cdot)$.
We write $\phi_s(r)$ and $\phi_g(r)$ corresponding to the PDFs of $\mathcal{S}(\sigma, \eta)$ and $\mathcal{G}(\sigma, \eta, k)$ respectively. More details for the distributions are deferred to Appendix \ref{appdis}.

\section{ESG's certifications: hardly changes with the exponent $\eta$}\label{ESG_theory}
In this section, we provide an analysis of ESG's certifications on randomized smoothing. Overall, for sufficiently large dimensions $d$, the certified radius offered by ESG has almost nothing to do with the exponent $\eta$, despite it being a significant hyperparameter for ESG. In the end, we derive a concise analytic formula that reveals the relationship between sampling probability and certified radius for the ESG distributions, which highly approximates the formula derived by \citep{cohen2019} for the Gaussian distribution.

To begin with, we consider the dual problem of Problem (\ref{primal}) for the ESG distributions. Let $\mathcal{P}=\mathcal{S}(\sigma,\eta)$,\ $p(x)$ be the PDF of ESG, and denote $\mathcal{V}$ for $\{z\mid p(z-\delta) + \nu p(z) < 0\}$ where $\nu\in\mathbb{R}$. Formally, to calculate the certified radius $r$, we have
\begin{subequations}\label{dual}
\begin{align}
r = \argmax_{\Vert\delta\Vert_2}\max_{\nu\in\mathbb{R}} \quad &\mathbb{P}_{z \sim \mathcal{P}+\delta}\{z\in\mathcal{V}\} \geq \frac{1}{2},\label{npdual1}\\
\textrm{s.t.} \quad &\mathbb{P}_{z\sim \mathcal{P}}\{z\in\mathcal{V}\} = A. \label{npdual2}
\end{align}
\end{subequations}
The proof of duality is omitted since it was provided in \citet{zhang2020}. Herein, we directly give the solution to this problem in the following theorem and defer the derivation to Appendix \ref{apesgNP}. 
\begin{theorem}\label{thesgNP}
For $\delta\in\mathbb{R}^d $ and $\rho\in\mathbb{R}_+$, letting $\delta=(\rho, 0, \cdots, 0)^T$, we have
\begin{subequations}
\begin{align}
\mathbb{P}_{z \sim \mathcal{P}+\delta}\{z\in\mathcal{V}\} = \mathbb{E}_{u\sim\Gamma(\frac{d}{\eta}, 1)}\omega_\sharp(u, \nu),\\
\mathbb{P}_{z \sim \mathcal{P}}\{z\in\mathcal{V}\}=\mathbb{E}_{u\sim\Gamma(\frac{d}{\eta}, 1)}\omega_\natural(u, \nu),
\end{align}
\end{subequations}
where
\begin{subequations}\
\scriptsize
\begin{align}
\omega_\sharp(u, \nu) = \left\{
\begin{aligned}\label{oshtxt}
&\Psi_{\frac{d-1}{2}}\left(\frac{2^{\frac{2}{\eta}}\sigma_s^2(u + \ln(-\nu))^{\frac{2}{\eta}} - (\sigma_s(2u)^\frac{1}{\eta} - \rho)^2}{4\rho \sigma_s(2u)^\frac{1}{\eta}}\right), \\
&\hspace*{16em} u + \ln(-\nu) \geq 0,\\
&0, \hspace*{15em} u + \ln(-\nu) < 0.\\
\end{aligned}\right.\\ 
\omega_\natural(u, \nu) = \left\{
\begin{aligned}\label{onattxt}
&\Psi_{\frac{d-1}{2}}\left(\frac{(\rho + \sigma_s(2u)^\frac{1}{\eta})^2-2^{\frac{2}{\eta}}\sigma_s^2(u - \ln(-\nu))^{\frac{2}{\eta}}}{4\rho \sigma_s(2u)^\frac{1}{\eta}}\right), \\
&\hspace*{16em} u-\ln(-\nu) \geq 0,\\
&1, \hspace*{15em} u-\ln(-\nu) < 0.
\end{aligned}\right.\\ \nonumber
\end{align}
\end{subequations}
\vspace{-12mm}
\end{theorem}
Conventionally, we obtain the certified radius $r$ by known sampling probability $A$. The procedure can be outlined as a two-layer binary search, the outer one of which searches the value of $\Vert\delta\Vert_2$, and the inner one finds the maximum $\nu$ for the specified $\Vert\delta\Vert_2$. However, if we think about this problem in reverse, the relationship between certified radius $\rho$ and sampling probability $A$ can become very obvious. That is, we do not consider calculating the certified radius $r$ from $A$, but get $A$ from a given $\rho$. For ESG, we have the following theorem:
\begin{theorem}\label{invprob}
\begin{subequations}
Let $\omega_{\sharp}(u, \nu)$ and $\omega_{\natural}(u, \nu)$ be defined as in Theorem \ref{thesgNP}. Then the sampling probability $A$ can be calculated from the certified radius $\rho$ by 
\begin{align}
A = &\mathbb{E}_{u\sim\Gamma(\frac{d}{\eta}, 1)}\omega_{\natural}(u, \nu), \label{geta}\\ 
\textrm{s.t.}\ \ &\mathbb{E}_{u\sim\Gamma(\frac{d}{\eta}, 1)}\omega_{\sharp}(u, \nu) = \frac{1}{2}. \label{getnu}
\end{align}
\end{subequations}
\end{theorem}
In fact, Theorem \ref{invprob} is still not succinct enough to demonstrate the computation from $\rho$ to $A$. To simplify it, we make two assumptions, both of which are very mild.
\begin{assumption}\label{as1}
For the ESG distribution $\mathcal{S}(\sigma, \eta)$ defined on $\mathbb{R}^d$, assume $\ d \gg \eta$ and $\sigma\in (0,1]$. 
\end{assumption}
This assumption is reasonable in the high-dimensional machine learning setting. For instance, if we add Gaussian noises to ImageNet, we have $d=150224, \eta=2$, and $\sigma$ is usually 0.5 or 1.0. Based on this, it is quite simple to approximate the formal variance $\sigma_s$. We see the following lemma:
\begin{lemma}\label{sigmasbound}
Under Assumption \ref{as1}, let $d$ be the dimension, and $\sigma_s$ be defined as in Table \ref{tb1}, and we have $\sigma_s = \Theta(d^{\frac{1}{2} - \frac{1}{\eta}})$.
\end{lemma}
The proof is based on Stirling's formula and we leave it to Appendix \ref{accsig}. With this property, talking $\sigma_s=(\frac{\eta}{2})^\frac{1}{\eta}\sigma d^{\frac{1}{2} - \frac{1}{\eta}}$ will only introduce infinitesimal errors. Next, another assumption we need is: 
\begin{assumption}\label{as2}
Let $\nu$ be found by Equation (\ref{npdual2}), and we assume $\frac{d}{\eta}\gg \ln(-\nu)$.
\end{assumption}
Appendix \ref{mildas2} has more on the mildness of the assumption above. It helps us exploit the concentration property of the gamma distributions:
\begin{lemma}\label{lemmag1}
(Bilateral Concentration of the Gamma Distribution) Let $X\sim \Gamma(\frac{d}{\eta}, 1)$ be a random variable, where $\eta \in \mathbb{R}_+, d \in \mathbb{N}_+$. Let $\iota=\frac{\eta}{\epsilon^2d}$, then for any $\iota \in (0, 1)$, the following inequality holds:
	\begin{equation}\label{130}
	\mathbb{P}\{(1-\epsilon)\frac{d}{\eta} < X < (1+\epsilon)\frac{d}{\eta}\} \geq 1-\iota.
	\end{equation}
\end{lemma} 
The proof is very similar to that of Lemma \ref{lemmae4}. We still take Gaussian noises on ImageNet to illustrate Lemma \ref{lemmag1}. In fact, the Gaussian distribution will lead $u\sim\Gamma(\frac{d}{2}, 1)$ in Equation (\ref{getnu}), which satisfies $\mathbb{P}\{0.99 \cdot \frac{d}{2} < u < 1.01 \cdot \frac{d}{2}\}>0.9935$. In other words, Assumption \ref{as2} essentially guarantees $\mathbb{P}\{u + \ln(-v)<0\}$ and $\mathbb{P}\{u - \ln(-v)<0\}$ are almost $0$, thus we can omit these branches in Equations (\ref{oshtxt}) and (\ref{onattxt}) when estimating the integral. WLOG, we let $u$ be a constant value, and then let 
\begin{equation}
 \Psi_{\frac{d-1}{2}}\left(\frac{2^{\frac{2}{\eta}}\sigma_s^2(u + \ln(-\nu))^{\frac{2}{\eta}} - (\rho - \sigma_s(2u)^\frac{1}{\eta})^2}{4\rho \sigma_s(2u)^\frac{1}{\eta}}\right) = \frac{1}{2}.
\end{equation}
Substituting $\sigma_s=(\frac{\eta}{2})^\frac{1}{\eta}\sigma d^{\frac{1}{2} - \frac{1}{\eta}}$ and $u=\frac{d}{\eta}$ into the equation above, we get
\begin{equation}
\frac{\sigma^2d(1+\frac{\eta\ln(-\nu)}{d})^{\frac{2}{\eta}} - (\rho-\sigma\sqrt{d})^2}{4\rho\sigma\sqrt{d}}=\frac{1}{2}.
\end{equation}
By Assumption \ref{as2}, recalling the equivalent infinitesimal replacement $(1+x)^a\sim 1+ax$ when $x\to 0$, we see 
\begin{equation}\label{approxnu}
\ln(-\nu)\approx\frac{\rho^2}{2\sigma^2}.
\end{equation}
The solution for $\ln{(-\nu)}$ also verifies the Assumption \ref{as2} that $\frac{d}{\eta} \gg \ln(-\nu)$ is mild, considering practically we seldom see $\ell_2$ certified radius $\rho>5$ in RS. WLOG, injecting $\ln(-\nu)=\frac{\rho^2}{2\sigma^2}$ into Equation (\ref{geta}), we finally obtain 
\begin{equation}\label{approximation}
A=\Psi_{\frac{d-1}{2}}(\frac{1}{2}+\frac{\rho}{2\sigma\sqrt{d}}).
\end{equation}
In fact, this estimation is convergent to \citet{cohen2019}'s formula when $d$ is sufficiently large:
\begin{equation}
\Psi_{\frac{d-1}{2}}(\frac{1}{2}+\frac{\rho}{2\sigma\sqrt{d}})\xrightarrow{d\to\infty}\Phi(\frac{\rho}{\sigma}),
\end{equation}
whose differential case is shown by \citet{Ryder2012}. Interestingly, though the exponent $\eta$ is a significant parameter for the derivation on ESG, this estimation (\ref{approximation}) is irrelevant to $\eta$. Furthermore, our experiments on real-world datasets in Section \ref{exs} show that ESG's certifications are highly inert to $\eta$ for both RS and DSRS, which corroborates our theoretical analysis, and indicates that there exist similar $A(\rho)$ relationships in DSRS.
 We show proof and more details for this approximation in Appendix \ref{extapp}. 

\section{EGG's certifications: significantly improve with the exponent $\eta$}\label{EGG_theory}

It is proved that taking the General Gaussian distribution as the smoothing distribution in DSRS provides an $\Omega(\sqrt d)$ lower bound for the $\ell_2$ certified radius~\citep{li2022}. This bound can be converted to an $\Omega(1)$ lower bound for the $\ell_{\infty}$ certified radius ~\citep{kumar2020}, which breaks the curse of dimensionality against $\ell_\infty$ adversaries in high-dimensional settings.
Nevertheless, though \citet{li2022} proposed a theoretical solution to the curse of dimensionality for the first time, the study on the curse of dimensionality is still lacking. In this work, we investigate the further alleviation of the curse of dimensionality on the basis of DSRS, showing that theoretically, EGG can provide tighter lower bounds for $\ell_2$ certified radius than the original DSRS. 
Primarily, we define a ($\sigma, p, \eta$)-concentration property to start our analysis: 
\begin{definition}\label{cct}
(($\sigma, p, \eta$)-Concentration Property) Let $f:  \mathbb{R}^d\to \mathcal{Y}$ be an arbitrarily determined base classifier, $(x_0, y_0) \in \mathbb{R}^d \times \mathcal{Y}$ be a labeled example. 
We say $f$ satisfies ($\sigma, p, \eta$)-concentration assumption at $(x_0, y_0)$ if for $p \in (0,1)$ and $T$ satisfying
\begin{equation}
\mathbb{P}_{z\sim\mathcal{S}(\sigma, \eta)}\{\Vert z\Vert_2\leq T\} = p,
\end{equation} 
$f$ satisfies
\begin{equation}\label{eq:all_one_prob}
\mathbb{P}_{z\sim \mathcal{S}(\sigma, \eta)}\{f(x_0+z)=y_0\mid \Vert z\Vert_2\leq T\}=1.
\end{equation} 
\end{definition}
The concentration property essentially defines a partly robust classifier that makes few mistakes on examples perturbed by limited noises. For instance, letting $\eta=2$, if a classifier satisfies the ($\sigma, p, 2$)-concentration property on example $(x_0, y_0)$, it will predict almost entirely correctly on perturbed examples with Gaussian noise $z$ with $\Vert z\Vert_2<T$. Despite the assumption being seemingly strict, it is satisfied well in the light of observations from \citet{li2022}: some robust classifiers provide almost perfect predictions for some examples from ImageNet when the $\ell_2$ length $\Vert z\Vert_2<\sigma\sqrt{d}$, but the accuracy intensely declines when the noise is stronger.

Herein, we first let the base classifier $f$ satisfy the ($\sigma, p, 2$)-concentration property, the original assumption used in \citet{li2022}. 
Then we show a theorem that some EGG distributions with their corresponding truncated counterparts can certify the $\ell_2$ radius with $\Omega(\sqrt d)$ lower bounds by DSRS. In other words, the current solution to the curse of dimensionality provided by DSRS can be significantly augmented, that all EGG with $\eta \in (0, 2)$ have the potential to break the curse with better lower bounds. 

\begin{theorem}[EGG with $\eta \in(0,2)$ can certify $\Omega(\sqrt{d})$ lower bounds]\label{fixbase}
Let $d \in \mathbb{N}_+$ be a sufficiently large input dimension, $(x_0, y_0) \in \mathbb{R}^d \times \mathcal{Y}$ be a labeled example and $f: \mathbb{R}^d\to \mathcal{Y}$ be a base classifier that satisfies the ($\sigma, p, 2$)-concentration property \wrt $(x_0, y_0)$. For the DSRS method, let $\mathcal{P} = \mathcal{G}(\sigma, \eta, k)$ be the smoothing distribution to give a smoothed classifier $\bar{f}_{\mathcal{P}}$, and $\mathcal{Q}=\mathcal{G}_t(\sigma, \eta, k, T)$ be the supplementary distribution with $T=\sigma\sqrt{2\Lambda_{\frac{d}{2}}^{-1}(p)}, d - 2k \in [1, 30] \cap \mathbb{N}$ and $\eta \in \{1, \frac{1}{2}, \frac{1}{3}, \cdots, \frac{1}{50}\}$. Then for the smoothed classifier $\bar{f}_{\mathcal{P}}(x)$, the certified $\ell_2$ radius satisfies
\begin{equation}
r_{DSRS} \geq 0.02\sigma \sqrt{d}.
\end{equation}
\end{theorem}

\begin{figure}[t!]
  \centering
  \includegraphics[width=3in]{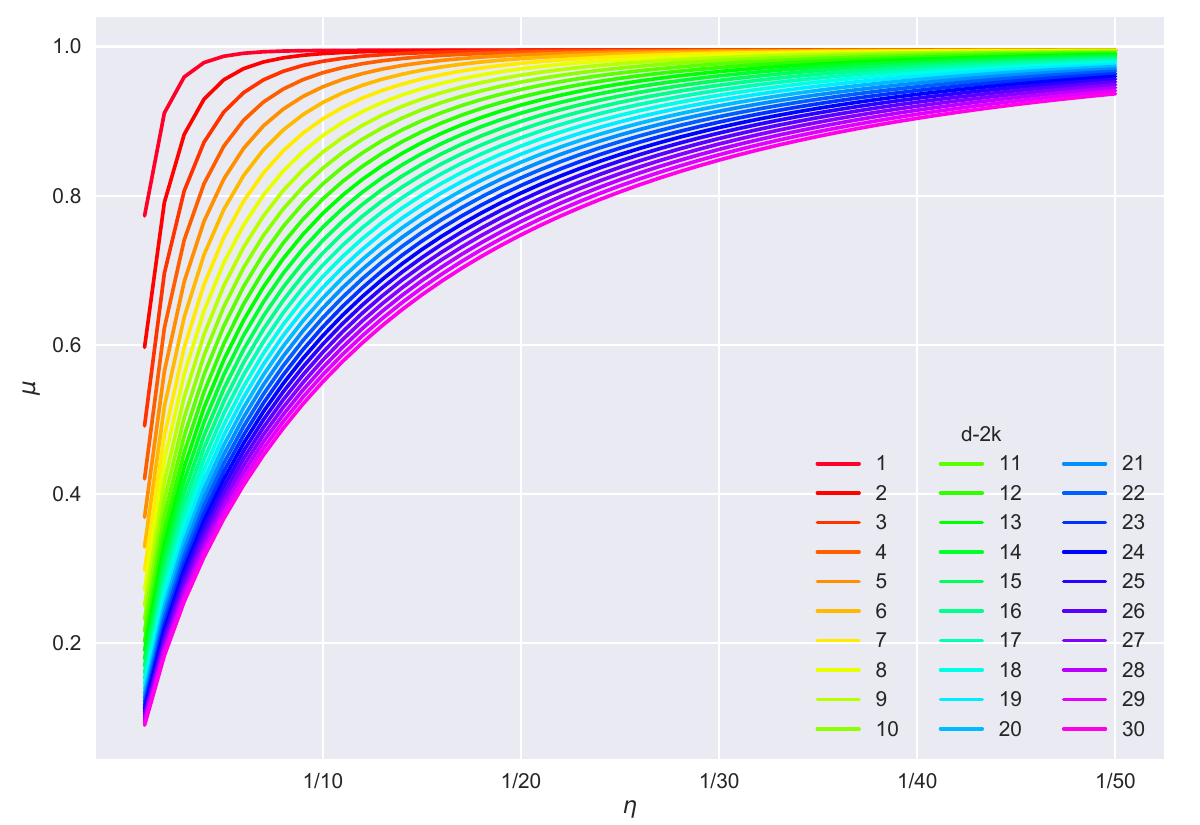}
  \caption{Tight factor $\mu$ grows as $\eta$ shrinks, for most $d-2k \in [1, 30] \cap \mathbb{N}$.}
  \label{incmu}
\end{figure}

\begin{figure*}[htbp!]
  \centering
  \includegraphics[width=0.8\linewidth]{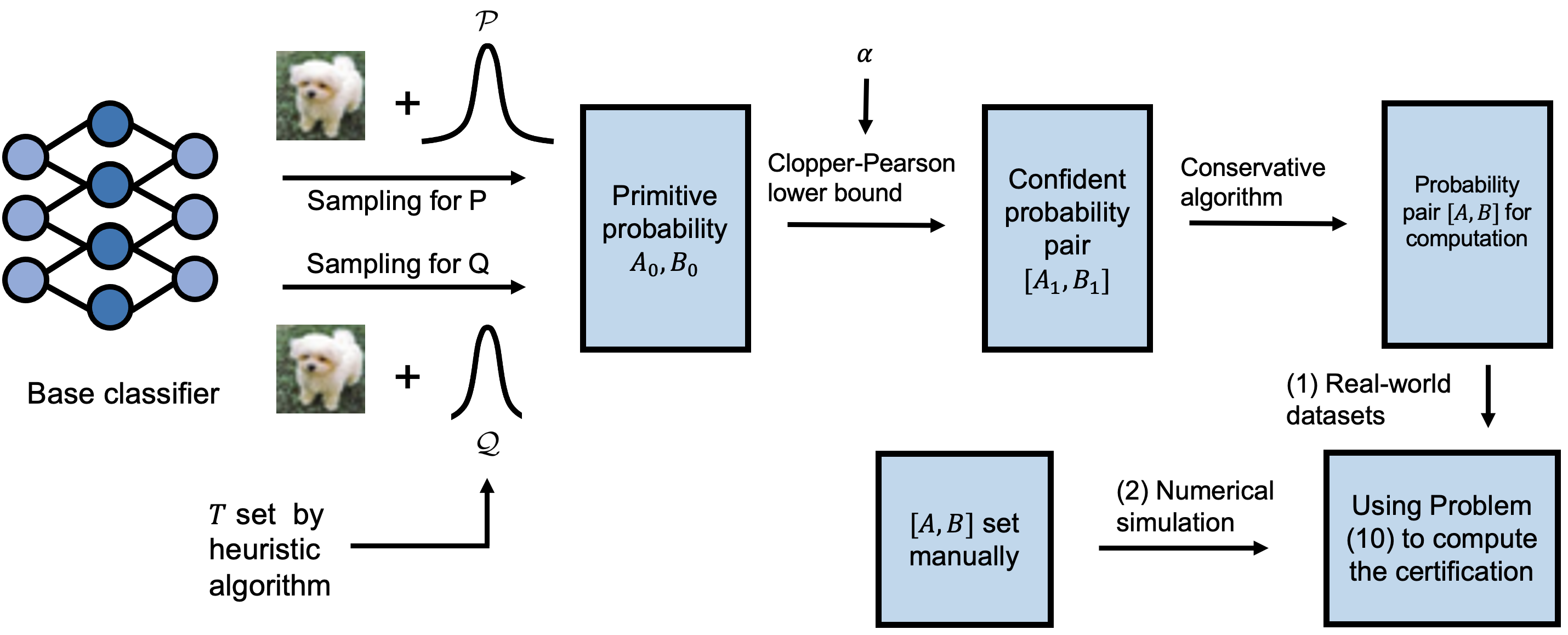}
  \caption{Illustration for experiments.}
 \label{ex_pcd}
 \vspace{-5mm}
\end{figure*}

We briefly summarize the proof here and leave the details in Appendix \ref{appth}.
In essence, Theorem 1 is generalizing Theorem 2 of \citet{li2022} to EGG.
Practically, it is intractable to derive an analytic solution for the certified radius $r_{DSRS}$ for EGG. 
Therefore, the problem is simplified by introducing the concentration assumption, whereupon we solve Problem (\ref{primaldsrs}) only considering the truncated distribution $\mathcal{Q}$. The derivation is very similar to the solution of Problem (\ref{primaldsrs}), and the certification of the radius $\rho$ for EGG is finally dependent on the discriminant
\begin{equation}\label{dis8}
\mathbb{E}_{u\sim \Gamma(\frac{d-2k}{\eta}, 1)}\Psi_{\frac{d-1}{2}}\left(\frac{T^2-(\sigma_g(2u)^{\frac{1}{\eta}}-\rho)^2}{4\rho \sigma_g(2u)^\frac{1}{\eta}}\right) \geq \frac{1}{2}.
\end{equation}
With this formulation, whether a radius $\rho$ can be certified can be directly judged since the LHS of Equation (\ref{dis8}) is a function of $\rho$. 
For Theorem \ref{fixbase}, we just need to substitute $\rho=0.02\sigma\sqrt{d}$ into Equation (\ref{dis8}), and check whether the inequality holds. If true, then $\rho=0.02\sigma\sqrt{d}$ is certified, and vice versa.
For the detailed derivation, please see the proof of Lemma \ref{lemmae2} in the appendix.

Plus, we explore the effects of EGG under the ($\sigma, p, \eta$)-concentration assumption, where we can construct $\Omega(d^{1/\eta})$ lower bounds for the certified radius. However, though it is formally tighter than $\Omega(\sqrt{d}$), we find these two lower bounds are fundamentally equivalent. See Theorem \ref{thcorres} in Appendix \ref{apthcores}. 

\noindent\textbf{Tighter constant factors of lower bounds.~}
Besides certifying the lower bound radius $0.02\sigma\sqrt{d}$, the proof of Theorem \ref{fixbase} naturally contains an approach to determining tight constant factors for each EGG distribution (Algorithm \ref{alg:tightmu}). In fact, the value of the LHS of Equation (\ref{dis8}) monotonically decreases with $\rho$, meaning that performing a simple binary search on $\rho$ can provide the accurate certified radius. Therefore, we consider parameterizing the radius $\rho$ into $\mu\sigma\sqrt{d}$, thus binary searching on the constant factor will derive the tight constant factor $\mu$ for the $\Omega(\sqrt{d})$ lower bound (Algorithm \ref{alg:tightmu}). We report computational results for $\mu$ in Figure \ref{incmu}, where for values of $d-2k$ except $1$, the tight constant factor $\mu$ increases monotonically as the $\eta$ decreases. Essentially, these results demonstrate that the solution to the curse of dimensionality provided by \citet{li2022} (with General Gaussian, namely $\eta=2$ in EGG) can be improved by choosing smaller $\eta\in (0,2)$, for most $d-2k \in [1, 30] \cap \mathbb{N}$. 

Overall, the theoretical analysis shows that when the base classifier satisfies the concentration property, EGG with $\eta \in(0,2)$ brings significant enhancement for the lower bound of the certified radius offered by DSRS \citep{li2022}. Obviously, it is hard for the realistic model to be perfectly \emph{concentrated}, thus EGG takes different effects under real classifiers. However, what remains unchanged is that EGG also comprehensively improves the certification of DSRS. See Section \ref{exs} for details.

\section{Experiments}\label{exs}

\begin{table*}[t!] 
\centering
\caption{Certified radius at $r$ for standardly augmented models, certified by ESG under DSRS}\label{best_ESG_std}
\resizebox{1.0\linewidth}{!}{%
\begin{tabular}{ccccc ccccc ccccc c} 
\toprule
\multirow{2}{*}{Dataset} & \multirow{2}{*}{Method} &\multicolumn{14}{c}{Certified accuracy at $r$}\\
\cline{3-16}
&&  0.25  &  0.50  &  0.75  &  1.00  &  1.25  &  1.50  &  1.75  &  2.00  &  2.25  &  2.50  &  2.75  &  3.00  &  3.25  &  3.50 \\
\hline
\multirow{4}{*}{CIFAR10}& ESG, $\eta=1.0$ &57.6\% &42.6\% &31.3\% &21.5\% &15.8\% &12.8\% &8.6\% &6.8\% &4.3\% &2.3\% &1.3\% &0.8\% &0.3\% &0.1\% \\
& ESG, $\eta=2.0$ (Gaussian) &57.6\% &42.6\% &31.6\% &21.5\% &15.8\% &12.7\% &8.8\% &6.8\% &4.5\% &2.4\% &1.3\% &0.7\% &0.2\% &0.2\%\\
& ESG, $\eta=4.0$ &57.6\% &42.6\% &31.3\% &21.5\% &15.9\% &12.9\% &8.6\% &6.9\% &4.3\% &2.4\% &1.3\% &0.8\% &0.2\% &0.1\%   \\
& ESG, $\eta=8.0$ &57.8\% &42.6\% &31.6\% &21.6\% &15.9\% &12.9\% &8.9\% &6.7\% &4.2\% &2.4\% &1.3\% &0.9\% &0.2\% &0.1\%  \\
\hline
\multirow{4}{*}{ImageNet}& ESG, $\eta=1.0$ &59.6\% &51.5\% &43.2\% &37.9\% &33.0\% &26.8\% &23.1\% &21.5\% &19.9\% &17.4\% &13.8\% &11.5\% &10.3\% &7.7\% \\
& ESG, $\eta=2.0$, (Gaussian)&59.6\% &51.6\% &43.1\% &38.0\% &32.9\% &26.9\% &23.1\% &21.5\% &19.7\% &17.4\% &13.6\% &11.4\% &10.1\% &8.3\% \\
& ESG, $\eta=4.0$ &59.6\% &51.5\% &43.2\% &38.0\% &32.9\% &27.2\% &23.1\% &21.6\% &19.9\% &17.2\% &13.6\% &11.4\% &10.2\% &8.0\% \\
& ESG, $\eta=8.0$ &59.6\% &51.5\% &43.2\% &38.0\% &33.0\% &26.8\% &23.1\% &21.6\% &19.7\% &17.3\% &13.6\% &11.5\% &10.1\% &8.4\% \\
\bottomrule
\end{tabular}
}
\vspace{-3mm}
\end{table*}

\begin{table*}[t] 
\centering
\caption{Certified radius at $r$ for standardly augmented models, certified by EGG under DSRS}\label{best_EGG_std}
\resizebox{1.0\linewidth}{!}{%
\begin{tabular}{ccccc ccccc ccccc c} 
\toprule 
\multirow{2}{*}{Dataset} & \multirow{2}{*}{Method} &\multicolumn{14}{c}{Certified accuracy at $r$}\\
\cline{3-16}
&&  0.25  &  0.50  &  0.75  &  1.00  &  1.25  &  1.50  &  1.75  &  2.00  &  2.25  &  2.50  &  2.75  &  3.00  &  3.25  &  3.50 \\
\hline
\multirow{7}{*}{CIFAR10}& EGG, $\eta=0.25$ &54.2\% &37.6\% &23.5\% &16.5\% &9.4\% &4.5\% &0.5\% &0.1\% &0.0\% &0.0\% &0.0\% &0.0\% &0.0\% &0.0\% \\
& EGG, $\eta=0.5$ &55.5\% &40.4\% &25.2\% &19.1\% &13.4\% &8.5\% &5.5\% &2.0\% &0.4\% &0.1\% &0.0\% &0.0\% &0.0\% &0.0\%\\
& EGG, $\eta=1.0$ &56.3\% &41.7\% &28.2\% &20.0\% &15.1\% &10.5\% &7.1\% &4.2\% &1.9\% &0.9\% &0.1\% &0.0\% &0.0\% &0.0\%  \\
& DSRS~\citep{li2022} (EGG, $\eta=2.0$) &56.7\% &42.4\% &29.3\% &20.2\% &15.7\% &11.5\% &8.0\% &5.5\% &2.6\% &1.5\% &0.6\% &0.1\% &0.0\% &0.0\%  \\
& EGG, $\eta=4.0$ &57.5\% &42.5\% &30.0\% &20.2\% &15.9\% &12.2\% &8.5\% &6.5\% &3.4\% &1.8\% &0.9\% &0.4\% &0.0\% &0.0\%  \\
&Ours (EGG, $\eta=8.0$) &\textbf{57.6\%} &\textbf{42.5\%} &\textbf{30.9\%} &\textbf{20.6\%} &\textbf{15.8\%} &\textbf{12.3\%} &\textbf{8.6\%} &\textbf{6.6\%} &\textbf{3.7\%} &\textbf{2.1\%} &\textbf{1.1\%} &\textbf{0.5\%} &\textbf{0.2\%} &0.0\% \\
\hline
\multirow{6}{*}{ImageNet}& EGG, $\eta=0.25$ &53.8\% &41.4\% &28.4\% &20.1\% &7.1\% &0.8\% &0.0\% &0.0\% &0.0\% &0.0\% &0.0\% &0.0\% &0.0\% &0.0\% \\
& EGG, $\eta=0.5$ &54.9\% &46.3\% &36.4\% &26.3\% &22.1\% &15.2\% &8.7\% &3.1\% &0.8\% &0.0\% &0.0\% &0.0\% &0.0\% &0.0\% \\
& EGG, $\eta=1.0$ &57.0\% &47.8\% &39.9\% &32.8\% &24.9\% &22.0\% &18.5\% &13.1\% &9.2\% &5.0\% &2.1\% &0.5\% &0.0\% &0.0\% \\
& DSRS~\citep{li2022} (EGG, $\eta=2.0)$ &58.4\% &48.5\% &41.5\% &35.2\% &28.9\% &23.3\% &21.3\% &18.8\% &14.1\% &11.1\% &8.9\% &6.1\% &2.2\% &1.4\%  \\
& EGG, $\eta=4.0$ &58.7\% &49.9\% &42.6\% &36.4\% &31.0\% &23.9\% &22.3\% &20.2\% &17.3\% &13.2\% &10.7\% &9.2\% &6.8\% &4.0\% \\
&Ours (EGG, $\eta=8.0$) &\textbf{59.1\%} &\textbf{50.8\%} & \textbf{42.9\%} &\textbf{36.8\%} &\textbf{31.8\%} &\textbf{24.6\%} &\textbf{22.6\%} &\textbf{20.7\%} &\textbf{18.9\%} &\textbf{14.5\%} &\textbf{11.7\%} &\textbf{10.1\%} &\textbf{8.6\%} &\textbf{5.2\%} \\
\bottomrule 
\end{tabular}
}
\vspace{-2mm}
\end{table*}

In this section, we report the effects of ESG and EGG distributions on the certified radius. Figure \ref{ex_pcd}
sketches the outline of our experiments. Here, we only focused on real-world datasets and leave details for numerical simulation experiments in Appendix \ref{apspns}. 
For the case of real-world datasets, $A$ and $B$ in Problem (\ref{dualdsrs}) are initially from Monte Carlo sampling results on a given base classifier, then reduced to respective Clopper-Pearson confidence intervals~\citep{clopper1934}, and finally determined by a conservative algorithm~\citep{li2022}. The procedure for real-world datasets can also be seen in Algorithm \ref{alg:DSRS} in Appendix \ref{DSRSalg}.

\textbf{Experimental setups.~} 
All base classifiers used in this work are trained by CIFAR-10~\citep{krizhevsky2009} or ImageNet~\citep{russakovsky2015}, taking EGG with $\eta=2$ as the noise distribution. 

We choose NP certification as the baseline since it is the state-of-the-art method for single distribution certification. The sampling distribution for the NP method is $\mathcal{P}=\mathcal{S}(\sigma, \eta)$ for ESG, and $\mathcal{P}=\mathcal{G}(\sigma, \eta, k)$ for EGG. For fairness, the sampling number $N$ is set to $100000$, with the significance level $\alpha=0.001$. 

In the double-sampling process, we set $k=1530$ and $k=75260$ for CIFAR-10 and ImageNet, respectively, in consistent with base classifiers. The threshold parameter $T$ for $\mathcal{Q}$ is determined by a heuristic algorithm from~\citet{li2022}. Specifically, we set 
\begin{equation}
\begin{aligned}
T_{\mathcal{S}} &= \sigma_s(2\Lambda^{-1}_{\frac{d}{\eta}}(\kappa))^{\frac{1}{\eta}} = \sigma\sqrt{\frac{d\Gamma(\frac{d}{\eta})}{\Gamma(\frac{d+2}{\eta})}}(\Lambda^{-1}_{\frac{d}{\eta}}(\kappa))^{\frac{1}{\eta}},\\
T_{\mathcal{G}} &= \sigma_g(2\Lambda^{-1}_{\frac{d-2k}{\eta}}(\kappa))^{\frac{1}{\eta}} = \sigma\sqrt{\frac{d\Gamma(\frac{d - 2k}{\eta})}{\Gamma(\frac{d-2k+2}{\eta})}}(\Lambda^{-1}_{\frac{d-2k}{\eta}}(\kappa))^{\frac{1}{\eta}},
\end{aligned}
\end{equation}
where $\kappa$ is determined by the heuristic algorithm~\citep{li2022}, a simple function of Monte Carlo sampling probability from $\mathcal{P}$. The sampling numbers $N_1, N_2$ are 50000, and the significance levels $\alpha_1, \alpha_2$ are 0.0005 for Monte Carlo sampling, equal for $\mathcal{P}$ and $\mathcal{Q}$. The error bound $e$ for certified radius is set at $1 \times 10^{-6}$. 

The settings of the exponent $\eta$ for ESG and EGG are slightly different. We choose $\eta\in\{1.0, 2.0, 4.0, 8.0\}$ as the exponent of ESG distributions. 
For EGG, we take $\mathcal{P}=\mathcal{G}(\sigma, \eta, k)$ and $\mathcal{Q}=\mathcal{G}_t(\sigma, \eta, k, T)$. To display the increasing trends more clearly, we choose $\eta\in\{0.25, 0.5, 1.0, 2.0, 4.0, 8.0\}$ for them. For the convenience of comparison, we fix the base classifier for each group of experiments, and all the sampling distributions keep $\mathbb{E}r^2$ the same with the base classifier following the setting of the previous work~\citep{yang2020, li2022}.

\textbf{Evaluation metrics.} We consider the $\ell_2$ certified radius in all experiments. To evaluate the certified robustness of smoothed classifiers, we take \emph{certified accuracy} $\triangleq CA(r, \mathcal{J})$ at radius $r$ on test dataset $\mathcal{J}$ as the basic metric~\citep{cohen2019, zhang2020, li2022}. For $(x_i, y_i)$ in $\mathcal{J}$, if the certified radius computed by a certification method (e.g. NP, DSRS) is $r_i$, and the output for $x_i$ through the smoothed classifier is $y_i$, we say $(x_i, y_i)$ is certified accurate at radius $r_i$ for the smoothed classifier. On this basis, $CA(r, \mathcal{J})$ is the ratio of examples in $\mathcal{J}$ whose certified radius $r_i\geq r$. Defined on the certified accuracy, \emph{Average Certified Radius} (ACR)~\citep{zhai2020, jeong2020} is another main metric that we use to show results of different distributions for real-world datasets. Formally, we have $ACR \triangleq \int_{r\geq0}CA(r, \mathcal{J})\cdot {\rm d}r$. 

\textbf{Integral methods.~}
The \texttt{scipy} package loses precision when calculating integrals for the $\Gamma(a, 1)$ distribution with large parameters (say, $a > 500$ ) on infinite intervals. To solve this problem, we implement a Linear Numerical Integration (LNI) method to compute the expectations fast and accurately based on Lemma \ref{lemmag1}. With this great property, we can compute the integral for the gamma distribution by only considering $1 - \iota$ total mass. Though primitive, we find the LNI method that uniformly segments the integration interval provides good precision for certifications on CIFAR-10 and ImageNet. For all the ESG experiments, we set the number of segments to $256$, and $\iota=10^{-4}$. We illustrate the effect of the segment number on CIFAR-10 and ImageNet in Figure \ref{effofseg}. 
For EGG distributions, we inherit the integral method from~\citet{li2022}, where we compute the integrals on the interval $(0, +\infty)$ by the $\texttt{scipy}$ package. The computational methods for ESG and EGG are left to Appendix \ref{computation}.

\begin{figure*}[t!]
	\centering
	\begin{subfigure}{0.24\linewidth}
		\centering
		\includegraphics[height=0.8\linewidth, width=1.0\linewidth]{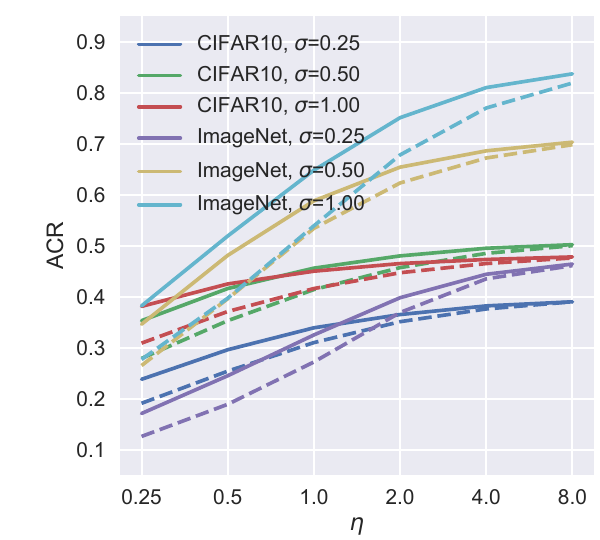}
		\caption{}\label{f4a}
	\end{subfigure}
	\centering
	\begin{subfigure}{0.24\linewidth}
		\centering
		\includegraphics[height=0.8\linewidth, width=1.0\linewidth]{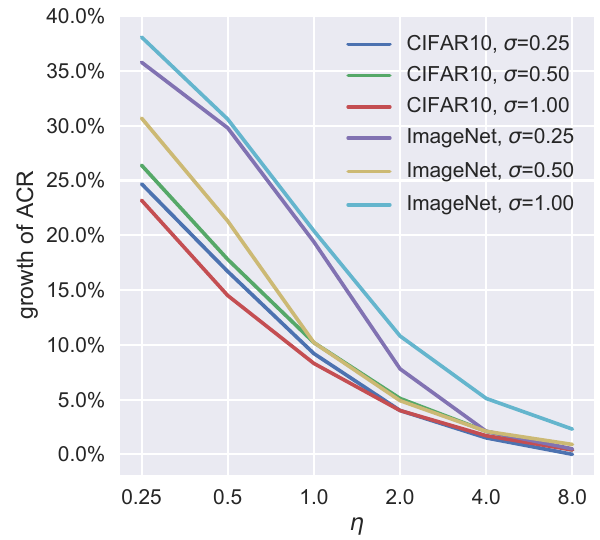}
		\caption{}\label{f4b}
	\end{subfigure}
    \centering
	\begin{subfigure}{0.24\linewidth}
		\centering
		\includegraphics[height=0.8\linewidth, width=1.0\linewidth]{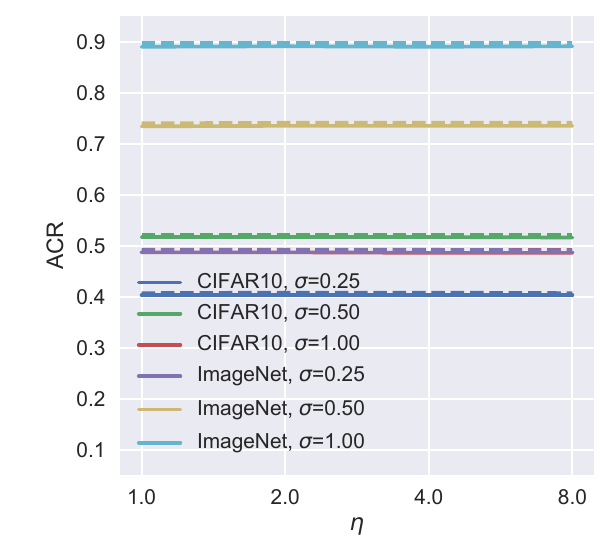}
		\caption{}\label{f4c}
	\end{subfigure}
    \centering
	\begin{subfigure}{0.24\linewidth}
		\centering
		\includegraphics[height=0.8\linewidth, width=1.0\linewidth]{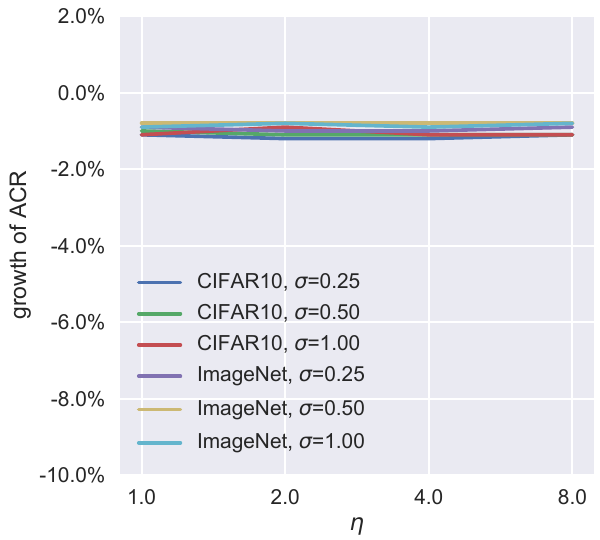}
		\caption{}\label{f4d}
	\end{subfigure}
	\caption{ACR results on real-world datasets. \textbf{(a).} ACR monotonically increases with $\eta$ in EGG. \textbf{(b).} The ACR growth gain from DSRS relative to NP shrinks with $\eta$ in EGG. \textbf{(c).} ACR stays almost constant in ESG. \textbf{(d).} The ACR growth gain from DSRS remains almost constant in ESG. For (a) and (c), solid lines represent results from DSRS, and dotted lines represent results from NP.}
	\label{demo_main}
\end{figure*}

\begin{figure}[t!]
  \centering
  \includegraphics[width=1.0\linewidth]{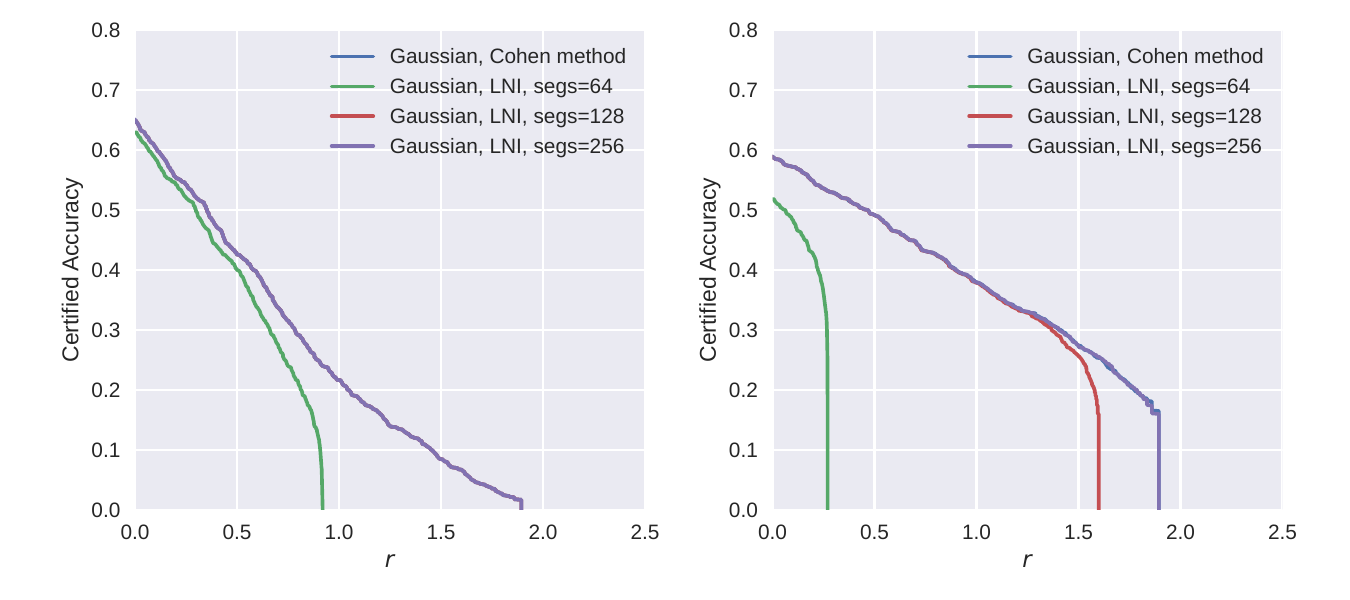}
  \caption{Results for the $\ell_2$ certified radius under different numbers of segment settings of LNI, both base clasifiers are standardly augmented by General Gaussian with $\sigma=0.50$. Left: For CIFAR-10, the curves show segs $\geq$ 128 is enough for CIFAR-10. Right: on ImageNet. The curves of Cohen's formula and segs = 256 are almost overlapped.} 
  \label{effofseg}
 \label{figmix}
 \vspace{-0.5cm} 
\end{figure}

\textbf{Experimental results.~}
We only discuss the effect of ESG and EGG in DSRS, since their influence on NP is similar. Table \ref{best_ESG_std} and Table \ref{best_EGG_std} report the maximum certified accuracy among base classifiers with $\sigma\in\{0.25, 0.50, 1.00\}$, which is widely adopted by previous work to show experimental results of certified robustness. To show the universality of our results, we provide experimental results on Consistency \citep{jeong2020} and SmoothMix \citep{jeong2021} models in Appendix \ref{supex}. We leave the detailed results on NP and DSRS in Appendix \ref{full_data}.

\emph{The exponent of ESG has little impact on the certification.}
As shown in Table \ref{best_ESG_std}, on both CIFAR10 and ImageNet, the certification provided by ESG distributions is highly insensitive to the alternation of $\eta$, which echos our theoretical analysis in Section \ref{ESG_theory}. Currently, the mainstream view believes Gaussian is the best distribution to provide the $\ell_2$ certified radius for RS, and our results show many ESG can provide the best as well.

\emph{The larger the exponent in EGG, the better the certification on real classifiers.}
Table \ref{best_EGG_std} reveals the phenomenon that certified accuracy at $r$ increases with the $\eta$ of EGG. On both CIFAR10 and ImageNet, our strategy to use EGG with a larger $\eta$ (8.0 in the tables) performs obviously better than General Gaussian (EGG with $\eta=2.0$) used in DSRS~\citep{li2022}. This is different from our theoretical analysis in \ref{EGG_theory}, because real classifiers do not perfectly satisfy the concentration property. Our Figure \ref{figrelax} has more details on this.

\emph{Better sampling probability + better probability utilization $\Rightarrow$ better certification.} The certified accuracy provided by (DS)RS is decided by two steps: sampling and certifying. For ESG, the stable results under changing $\eta$ indicate different ESG noise gives almost the same prediction accuracy on classifiers. Moreover, our results show that the larger $\eta$ in EGG offers both higher sampling accuracy and a better ability to use the probability. 

\emph{The improvement brought by EGG is finite.}
For EGG, though the certification improves with $\eta$, Figure \ref{f4b} reveals the growth brought by DSRS relative to NP shrinks with $\eta$, despite the fact that certified accuracy provided by DSRS keeps increasing. Furthermore, we see there is likely to be an upper bound for the DSRS certification: Gaussian's certification (in this work, $\eta=2$ for ESG). Additionally, the right of Figure \ref{primaldsrs} also implies the convergence. Therefore, chances are high that the growth of EGG's certification is not endless, and it is hard for EGG to challenge ESG's position as the optimal distribution in (DS)RS.

\section{Conclusion}
We report in this paper that the exponent $\eta$ in ESG is almost unable to affect the certification in RS and DSRS, from both theoretical and experimental perspectives. We derive the analytic sampling probability-certified radius relation under high-dimensional assumptions, creatively bridging the beta distribution and the normal distribution, and broadening the current optimal distribution in RS from Gaussian to many of ESG.
In addition, we find EGG distributions provide significant amelioration on the current solution to the curse of dimensionality on RS, and they can improve the $\ell_2$ certified radii of smoothed classifiers on real datasets. 
The working mechanism of EGG can be quite different, depending on whether the classifier satisfies a concentration property.

\section*{Impact Statement}
This paper is dedicated to advancing the certified robustness of machine learning, and we do not think any of its societal impacts must be highlighted here.

\section*{Acknowledgements}
This work is supported in part by the Overseas Research Cooperation Fund of Tsinghua Shenzhen International Graduate School (HW2021013).

\bibliography{final_paper}
\bibliographystyle{icml2024}

\newpage
\appendix
\onecolumn

\addcontentsline{toc}{section}{Appendix} 
\part{Appendix} 

\section{Supplementary for definitions of distributions}\label{appdis}
\subsection{Derivation for PDFs}
Let $\phi_s(r), \phi_s(r, T)$ be PDFs of $\mathcal{S}(\sigma, \eta)$ and $\mathcal{S}_t(\sigma, \eta, T)$ respectively. Given $\mathcal{S}(\sigma, \eta)\propto \exp\left(-\frac{r^\eta}{2\sigma_s^\eta}\right)$, we have
\begin{subequations}
\begin{align}
\phi_s(r) &= \frac{\eta}{2}\frac{1}{(2\sigma_s^\eta)^{\frac{d}{\eta}}\pi^{\frac{d}{2}}}\frac{\Gamma(\frac{d}{2})}{\Gamma(\frac{d}{\eta})}\exp(-\frac{1}{2}(\frac{r}{\sigma_s})^\eta),\label{67a}\\
\phi_s^{-1}(r) &=  \sigma_s\left(-2\ln\left(\frac{\Gamma(\frac{d}{\eta})}{\Gamma(\frac{d}{2})}2^{\frac{d}{\eta}+1}\sigma_s^d\pi^{\frac{d}{2}}\frac{r}{\eta}\right)\right) ^ {\frac{1}{\eta}},\label{phis}\\
\phi_s(r, T) &= \frac{\eta}{2}\frac{1}{(2\sigma_s^\eta)^{\frac{d}{\eta}}\pi^{\frac{d}{2}}}\frac{\Gamma(\frac{d}{2})}{\gamma(\frac{d}{\eta}, \frac{T^\eta}{2\sigma_s^\eta})}\exp(-\frac{1}{2}(\frac{r}{\sigma_s})^\eta), \label{67b}
\end{align}
\end{subequations}
and the ratio constant
\begin{equation}
C_s= \frac{\phi_s(r, T)}{\phi_s(r)} =\frac{\Gamma(\frac{d}{\eta})}{\gamma(\frac{d}{\eta}, \frac{T^\eta}{2\sigma_s^\eta})}.
\end{equation}

Similarly, since $\phi_g(r) \propto r^{-2k}\exp\left(-\frac{r^\eta}{2\sigma_g^\eta}\right)$, we have
\begin{subequations}
\begin{align}
\phi_g(r) &= \frac{\eta}{2}\frac{1}{(2\sigma_g^\eta)^{\frac{d-2k}{\eta}}\pi^{\frac{d}{2}}}\frac{\Gamma(\frac{d}{2})}{\Gamma(\frac{d-2k}{\eta})}r^{-2k}\exp(-\frac{1}{2}(\frac{r}
{\sigma_g})^\eta)\label{phigr}, \\
\phi_g^{-1}(r) &=  \sigma_g\left(\frac{4k}{\eta}W\left(\frac{\eta}{2k}\left(\frac{\Gamma(\frac{d-2k}{\eta})}{\Gamma(\frac{d}{2})}2^{\frac{d}{\eta} + 1}\sigma_g^d\pi^{\frac{d}{2}}\frac{r}{\eta}\right)^{-\frac{\eta}{2k}}\right)\right)^{\frac{1}{\eta}}.\label{phig}
\end{align}
\end{subequations}
In the equations above, $W(\cdot)$ is the principal branch of the Lambert W function. Let $\phi_g(r,T)$ be PDF of $\mathcal{G}_t(\sigma, \eta, k, T)$, then 
\begin{equation}
\phi_g(r, T) = \frac{\eta}{2}\frac{1}{(2\sigma_g^\eta)^{\frac{d-2k}{\eta}}\pi^{\frac{d}{2}}}\frac{\Gamma(\frac{d}{2})}{\gamma(\frac{d-2k}{\eta}, \frac{T^\eta}{2\sigma_g^\eta})}r^{-2k}\exp(-\frac{1}{2}(\frac{r}{\sigma_g})^\eta), 
\end{equation}
where $\gamma(\cdot)$ is the lower incomplete gamma function. We see
\begin{equation}
C_g = \frac{\phi_g(r, T)}{\phi_g(r)} =\frac{\Gamma(\frac{d-2k}{\eta})}{\gamma(\frac{d-2k}{\eta}, \frac{T^\eta}{2\sigma_g^\eta})}.
\end{equation}

\subsection{Visualization for distributions}
We show the distributions for $r$ in the space by setting
\begin{equation}
y = \phi_s(r) \cdot V_d(r)
	= \frac{\eta}{2}\frac{1}{(2\sigma_s^\eta)^{\frac{d}{\eta}}\pi^{\frac{d}{2}}}\frac{\Gamma(\frac{d}{2})}{\Gamma(\frac{d}{\eta})}\exp(-\frac{1}{2}(\frac{r}{\sigma_s})^\eta) \cdot \frac{d\pi^{\frac{d}{2}}}{\Gamma(\frac{d}{2} + 1)}r^{d-1}
\end{equation}
for ESG, and
\begin{equation}
y = \phi_g(r) \cdot V_d(r)
	= \frac{\eta}{2}\frac{1}{(2\sigma_g^\eta)^{\frac{d-2k}{\eta}}\pi^{\frac{d}{2}}}\frac{\Gamma(\frac{d}{2})}{\Gamma(\frac{d-2k}{\eta})}r^{-2k}\exp(-\frac{1}{2}(\frac{r}
{\sigma_g})^\eta) \cdot \frac{d\pi^{\frac{d}{2}}}{\Gamma(\frac{d}{2} + 1)}r^{d-1}
\end{equation}
for EGG. See the definition of $V_d(r)$ in Lemma \ref{vol}. We provide a simple demonstration for ESG and EGG in Figure \ref{demo_for_dis}. Overall, under high-dimensional settings, the difference between $\phi(r)$ of exponential Gaussian distributions can be huge, we thus suggest using $\phi(r)\cdot V_d(r)$ to assist in observing the properties of the distributions. As shown in Figure \ref{vis2}, ESG tends to concentrate on a \emph{thin shell} in the space, while the $r^{-2k}$ term in EGG alleviates that property.

\begin{figure}[t!]
	\centering
	\begin{subfigure}{0.49\linewidth}
		\centering
		\includegraphics[height=0.5\linewidth, width=1.0\linewidth]{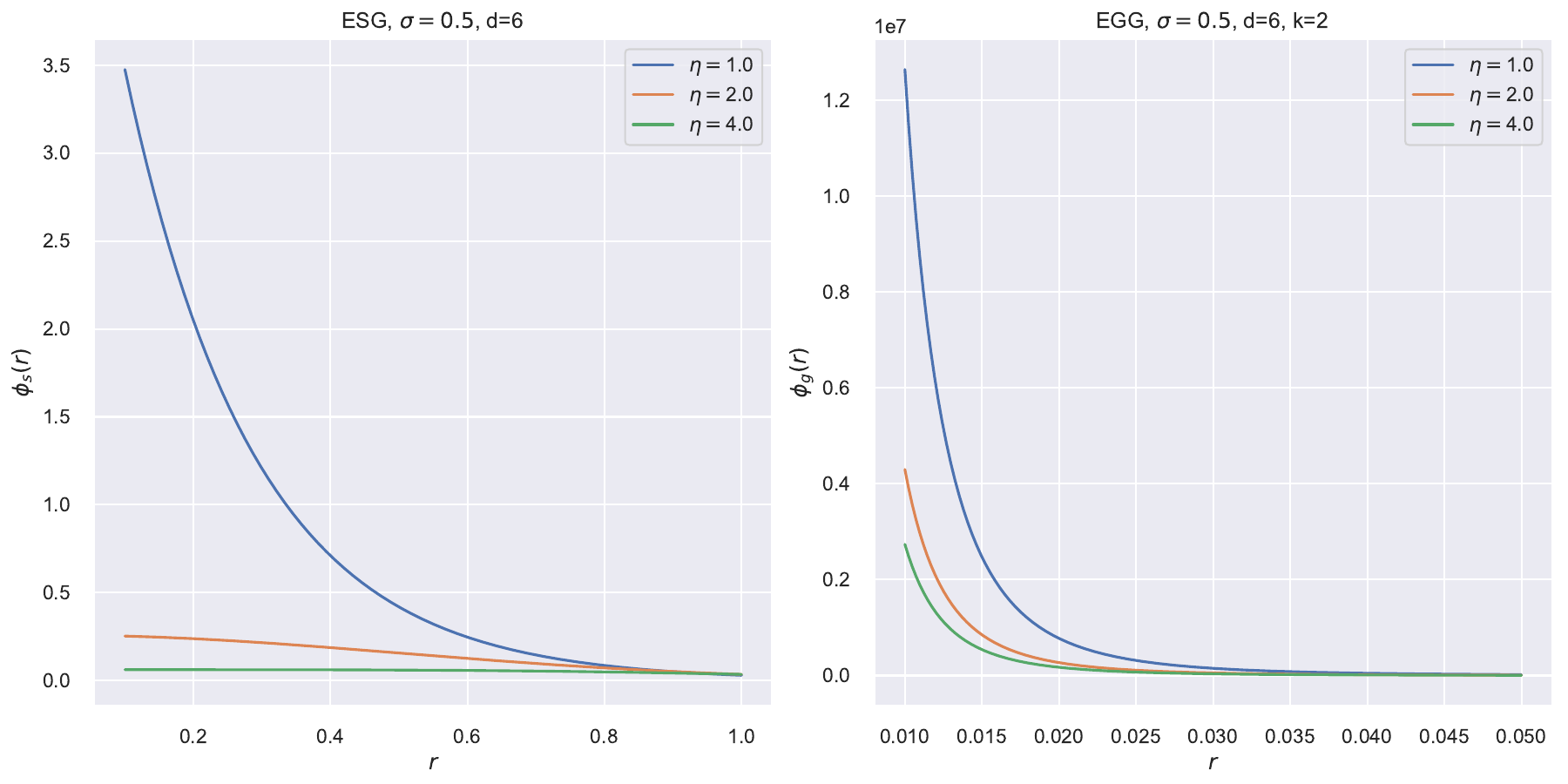}
		\caption{Visualization for $\phi(r)$.}\label{vis1}
	\end{subfigure}
	\centering
	\begin{subfigure}{0.49\linewidth}
		\centering
		\includegraphics[height=0.5\linewidth, width=1.0\linewidth]{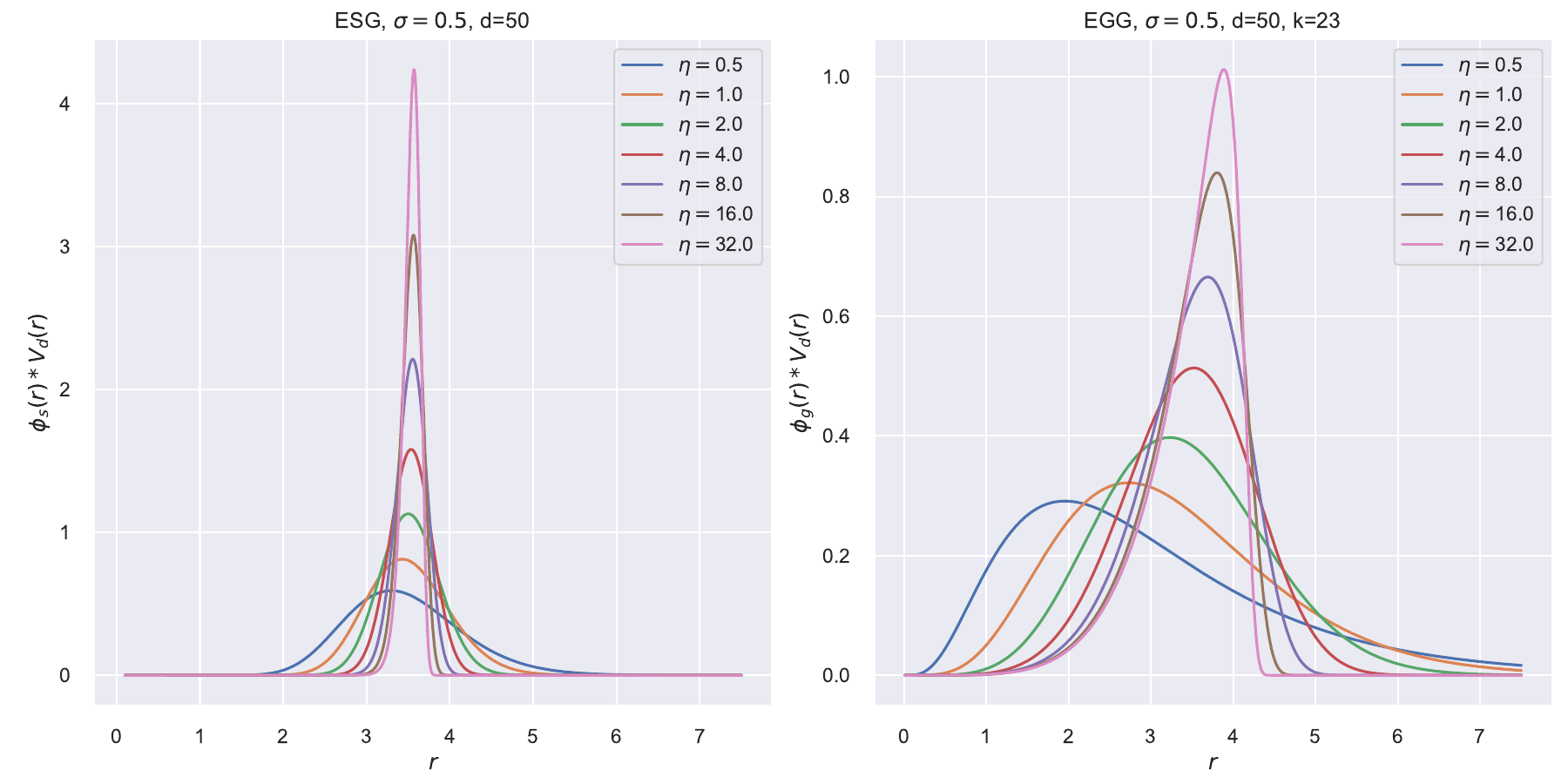}
		\caption{Visualization for $\phi(r)\cdot V_d(r)$.}\label{vis2}
	\end{subfigure}
	\caption{Demonstration for ESG and EGG.}
 \label{demo_for_dis}
\end{figure}

\section{Supplementary for the theoretical analysis of ESG}
\subsection{Proof of Theorem \ref{thesgNP}}\label{apesgNP}
The ESG distributions inherit the $\ell_2$-symmetry of the Gaussian distribution in high-dimensional space, thus the certified radius can be derived by the level-set method \citep{yang2020, li2022}. To begin with, we consider the problem for the Neyman-Pearson certification. Consider example $(x_0, y_0) \in \mathbb{R}^d \times \mathcal{Y}$. Let $\tilde{f}_{x_0}: \mathbb{R}^d\to \{0, 1\}$ be a true binary classifier of base classifier $f: \mathbb{R}^d \to \mathcal{Y}$ with respect to input $x_0$:
\begin{equation}\label{primal}
\begin{aligned}
\min_{\tilde{f}_{x_0}\in \mathcal{F}} \quad & \mathbb{E}_{z\sim \mathcal{S}_{(\sigma, \eta)}}\left(\tilde{f}_{x_0}(\delta + z)\right),\\
\textrm{s.t.} \quad & \mathbb{E}_{z\sim \mathcal{S}_{(\sigma, \eta)}}\left(\tilde{f}_{x_0}(z)\right) = A.
\end{aligned}
\end{equation}
We usually solve the problem above by transforming it into its dual problem:
\begin{subequations}\label{dual}
\begin{align}
\max_{\nu\in\mathbb{R}} \quad & \mathbb{P}_{z \sim \mathcal{S}_{(\sigma, \eta)}+\delta}\{p(z-\delta) + \nu p(z) < 0\},\label{npdual1}\\
\textrm{s.t.} \quad & \mathbb{P}_{z\sim \mathcal{S}_{(\sigma, \eta)}}\{p(z-\delta) + \nu p(z) < 0\} = A. \label{npdual2}
\end{align}
\end{subequations}
Unlike the DSRS case, the invariant set in Neyman-pearson certification is $\{z\mid p(z-\delta) + \nu p(z) < 0\}$, and we denote it as $\mathcal{V}$. We have the following lemma to calculate the volume of $d-1$ dimensional hypersphere.
\begin{lemma}\label{vol}
(Equation (16) in ~\citet{li2022}, Volume of Hypersphere, restated) Let $r \in \mathbb{R}_+, d \in \mathbb{N}_+$. The volume $V_d(r)$ of $d-1$ dimensional hypersphere with radius $r$ is 
\begin{equation}
V_d(r) = \frac{d\pi^{\frac{d}{2}}}{\Gamma(\frac{d}{2} + 1)}r^{d-1}.
\end{equation}
\end{lemma}
For all integrals in this paper, WLOG we let $d$ be an even number. 

\textbf{For Equation (\ref{npdual2})}, we have
\begin{equation}\label{esgnp1}
\begin{aligned}
&\mathbb{P}_{z\sim\mathcal{S}_{(\sigma, \eta)}}\{z\in\mathcal{V}\} \\
	=&\int_0^\infty \phi_s(r)V_d(r)dr\cdot\mathbb{P}\{x \in \mathcal{V}\mid\Vert x
 \Vert_2=r\} \\
	=&\frac{1}{\Gamma(\frac{d}{\eta})}\int_0^\infty u^{\frac{d}{\eta} - 1}\exp(-u)du\cdot\mathbb{P}\{x \in \mathcal{V}\mid\Vert x\Vert_2=\sigma_s(2u)^\frac{1}{\eta}\} \\
	=& \mathbb{E}_{u\sim\Gamma(\frac{d}{\eta}, 1)}\mathbb{P}\{x \in \mathcal{V}\mid\Vert x\Vert_2=\sigma_s(2u)^\frac{1}{\eta}\}
\end{aligned}
\end{equation}
Next, we consider $\omega_\natural(u,\nu)=\mathbb{P}\{p(x -\delta) + \nu p(x) < 0\mid\Vert x\Vert_2 = \sigma_s(2u)^\frac{1}{\eta}\} $. Obviously, $\phi_s(r)$ is a monotonically decreasing function with respect to $r$. Since all PDFs of $\mathcal{S}(\sigma, \eta)$ are bounded functions, we set the upper bound to be $U$ for convenience. Let $r = 0$, we have
\begin{equation}
U = \frac{\eta}{2}\frac{1}{(2\sigma_s^\eta)^{\frac{d}{\eta}}\pi^{\frac{d}{2}}}\frac{\Gamma(\frac{d}{2})}{\Gamma(\frac{d}{\eta})}\end{equation}
for $\phi_s(r)$. Namely, $\forall x \in [0, \infty),\ p(x) \in (0, U]$, where $U < +\infty$ is a constant for any determined distribution. As a result, if for a specific $x$, we have $-\nu p(x) > U$, the probability $\mathbb{P}\{p(x-\delta) + \nu p(x) < 0\mid\Vert x\Vert_2 = \sigma_s(2u)^\frac{1}{\eta}\}$ will always be 1. Next, we suppose $-\nu p(x) \in (0, U]$, otherwise $-\nu p(x)$ is outside the domain of $\phi^{-1}_s(x)$. When $\Vert x\Vert_2 = \sigma_s(2u)^\frac{1}{\eta}$, we have
\begin{equation}
\begin{aligned}
&p(x-\delta) + \nu p(x) < 0 \\
\Longleftrightarrow & \phi_s(\Vert x -\delta\Vert_2) \leq -\nu\phi_s(\Vert x\Vert_2) \\
\Longleftrightarrow & \Vert x -\delta\Vert_2 \geq \phi_s^{-1}(-\nu U \exp(-u)).
\end{aligned}	
\end{equation}
Then we are ready to solve
\begin{equation}
0 < -\nu U \exp(-u) \leq U, 
\end{equation}
where $u \geq 0$. Since $\nu$ is always negative, the left side $0 < -\nu U \exp(-u)$ always holds. For the right side, we notice
\begin{equation}\label{77}
-\nu U \exp(-u) \leq U \Longleftrightarrow \exp(-u) \leq -\frac{1}{\nu} \Longleftrightarrow u - \ln(-\nu) \geq 0.
\end{equation}
Now we know for $\nu < 0$, 
\begin{equation}
u - \ln(-\nu) \geq 0  \Longleftrightarrow -\nu p(x) \in (0, U],
\end{equation}
and
\begin{equation}
u - \ln(-\nu) <  0  \Longleftrightarrow -\nu p(x) \in (U, \infty],
\end{equation}
which means $\mathbb{P}\{p(x-\delta) + \nu p(x) < 0\mid\Vert x\Vert_2 = \sigma_s(2u)^\frac{1}{\eta}\} = 1$  when $u - \ln(-\nu) <  0$. Therefore, 
\begin{equation}\label{omeganat}
\omega_\natural(u, \nu) = \left\{
\begin{aligned}
&\Psi_{\frac{d-1}{2}}\left(\frac{(\rho + \sigma_s(2u)^\frac{1}{\eta})^2-\left(\phi_s^{-1}(-\nu\phi_s(\sigma_s(2u)^{\frac{1}{\eta}}))\right)^2}{4\rho \sigma_s(2u)^\frac{1}{\eta}}\right)&, u-\ln(-\nu) \geq 0,\\
&1&, u-\ln(-\nu) < 0.
\end{aligned}\right.\\
\end{equation}
\textbf{For Equation (\ref{npdual1})}, we have
\begin{equation}
\begin{aligned}
&\mathbb{P}_{z\sim\mathcal{S}_{(\sigma, \eta)}}\{z\in\mathcal{V}+\delta\} \\
	=&\int_0^\infty \phi_s(r)V_d(r)dr\cdot\mathbb{P}\{x \in \mathcal{V}+\delta\mid\Vert x
 \Vert_2=r\} \\
	=&\frac{1}{\Gamma(\frac{d}{\eta})}\int_0^\infty u^{\frac{d}{\eta} - 1}\exp(-u)du\cdot\mathbb{P}\{x \in \mathcal{V}+\delta\mid\Vert x\Vert_2=\sigma_s(2u)^\frac{1}{\eta}\} \\
	=& \mathbb{E}_{u\sim\Gamma(\frac{d}{\eta}, 1)}\mathbb{P}\{x \in \mathcal{V}+\delta\mid\Vert x\Vert_2=\sigma_s(2u)^\frac{1}{\eta}\}
\end{aligned}
\end{equation}
We write $\omega_\sharp$ for $\mathbb{P}\{p(x) + \nu p(x+\delta) < 0 \mid\Vert x\Vert_2=\sigma_g(2u)^\frac{1}{\eta}\}$. When $-\frac{1}{\nu}p(x) \in (0, U]$, we have 
\begin{equation}
\begin{aligned}\label{esgnp2}
&p(x) + \nu p(x+\delta) < 0 \\
\Longleftrightarrow &p(x) \leq  -\nu p(x + \delta) \\
\Longleftrightarrow &\Vert x + \delta \Vert_2 \leq \phi_s^{-1} (-\frac{1}{\nu}\phi_s(\Vert x \Vert_2)).
\end{aligned}
\end{equation}
Injecting $\Vert x\Vert_2 = \sigma_s(2u)^{\frac{1}{\eta}}$ into inequalities above, we get 
\begin{equation}
-\frac{1}{\nu}p(x) \in (0, U] \Longleftrightarrow u + \ln(-\nu) \geq 0, 
\end{equation}
which is the boundary condition for $\omega_\sharp$. Though looks similar, it differs significantly from $\omega_\natural$. Let $-\frac{1}{\nu}\phi_s(\Vert x \Vert_2) \in (U, +\infty]$, then 
\begin{equation}
p(x) + \nu p(x+\delta) < 0 \Longleftrightarrow \phi_s(\Vert x + \delta \Vert_2) \geq U,
\end{equation}
which means $\mathbb{P}\{p(x) + \nu p(x+\delta) < 0 \mid\Vert x\Vert_2=\sigma_g(2u)^\frac{1}{\eta}\} = 0$ under $u + \ln(-\nu) < 0$. Finally, we have
\begin{equation}\label{omegasharp}
\omega_\sharp(u, \nu) = \left\{
\begin{aligned}
&\Psi_{\frac{d-1}{2}}\left(\frac{\phi_s^{-1}(-\frac{1}{\nu}\phi_s(\sigma_s(2u)^{\frac{1}{\eta}}))^2\} - (\sigma_s(2u)^\frac{1}{\eta} - \rho)^2}{4\rho \sigma_s(2u)^\frac{1}{\eta}}\right)&, u + \ln(-\nu) \geq 0,\\
&0&, u + \ln(-\nu) < 0.\\
\end{aligned}\right.
\end{equation}
\textbf{Calculation of $\phi_s^{-1}(r)$.}  
Appearing in both $\omega_\natural$ and $\omega_\sharp$ functions above, $\phi_s^{-1}(r)$ is an indispensable value in the system. Here we consider the case of $\omega_{\natural}$. Let $\xi$ be $\phi_s^{-1}(-\nu \phi_s(\sigma_s(2u)^\frac{1}{\eta})$ when $-\nu \phi_s(\sigma_s(2u)^\frac{1}{\eta}\in(0,U]$, then
\begin{equation} \label{69}
\phi_s(\xi) = -\nu \phi_s(\sigma_s(2u)^{\frac{1}{\eta}}).
\end{equation}
We thus have
\begin{equation}\label{70}
 \begin{aligned}
&\frac{\eta}{2}\frac{1}{(2\sigma_s^\eta)^{\frac{d}{\eta}}\pi^{\frac{d}{2}}}\frac{\Gamma(\frac{d}{2})}{\Gamma(\frac{d}{\eta})}\exp(-\frac{\xi^\eta}{2\sigma_s^\eta}) = -\frac{\nu\eta}{2}\frac{1}{(2\sigma_s^\eta)^{\frac{d}{\eta}}\pi^{\frac{d}{2}}}\frac{\Gamma(\frac{d}{2})}{\Gamma(\frac{d}{\eta})}\exp(-u) \\
	  \Longleftrightarrow &\exp(-\frac{1}{2}(\frac{\xi}{\sigma_s})^\eta) = -\nu\exp(-u).\\
\end{aligned}
\end{equation}
Obviously $\xi \geq 0$, then $\exp(-\frac{1}{2}(\frac{\xi}{\sigma_s})^\eta) \in (0, 1]$. When Equation (\ref{70}) has a solution, we need
\begin{equation}
0 < -\nu\exp(-u) \leq 1.
\end{equation}
Therefore, we get the boundary condition for $ \nu \textless 0$:
\begin{equation} \label{72}
u \geq \ln(-\nu).
\end{equation}
Under Equation (\ref{72}), Equation (\ref{69}) can be solved:
\begin{equation}\label{73}
\xi = 2^{\frac{1}{\eta}}\sigma_s(u - \ln(-\nu))^{\frac{1}{\eta}}.\\
\end{equation}
Substituting Equation (\ref{73}) into Equation (\ref{omeganat}) gives Equation (\ref{onattxt}), and the case of $\omega_{\sharp}$ can be derived similarly. 

\subsection{Mildness of Assumption \ref{as1}}\label{accsig}
We first provide the proof for Lemma \ref{sigmasbound}.
\begin{proof}
 By Stirling's approximation, we see 
\begin{equation}
\begin{aligned}
\lim_{d\to\infty}\frac{\sigma_s}{d^{\frac{1}{2}-\frac{1}{\eta}}} &= 2^{-\frac{1}{\eta}}\sqrt{\frac{d^{\frac{2}{\eta}}\Gamma({\frac{d}{\eta}})}{\Gamma(\frac{d+2}{\eta})}}\sigma \\
&=2^{-\frac{1}{\eta}}\sqrt{d^{\frac{2}{\eta}}\sqrt{\frac{d-\eta}{d+2-\eta}}\frac{(\frac{d-\eta}{e\eta})^{\frac{d-\eta}{\eta}}}{(\frac{d+2-\eta}{e\eta})^{\frac{d+2-\eta}{\eta}}}}\\
&= (\frac{\eta}{2})^\frac{1}{\eta}\sigma,
\end{aligned}
\end{equation}
which by definition means $\sigma_s = \Theta(d^{\frac{1}{2} - \frac{1}{\eta}})$.
\end{proof}
Lemma \ref{sigmasbound} provides a simple but very accurate approximation for computing integrals in Theorem \ref{thesgNP}.
Moreover, when $d\gg\eta$ and $\sigma$ is small, the constant factor above is rather accurate. Let absolute error be $|\sigma_s - (\frac{\eta}{2})^\frac{1}{\eta}\sigma d^{\frac{1}{2} - \frac{1}{\eta}}|$, and relative error be $\frac{|\sigma_s - (\frac{\eta}{2})^\frac{1}{\eta}\sigma d^{\frac{1}{2} - \frac{1}{\eta}}|}{\sigma_s}$. Table \ref{3072err} and Table \ref{150224err} show the errors for the approximation. From the tables, we observe that under common settings of variance and dimension for randomized smoothing, the approximation shows great accuracy. On CIFAR10, the maximum absolute error is $2.95 \times 10^{-3}$, and the maximum relative error is $4.88 \times 10^{-4}$. Furthermore, the errors are even lower on ImageNet, as we see that the maximum absolute error is $2.59 \times 10^{-4}$ and the maximum relative error is no more than $9.99\times10^{-6}$. It is reasonable that a larger dimension ensures less error for approximation since there are only two sources for the error: one is Stirling's approximation, which demands large $d$ to be precise, and another is that $d$ should be far greater than $\eta$, since larger $d$ benefits more since we fix $\eta\in\{1.0, 2.0, 4.0, 8.0\}$ in our work.
\begin{table}[htbp]
\scriptsize
\vspace{-5mm}
\centering
\caption{Errors for the approximation of $\sigma_s$, $d=3072$}\label{3072err}
\begin{tabular}{ccccccccccc}
\toprule
\multirow{2}{*}{$\sigma$} & \multicolumn{2}{c}{$\eta=0.5$} & \multicolumn{2}{c}{$\eta=1.0$} & \multicolumn{2}{c}{$\eta=2.0$} & \multicolumn{2}{c}{$\eta=4.0$} & \multicolumn{2}{c}{$\eta=8.0$} \\
\cmidrule(lr){2-3} \cmidrule(lr){4-5} \cmidrule(lr){6-7} \cmidrule(lr){8-9} 
\cmidrule(lr){10-11}
                          & AE             & RE            & AE             & RE            & AE             & RE            & AE             & RE            & AE             & RE            \\
\midrule
$0.12$                    &2.15e-11 & 4.88e-04 & 1.76e-07 & 1.63e-04 & 9.88e-14 & 8.23e-13 & 8.65e-05 & 8.14e-05 & 3.54e-04 & 1.22e-04 \\
$0.25$                    &4.48e-11 & 4.88e-04 & 3.67e-07 & 1.63e-04 & 2.06e-13 & 8.23e-13 & 1.80e-04 & 8.14e-05 & 7.37e-04 & 1.22e-04  \\
$0.50$                    &8.96e-11 & 4.88e-04 & 7.34e-07 & 1.63e-04 & 4.12e-13 & 8.23e-13 & 3.60e-04 & 8.14e-05 & 1.47e-03 & 1.22e-04 \\
$1.00$                    &1.79e-10 & 4.88e-04 & 1.47e-06 & 1.63e-04 & 8.23e-13 & 8.23e-13 & 7.21e-04 & 8.14e-05 & 2.95e-03 & 1.22e-04 \\     
\bottomrule
\end{tabular}
\end{table}

\begin{table}[htbp]
\scriptsize
\vspace{-5mm}
\centering
\caption{Errors for the approximation of $\sigma_s$, $d=150224$}\label{150224err}
\begin{tabular}{ccccccccccc}
\toprule
\multirow{2}{*}{$\sigma$} & \multicolumn{2}{c}{$\eta=0.5$} & \multicolumn{2}{c}{$\eta=1.0$} & \multicolumn{2}{c}{$\eta=2.0$} & \multicolumn{2}{c}{$\eta=4.0$} & \multicolumn{2}{c}{$\eta=8.0$} \\
\cmidrule(lr){2-3} \cmidrule(lr){4-5} \cmidrule(lr){6-7} \cmidrule(lr){8-9} 
\cmidrule(lr){10-11}
                          & AE             & RE            & AE             & RE            & AE             & RE            & AE             & RE            & AE             & RE            \\
\midrule
$0.12$                    &1.29e-15 & 9.99e-06 & 5.15e-10 & 3.33e-06 & 4.46e-12 & 3.72e-11 & 4.68e-06 & 1.66e-06 & 3.11e-05 & 2.50e-06  \\
$0.25$                    &2.68e-15 & 9.99e-06 & 1.07e-09 & 3.33e-06 & 9.30e-12 & 3.72e-11 & 9.74e-06 & 1.66e-06 & 6.48e-05 & 2.50e-06   \\
$0.50$                    &5.36e-15 & 9.99e-06 & 2.15e-09 & 3.33e-06 & 1.86e-11 & 3.72e-11 & 1.95e-05 & 1.66e-06 & 1.30e-04 & 2.50e-06  \\
$1.00$                    &1.07e-14 & 9.99e-06 & 4.29e-09 & 3.33e-06 & 3.72e-11 & 3.72e-11 & 3.90e-05 & 1.66e-06 & 2.59e-04 & 2.50e-06  \\     
\bottomrule
\end{tabular}
\end{table}

\subsection{Mildness of Assumption \ref{as2}}\label{mildas2}
\begin{figure}[htbp!]
  \centering
  \includegraphics[width=0.5\linewidth]{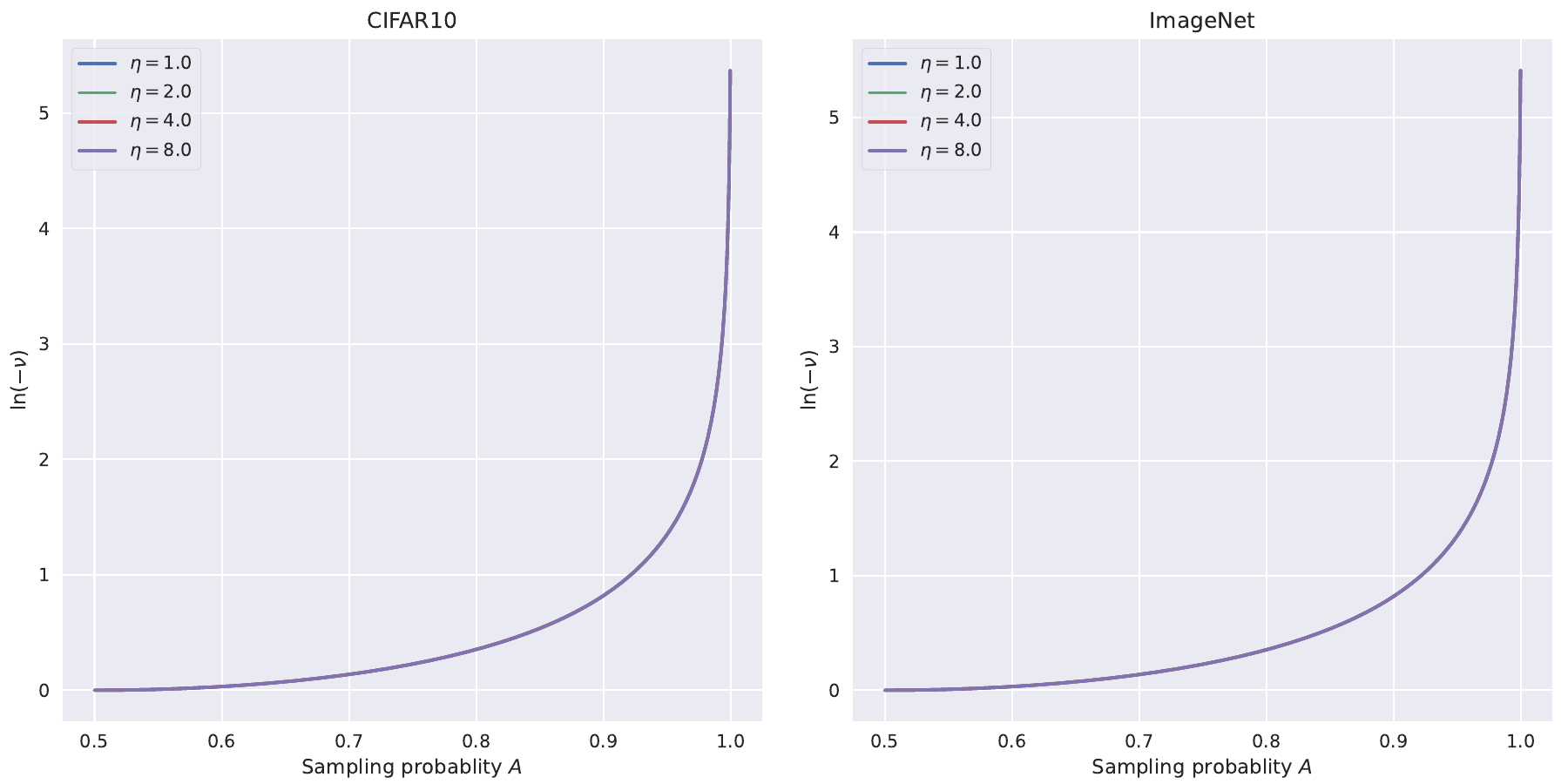}
  \caption{Under commonly used settings for RS, $\frac{d}{\eta}\gg \ln(-\nu)$ almost always holds.}
 \label{figfornu}
 \vspace{-5mm}
\end{figure}
In addition to Equation (\ref{approxnu}), we show Assumption \ref{as2} is very mild for practical use in RS computation. In Figure \ref{figfornu}, all the $\ln(-\nu)$ are computed by the conventional procedure by binary search without any approximation. For each subfigure, we uniformly select 1000 points in $(0.5, 1)$ to simulate the sampling probability $A$. The maximum $A$ we show in the figure is $0.9995$, which corresponds to $\ln(-\nu)\approx 5.4$ for the listed ESG distributions. Intuitively, the value of $\ln(-\nu)$ will increase exponentially when $A\to1^{-}$, but it is more than enough for the practical use of the RS framework (the confident lower bound of $A$ can hardly ever reach 1), especially for ImageNet. 
 
\subsection{Extensions for Equation (\ref{approximation})} \label{extapp}
Equation (\ref{approximation}) is an accurate approximation for \citet{cohen2019}'s formula not only for little $\rho$ that appears in the RS certification. For this, we consider
\begin{equation}
    A=\Psi_{\frac{d-1}{2}}(\frac{1}{2} + \frac{x}{2\sqrt{d}}),
\end{equation}
where $x\in(0, \sqrt{d})$. Here we provide a simple proof for the convergence in Equation (\ref{approximation}). 

\begin{theorem}
(Convergence of the Beta($d,d$) distribution's CDF) Let $\Psi_d(x)$ be the CDF of the beta($d,d$) distribution and $\Phi(x)$ be the CDF of the standard normal distribution, then for $x\in(-\sqrt{2d}, \sqrt{2d})$, $\lim_{d\to\infty}\Psi_d(\frac{1}{2} + \frac{x}{2\sqrt{2d}}) = \Phi(x)$, with an $O(1/\sqrt{d})$ error bound .
\end{theorem}
\begin{proof}
    Let $\psi_d(x)$ be the PDF of the $Beta(d, d)$ distribution and $\phi(x)$ be the PDF of the standard normal distribution. By \citet{Ryder2012}, we have 
    \begin{equation}
        \psi_d(x) = \phi(x)(1 + O(\frac{1}{d})).
    \end{equation}
    WLOG, we consider appropriate real numbers $c_1, c_2\in\mathbb{R}_+$, such that
     \begin{equation}
    (1-\frac{c_1}{d})\phi(x) < \psi_d(x) < (1+\frac{c_2}{d})\phi(x).
    \end{equation}
    Then we get
    \begin{equation}
    \lim_{d\to+\infty}\int_{-\sqrt{2d}}^{x}(1-\frac{c_1}{d})\phi(t)dt
    < \lim_{d\to+\infty}\int_{-\sqrt{2d}}^{x}\psi_d(t)dt
    <  \lim_{d\to+\infty}\int_{-\sqrt{2d}}^{x}(1+\frac{c_2}{d})\phi(t)dt.
    \end{equation}
    By $|\phi(x)| < \frac{1}{\sqrt{2\pi}}$, we have
    \begin{equation}
         \lim_{d\to+\infty}(\int_{-\sqrt{2d}}^{x}\phi(t)dt - \frac{(x + \sqrt{2d})c_1}{\sqrt{2\pi}d})  < \lim_{d\to+\infty}\int_{-\sqrt{2d}}^{x}\psi_d(t)dt < \lim_{d\to+\infty}(\int_{-\sqrt{2d}}^{x}\phi(t)dt + \frac{(x + \sqrt{2d})c_2}{\sqrt{2\pi}d}).
    \end{equation}
    Thus for $x\in(-\sqrt{2d}, \sqrt{2d})$, we have
    \begin{equation}
   \Phi(x)-\lim_{d\to+\infty}\frac{2c_1}{\sqrt{\pi d}}< \lim_{d\to+\infty}\int_{-\sqrt{2d}}^{x}\psi_d(t)dt <  \Phi(x)+\lim_{d\to+\infty}\frac{2c_2}{\sqrt{\pi d}},
   \end{equation}
   which by definition means 
   \begin{equation}
   \lim_{d\to\infty}\Psi_d(\frac{1}{2} + \frac{x}{2\sqrt{2d}}) = \Phi(x),
   \end{equation}
   and the error bound is $O(1/\sqrt{d})$.
\end{proof}

In addition, we show $\Psi_{\frac{d-1}{2}}(\frac{1}{2} + \frac{x}{2\sqrt{d}})$ and $\Phi(x)$ in the same figure for $d\in[10^1, 10^2, 10^3, 10^4, 10^5, 10^6]$ in Figure \ref{extend_psi}. From the figure, only when $d=10$ can we perceive the errors with eyes. Furthermore, to see the error more clearly, we uniformly select 100000 points in the interval $(0, \sqrt{d})$, and report the maximum absolute error and the maximum relative error among all the points in Table \ref{extension_table}, where we see both errors match the $O(1/d)$ error bound in \citet{Ryder2012}, slightly tighter than our proved $O(1/\sqrt{d})$ bound.

\begin{figure}[htbp!]
  \centering
  \includegraphics[width=1.0\linewidth]{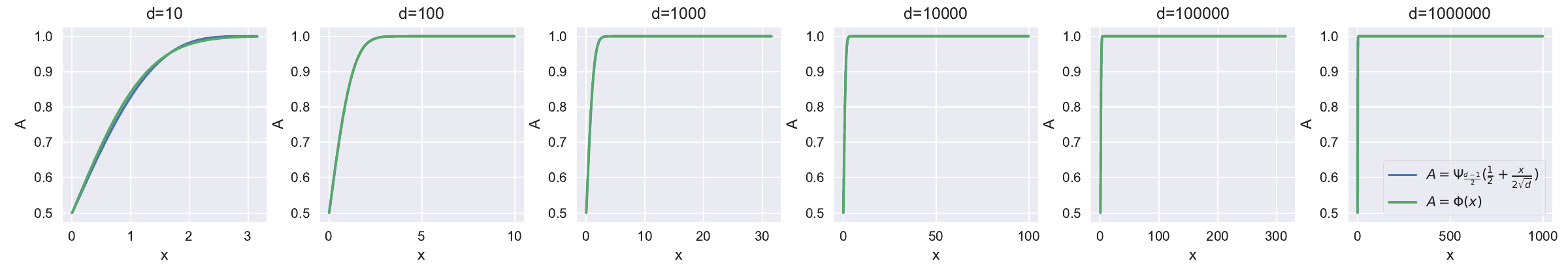}
  \caption{Comparison of $\Psi_{\frac{d-1}{2}}(\frac{1}{2} + \frac{x}{2\sqrt{d}})$ and $\Phi(x)$ for different $d$.}
 \label{extend_psi}
 \vspace{-5mm}
\end{figure}

\begin{table}[htbp!]
\centering
\caption{Maximum absolute errors and maximum relative errors in 100000 points in $(0, \sqrt{d})$ for different $d$}\label{extension_table}
\begin{tabular}{ccccccc}
\toprule
$d$  & $10^1$  & $10^2$  & $10^3$  & $10^4$  & $10^5$  & $10^6$  \\
\midrule
AE & 1.46e-2 & 1.38e-3 & 1.38e-4 & 1.38e-5 & 1.38e-6 & 1.38e-7 \\
RE & 1.93e-2 & 1.83e-3 & 1.82e-4 & 1.82e-5 & 1.82e-6 & 1.82e-7 \\
\bottomrule
\end{tabular}
\end{table}

\section{Computational methods for ESG and EGG in DSRS}\label{computation}
To evaluate the performance of ESG/EGG on real-world datasets, we consider solving Problem~(\ref{dualdsrs}), the strong dual problem of Problem (\ref{primaldsrs}):
\begin{small}
\begin{equation}\label{dualdsrs}
\begin{aligned}
\max_{\nu_1, \nu_2 \in\mathbb{R}} \quad & \mathbb{P}_{z \sim \mathcal{P+\delta}}\{p(z-\delta) + \nu_1 p(z) + \nu_2 q(z) < 0\},\\
\textrm{s.t.} \quad & \mathbb{P}_{z\sim \mathcal{P}}\{p(z-\delta) + \nu_1 p(z)  + \nu_2 q(z)< 0\} = A,\\
& \mathbb{P}_{z\sim \mathcal{Q}}\{p(z-\delta) + \nu_1 p(z)  + \nu_2 q(z)< 0\} = B.
\end{aligned}
\end{equation}
\end{small}

We do not elaborate on the solution process since previous works \citep{yang2020, li2022} have done well.
Herein, we directly give solutions to Problem (\ref{dualdsrs}) for ESG and EGG.
Overall, the case of EGG is a relatively straightforward generalization of General Gaussian, while the derivation for ESG includes nontrivial branches. When taking ESG as the smoothing distribution and TESG as the supplementary distribution, we solve Problem (\ref{dualdsrs}) by the following theorem:

\begin{theorem}[Integral form of Problem (\ref{dualdsrs}) with ESG]\label{thesg}
Let $\mathcal{P}= \mathcal{S}(\sigma, \eta)$ with PDF $p(\cdot)$, $\mathcal{Q}= \mathcal{S}_t(\sigma, \eta, T)$ with PDF $q(\cdot)$ and $C_s =  \frac{\Gamma(\frac{d}{\eta})}{\gamma(\frac{d}{\eta}, \frac{T^\eta}{2\sigma_s^\eta})}$, where $\gamma(\cdot, \cdot)$ is the lower incomplete gamma function. Let $\mathcal{W}\triangleq \{z\mid p(z-\delta) + \nu_1 p(z) + \nu_2 q(z) < 0\}$. Then
\begin{equation}
\scriptsize
\begin{aligned}
\mathbb{P}_{z\sim \mathcal{P}}\{z\in \mathcal{W}\} &= \mathbb{E}_{u \sim \Gamma(\frac{d}{\eta}, 1)}\left\{
	\begin{aligned}
       		&\omega_1(u, \nu_1), &u\geq \frac{T^\eta}{2\sigma_s^\eta}, \\
		&\omega_1(u, \nu_1 + C_s\nu_2), &u < \frac{T^\eta}{2\sigma_s^\eta},
	 \end{aligned}\right.\\
 \mathbb{P}_{z\sim \mathcal{Q}}\{z\in \mathcal{W}\} &= C_s\mathbb{E}_{u \sim \Gamma(\frac{d}{\eta}, 1)}\ \omega_1(u, \nu_1 + C_s\nu_2) \cdot \mathds{1}_{u\leq \frac{T^\eta}{2\sigma_s^\eta}},
\\
 \mathbb{P}_{z\sim \mathcal{P+\delta}}\{z\in \mathcal{W}\} &= \left\{
 \begin{aligned}
 	&\mathbb{E}_{u \sim \Gamma(\frac{d}{\eta}, 1)}\ \omega_2(u), &\nu_1 \geq 0,\\
	&\mathbb{E}_{u \sim \Gamma(\frac{d}{\eta}, 1)}\{\omega_2(u) + \omega_3(u)\}, &\nu_1 < 0,
 \end{aligned}\right.\\
 \end{aligned}
\end{equation}
where
\begin{tiny}
\begin{equation}
\begin{aligned}
&\omega_1(u, \nu) = \left\{
\begin{aligned}
&\Psi_{\frac{d-1}{2}}\left(\frac{(\rho + \sigma_s(2u)^\frac{1}{\eta})^2-2^{\frac{2}{\eta}}\sigma_s^2(u - \ln(-\nu))^{\frac{2}{\eta}}}{4\rho \sigma_s(2u)^\frac{1}{\eta}}\right)&, u-\ln(-\nu) \geq 0,\\
&1&, u-\ln(-\nu) < 0,
\end{aligned}\right.\\
&\omega_2(u) = \left\{
\begin{aligned}
&\Psi_{\frac{d-1}{2}}\left(\frac{\min \{T^2, 2^{\frac{2}{\eta}}\sigma_s^2(u + \ln(-\nu_1 -C_s\nu_2))^{\frac{2}{\eta}}\} - (\sigma_s(2u)^\frac{1}{\eta} - \rho)^2}{4\rho \sigma_s(2u)^\frac{1}{\eta}}\right)&, u + \ln(-(\nu_1 + C_s\nu_2)) \geq 0,\\
&0&, u + \ln(-(\nu_1 + C_s\nu_2)) < 0,\\
\end{aligned}\right.\\
&\omega_3(u) = \left\{
\begin{aligned} 
&\max\left\{\Psi_{\frac{d-1}{2}}\left(\frac{2^{\frac{2}{\eta}}\sigma_s^2(u + \ln(-\nu_1))^{\frac{2}{\eta}} - (\sigma_s(2u)^\frac{1}{\eta} - \rho)^2}{4\rho \sigma_s(2u)^\frac{1}{\eta}}\right) - \Psi_{\frac{d-1}{2}}\left(\frac{T^2-(\sigma_s(2u)^\frac{1}{\eta} - \rho)^2}{4\rho \sigma_s(2u)^\frac{1}{\eta}}\right), 0 \right\}&, u+ \ln(-\nu_1) \geq 0, \\
&0&, u + \ln(-\nu_1) < 0.
\end{aligned}\right.
 \end{aligned}
\end{equation}
\end{tiny}
\end{theorem}

Similarly, We have the solution to Problem (\ref{dualdsrs}) for EGG as follows:

\begin{theorem}[Integral form of Problem (\ref{dualdsrs}) with EGG]\label{thegg}
Let $\mathcal{P}= \mathcal{G}(\sigma, \eta, k)$ with PDF $p(\cdot)$, $\mathcal{Q}= \mathcal{G}_t(\sigma, \eta, k, T)$ with PDF $q(\cdot)$ and $C_g = \frac{\Gamma(\frac{(d-2k)}{\eta})}{\gamma(\frac{d-2k}{\eta}, \frac{T^\eta}{2\sigma_g^\eta})}$, where $\gamma(\cdot, \cdot)$ is the lower incomplete gamma function. Let $\mathcal{W}\triangleq \{z\mid p(z-\delta) + \nu_1 p(z) + \nu_2 q(z) < 0\}$. Then
\begin{equation}
\begin{scriptsize}
\begin{aligned}
\mathbb{P}_{z\sim \mathcal{P}}\{z\in \mathcal{W}\} &= \mathbb{E}_{u \sim \Gamma(\frac{d-2k}{\eta}, 1)}\left\{
\begin{aligned}
       	&\omega_1(u, \nu_1), &u\geq \frac{T^\eta}{2\sigma_g^\eta}, \\
	&\omega_1(u, \nu_1 + C_g\nu_2), &u < \frac{T^\eta}{2\sigma_g^\eta},
 \end{aligned}\right.\\
 \mathbb{P}_{z\sim \mathcal{Q}}\{z\in \mathcal{W}\} &= 
 C_g\mathbb{E}_{u \sim \Gamma(\frac{d-2k}{\eta}, 1)}\ \omega_1(u, \nu_1 + C_g\nu_2) \cdot \mathds{1}_{u\leq \frac{T^\eta}{2\sigma_g^\eta}}, \\
 \mathbb{P}_{z\sim \mathcal{P+\delta}}\{z\in \mathcal{W}\} &= \left\{
 \begin{aligned}
 	&\mathbb{E}_{u \sim \Gamma(\frac{d-2k}{\eta}, 1)}\ \omega_2(u), &\nu_1 \geq 0,\\
	&\mathbb{E}_{u \sim \Gamma(\frac{d-2k}{\eta}, 1)}\{\omega_2(u) + \omega_3(u)\}, &\nu_1 < 0,
 \end{aligned}\right.\\
 \end{aligned}
 \end{scriptsize}
\end{equation}
where
\begin{scriptsize}
\begin{equation}
\begin{aligned}
\omega_1(u, \nu) &= \Psi_{\frac{d-1}{2}}\left(\frac{(\rho + \sigma_g(2u)^\frac{1}{\eta})^2- \left(\frac{4kW(\frac{\eta u}{2k}(-\nu)^{\frac{\eta}{2k}}\exp(\frac{\eta u}{2k}))\sigma_g^\eta}{\eta}\right)^{\frac{2}{\eta}}}{4\rho\sigma_g (2u)^\frac{1}{\eta}}\right), \\
\omega_2(u) &= \Psi_{\frac{d-1}{2}}\left(\frac{\min \{T^2, \sigma_g^2(\frac{4kW(\frac{\eta u}{2k}(-\nu_1 - C_g\nu_2)^{\frac{\eta }{2k}}\exp(\frac{\eta u}{2k}))}{\eta})^{\frac{2}{\eta}}\} - (\sigma_g(2u)^\frac{1}{\eta} - \rho)^2}{4\rho \sigma_g(2u)^\frac{1}{\eta}}\right),\\
\omega_3(u) &= \Psi_{\frac{d-1}{2}}\left(\frac{\sigma_g^2(\frac{4kW(\frac{\eta u}{2k}(-\nu_1)^{\frac{\eta}{2k}}\exp(\frac{\eta u}{2k}))}{\eta})^{\frac{2}{\eta}} - (\sigma_g(2u)^\frac{1}{\eta} - \rho)^2}{4\rho \sigma_g(2u)^\frac{1}{\eta}}\right) - \Psi_{\frac{d-1}{2}}\left(\frac{T^2-(\sigma_g(2u)^\frac{1}{\eta} - \rho)^2}{4\rho \sigma_g(2u)^\frac{1}{\eta}}\right).
 \end{aligned}
\end{equation}
\end{scriptsize}
\end{theorem}

Compared to Theorem \ref{thegg}, we emphasize the following differences: (1) $p(\cdot)$ represents different PDFs in the two theorems; (2) there are more branches for $\omega$ functions in Theorem \ref{thesg}, which originates from different properties of the logarithmic function and the Lambert W function. 

\section{Proof of Theorem \ref{fixbase} and Theorem \ref{thcorres}}\label{appth}
We put proofs of Theorem \ref{fixbase} and Theorem \ref{thcorres} together since they share common thinking. Overall, this section includes 3 parts. We first introduce lemmas based mainly on the concentration properties of beta and gamma distributions, then derive the solution for a lower bound of Problem (\ref{primaldsrs}). Finally, we prove Theorem \ref{fixbase} and Theorem \ref{thcorres} respectively, on the basis of those introduced lemmas. Both proofs are essentially generalizations for appendix F.3  in~\citet{li2022}.
 
\subsection{Preliminaries}
We show lemmas for proofs in this section. In a nutshell, Lemma \ref{lemmae1} offers the probability that the mass of ESG is within $T$; Lemma \ref{lemmae2} gives a solution for a lower bound of Problem (\ref{primaldsrs}); Lemma~\ref{lemmae3} and Lemma \ref{lemmae4} reveal concentration properties of beta and gamma distributions; Lemma \ref{lemmae5} proves monotonicity of a function that appears in the proof of Theorem \ref{thcorres}.

\begin{lemma}\label{lemmae1}
For a random variable $z \sim \mathcal{S}(\sigma, \eta)$ and a determined threshold $T \in \mathbb{R}_+$,  
\begin{equation}
	\mathbb{P}\{\Vert z\Vert _2\leq T\}=\Lambda_{\frac{d}{\eta}}(\frac{T^\eta}{2\sigma_s^\eta}).
\end{equation}
\end{lemma}
	
	\begin{proof} We have $\phi_s$, the PDF of $z$ from Equation (\ref{67a}) and $V_d(r)$ from Lemma \ref{vol}, then
	\begin{equation}
	\begin{aligned}
	\mathbb{P}\{\Vert z\Vert _2\leq T\} &=\int_0^T\phi_s(r)\exp(-\frac{r^\eta}{2\sigma_s^\eta})\frac{d\pi^{\frac{d}{2}}}{\Gamma(\frac{d}{2} + 1)}r^{d-1}dr \\
	&= \frac{1}{\Gamma(\frac{d}{\eta})}\int_0^{\frac{T^\eta}{2\sigma_s^\eta}}t^{\frac{d}{\eta}-1}\exp(-t)dt \\
	&= \Lambda_{\frac{d}{\eta}}(\frac{T^\eta}{2\sigma_s^\eta}). 
	\end{aligned}
	\end{equation}
\end{proof}

\begin{lemma}\label{lemmaforq}
Given a base classifier $f:\mathbb{R}^d \to \mathcal{Y}$ satisfies $(\sigma, p, \eta)$-concentration property at input $(x_0,y_0) \in \mathbb{R}^d$. For the supplementary distribution $\mathcal{Q}=\mathcal{G}_t(\sigma, \eta, k, T)$, where $T=\sigma_s(2\Lambda_{\frac{d}{\eta}}^{-1}(p))^{\frac{1}{\eta}}, \eta \in \mathbb{R}_+$ and $2d-k\in[1,30] \cap \mathbb{N}$, we have
\begin{equation}
\mathbb{P}_{z\sim\mathcal{Q}}\{f(x_0 + z) = y_0\} = 1.
\end{equation}
\end{lemma}
\begin{proof} If $f$ satisfies the $(\sigma, p, \eta)$-concentration property, when $p$ is fixed, we have 
\begin{equation}\label{eq24}
\mathbb{P}_{z\sim S(\sigma, \eta)}\{f(x_0+z)=y_0\mid \Vert z\Vert_2\leq T\}=1
\end{equation}
for $T=\sigma_s(2\Lambda_{\frac{d}{\eta}}^{-1}(p))^{\frac{1}{\eta}}$ from Definition \ref{cct} and Lemma \ref{lemmae1}. Notice that though Equation (\ref{eq24}) is defined by $S(\sigma_s, \eta)$, $f(x_0+z)=y_0$ holds for almost all $\Vert z\Vert_2\leq T$ since distribution $\mathcal{S}$ has positive density almost everywhere. Consider the case $z\sim \mathcal{G}_t(\sigma, \eta, k, T)$; we thereby have
\begin{equation}
\mathbb{P}_{z \sim \mathcal{G}_t(\sigma, \eta, k, T)}\{f(x_0+z)\neq y_0\mid 0<\Vert z\Vert_2\leq T\}=0.
\end{equation}
By Equation (\ref{phigr}), we have
\begin{equation}
\begin{aligned}
&\mathbb{P}_{z\sim\mathcal{G}_t(\sigma, \eta, k, T)}\{z=0\}\\
=&\mathbb{P}_{z\sim\mathcal{G}(\sigma, \eta, k)}\{z=0\mid\Vert z\Vert_2\leq T\}\\
\leq& \lim_{t\to 0}\mathbb{P}_{z\sim \mathcal{G}(\sigma, \eta, k)}\{\Vert z\Vert_2\leq t\}\\
=&\lim_{t\to 0}\int_0^t\frac{\eta}{2}\frac{1}{(2\sigma_g^\eta)^{\frac{d-2k}{\eta}}\pi^{\frac{d}{2}}}\frac{\Gamma(\frac{d}{2})}{\Gamma(\frac{d-2k}{\eta})}r^{-2k}\exp(-\frac{1}{2}(\frac{r}{\sigma_g})^\eta)\frac{d\pi^{\frac{d}{2}}}{\Gamma(\frac{d}{2} + 1)}r^{d-1}dr \\
=&\frac{\eta}{2}\frac{1}{(2\sigma_g^\eta)^{\frac{d-2k}{\eta}}\pi^{\frac{d}{2}}}\frac{\Gamma(\frac{d}{2})}{\Gamma(\frac{d-2k}{\eta})}\frac{d\pi^{\frac{d}{2}}}{\Gamma(\frac{d}{2} + 1)}\lim_{t\to 0}\int_0^t\exp(-\frac{1}{2}(\frac{r}{\sigma_g})^\eta)r^{d-2k-1}dr \\
=& 0.
\end{aligned}
\end{equation}
Therefore, we see
\begin{equation}
\begin{aligned}
\mathbb{P}_{z\sim \mathcal{G}_t(\sigma, \eta, k, T)}\{f(x_0+z)=y_0\} =  \mathbb{P}_{z\sim \mathcal{G}(\sigma, \eta, k)}\{f(x_0+z)=y_0\mid 0 \leq \Vert z\Vert_2\leq T\} = 1,
\end{aligned}
\end{equation}
which concludes the proof.
 \end{proof}
Here we give a solution for a lower bound of Problem (\ref{primaldsrs}). The lower bound found with the help of Lemma \ref{lemmaforq} can be solved by the level set method, sharing the same thinking with proofs of Theorem \ref{thesg} and \ref{thegg}. We refer the readers to previous papers~\citep{yang2020, li2022} for more details. 
\begin{lemma}\label{lemmae2}
Under the setting of Lemma \ref{lemmaforq}, we let 
\begin{equation}
\begin{aligned}
R = \max \quad & \rho,\\
\rm{s.t.} \quad & \ \mathbb{E}_{u\sim \Gamma(\frac{d-2k}{\eta}, 1)}\Psi_{\frac{d-1}{2}}\left(\frac{T^2-(\sigma_g(2u)^{\frac{1}{\eta}}-\rho)^2}{4\rho \sigma_g(2u)^\frac{1}{\eta}}\right) \geq \frac{1}{2}.
\end{aligned}
\end{equation}
	If $R_D$ is the tightest $\ell_2$ certified radius when $\mathcal{P}=\mathcal{G}(\sigma, \eta, k)$ is the smoothing distribution by DSRS, then $R_D \geq R$.
\end{lemma}
	
\begin{proof} According to Lemma \ref{lemmaforq}, we define $\mathcal{Q} = \mathcal{G}_t(\sigma, \eta, k, T)$ as the supplementary distribution for $\mathcal{P}$. Then the Problem (\ref{primaldsrs}) can be simplified as 
\begin{equation}
\begin{aligned}
	&\min_{\tilde{f}_{x_0}\in \mathcal{F}} \quad  \mathbb{E}_{z\sim \mathcal{P}}\left(\tilde{f}_{x_0}(\delta + z)\right),\\
&\textrm{s.t.} \quad  \mathbb{E}_{z\sim \mathcal{P}}\left(\tilde{f}_{x_0}(z)\right) = A,\ \mathbb{E}_{z\sim \mathcal{Q}}\left(\tilde{f}_{x_0}(z)\right) = B \\
\stackrel{(a)}{\geq} &\min_{\tilde{f}_{x_0}\in \mathcal{F}} \quad  \mathbb{E}_{z\sim \mathcal{P}}\left(\tilde{f}_{x_0}(\delta + z)\right),\\
&\textrm{s.t.} \quad  \mathbb{E}_{z\sim \mathcal{Q}}\left(\tilde{f}_{x_0}(z)\right) = 1, \\
\end{aligned}
\end{equation}
	where (a) is because the subjection $\mathbb{E}_{z\sim\mathcal{P}}[\tilde{f}_{x_0}(z)] = A$ offers extra information outside $\{x\mid \Vert x - x_0 \Vert_2 \leq T\}$ in $\mathbb{R}^n$, where $p(x)\geq 0$. It is obvious that the equality holds if $A=0$. For the simplified problem, we have the worst function $\tilde{f}^{*}_{x_0}(t)$: 
	\begin{equation}
	\tilde{f}^{*}_{x_0}(t) = \left\{
	\begin{aligned}
	1,\ &\Vert t\Vert _2\leq T,\\
	0,\ &\Vert t\Vert _2> T.
	\end{aligned}
	\right.
	\end{equation}
	With the worst classifier, a lower bound for Problem (\ref{primaldsrs}) can finally be written as 
 \begin{equation}\label{94}
\begin{aligned}
\mathbb{E}_{z\sim \mathcal{P}}\left(\tilde{f}^{*}_{x_0}(\delta + z)\right)
\end{aligned}
\end{equation}

	for a fixed $\delta$. Similar to Theorem \ref{thesg} and Theorem \ref{thegg}, the final simplified problem is solved by the level set method~\citep{yang2020}. We have 
	\begin{equation}\label{95}
	\begin{aligned}
	&\mathbb{E}_{z\sim \mathcal{P}}\left(\tilde{f}^{*}_{x_0}(\delta + z)\right)\\
	=&\mathbb{E}_{\epsilon\sim \mathcal{P}}\{\Vert\delta+x\Vert_2 \leq T\} \\
	=& \int_0^\infty \frac{\eta}{2}\frac{1}{(2\sigma_g^\eta)^{\frac{d-2k}{\eta}}\pi^{\frac{d}{2}}}\frac{\Gamma(\frac{d}{2})}{\Gamma(\frac{d-2k}{\eta})}r^{d-2k-1}\exp[-\frac{1}{2}(\frac{r}{\sigma_g})^\eta]\frac{d\pi^{\frac{d}{2}}}{\Gamma(\frac{d}{2} + 1)}dr\mathbb{P}\{\Vert x+\delta\Vert _2\leq T\mid \Vert x\Vert _2=r\} \\
	=&  \frac{1}{\Gamma(\frac{d-2k}{\eta})}\int_0^{\infty} u^{\frac{d-2k}{\eta} - 1}\exp(-u)du\mathbb{P}\{\Vert x+\delta\Vert _2\leq T\mid \Vert x\Vert _2=\sigma_g(2u)^\frac{1}{\eta}\}.
	\end{aligned}
	\end{equation}
	
	As $\Vert x+\delta\Vert _2 \leq T\Longleftrightarrow  x_1 \leq \frac{T^2 - \rho^2 - \sigma_g^2(2u)^\frac{2}{\eta}}{2\rho}$, we have 
	\begin{equation}
	\frac{1 + \frac{x_1}{\sigma_g(2u)^\frac{1}{\eta}}}{2} \sim {\rm Beta}(\frac{d-1}{2}, \frac{d-1}{2})
	\end{equation}
	
	by Lemma I.23 of \cite{yang2020}. Thus, we get 
	\begin{equation}\label{97}
	\mathbb{P}\{\Vert x+\delta\Vert _2\leq T\mid \Vert x\Vert _2=\sigma_g(2u)^\frac{1}{\eta}\}  = \Psi_{\frac{d-1}{2}}\left(\frac{T^2-(\sigma_g(2u)^\frac{1}{\eta} - \rho)^2}{4\rho \sigma_g(2u)^\frac{1}{\eta}}\right).
	\end{equation}
	
	Combining Equation (\ref{95}) and Equation (\ref{97}), we finally get 
	\begin{equation}
	\mathbb{E}_{z\sim \mathcal{P}}\left(\tilde{f}^{*}_{x_0}(\delta + z)\right) = \mathbb{E}_{u\sim \Gamma(\frac{d-2k}{\eta}, 1)}\Psi_{\frac{d-1}{2}}\left(\frac{T^2-(\sigma_g(2u)^\frac{1}{\eta} - \rho)^2}{4\rho \sigma_g(2u)^\frac{1}{\eta}}\right).
	\end{equation}
	If we can ensure that $\mathbb{E}_{z\sim \mathcal{P}}\left(\tilde{f}^{*}_{x_0}(\delta + z)\right) \geq 0.5$ for $\delta = (\rho, 0, 0, \cdots, 0)^T$, $\Vert \delta\Vert_2$ will be qualified as a $\ell_2$ certified radius due to $\ell_2$-symmetry in the $R^n$ space~\citep{zhang2020}. As $R$ is the solution to the lower-bound problem of Problem (\ref{primaldsrs}), we have $R_D \geq R$, which concludes the proof.  
\end{proof}

The next two lemmas describe the concentration of the gamma distribution and the beta distribution, both are required for the lower estimation of Problem (\ref{primaldsrs}). 
\begin{lemma}\label{lemmae3}
(Concentration of the beta distribution) Let $\tau \in (\frac{1}{2}, 1)$, $\theta \in (0, 1)$. Then there exist $d_0 \in \mathbb{N}_+$, for any $d \geq d_0$,
\begin{equation}
	\Psi_{\frac{d-1}{2}}(\tau) \geq \theta.
\end{equation}
\end{lemma}
	
	\begin{proof} Let $X\sim {\rm Beta}(\frac{d-1}{2},\frac{d-1}{2})$. By property of the beta distribution, we have $\mathbb{E}X = \frac{1}{2},\ \mathbb{D}X = \frac{1}{4d},$ and
	
	\begin{equation}
 \mathbb{P}\{X > \tau\} = \mathbb{P}\{X-\frac{1}{2} > \tau - \frac{1}{2}\} \leq \mathbb{P}\{\vert X - \frac{1}{2}\vert \geq \tau - \frac{1}{2}\} \leq \frac{1}{4d(\tau-\frac{1}{2})^2}.
	\end{equation}
	We then have
 \begin{equation}
 \Psi_{\frac{d-1}{2}}(\tau) = 1 - \mathbb{P}\{X>\tau\} \geq 1 - \frac{1}{4d(\tau-\frac{1}{2})^2}.
 \end{equation}
	Let $1 - \frac{1}{4d(\tau - \frac{1}{2})^2} \geq \theta$, then $d \geq \frac{1}{4(1-\theta)(\tau - \frac{1}{2})^2}$. Picking $d_0 = \lceil \frac{1}{4(1-\theta)(\tau - \frac{1}{2})^2} \rceil$ concludes the proof.
	\end{proof}
	
	\begin{lemma} \label{lemmae4}
	(Unilateral concentration of the gamma distribution) Let $p \in (0, 1),\ d \in \mathbb{N}_+,\ \eta \in \mathbb{R}_+$ and $\beta \in (0, 1)$. Then there exist $d_0 \in \mathbb{N}_+$, for any $d \geq d_0$, 
	\begin{equation}
	\Lambda_{\frac{d}{\eta}}\left(\frac{\beta d}{\eta}\right) \leq p.
	\end{equation}
	\end{lemma}
	\begin{proof} Let $X \sim \Gamma(\frac{d}{\eta}, 1)$, then $\mathbb{E}X=\frac{d}{\eta}$, $\mathbb{D}X=\frac{d}{\eta}$, 
	\begin{equation}
	\Lambda_{\frac{d}{\eta}}\left(\frac{\beta d}{\eta}\right) = \mathbb{P}\left\{X\leq\frac{\beta d}{\eta}\right\}
	\leq \mathbb{P}\left\{|X - \frac{d}{\eta}|\geq \frac{(1 - \beta)d}{\eta}\right\} 
	\leq \frac{\eta}{(1 - \beta) ^ 2 d } .
	\end{equation}
	
	If $\frac{\eta}{(1 - \beta)^2d} \leq p$, then $d \geq \frac{\eta}{(1 - \beta)^2p}$. Let $d_0 = \lceil\frac{\eta}{(1 - \beta)^2p}\rceil$, for any $d \geq d_0$, we have $\frac{\eta}{(1 - \beta)^2d} \leq \frac{\eta}{(1 - \beta)^2d_0} \leq p$, which concludes the proof.
	\end{proof}
	\textbf{Remark.} Both Lemma \ref{lemmae3} and Lemma \ref{lemmae4} illustrate that random variables following beta and gamma distributions are highly concentrated towards their respective expectations in the high-dimensional setting. 
	
	The next lemma is required in the proof of Theorem \ref{thcorres}.
	\begin{lemma} \label{lemmae5}
	Let $x \in \mathbb{N}_+,\ \eta = \frac{1}{n},n \in \mathbb{N}_+,\ g(x) = \frac{x}{\left(\prod_{i=1}^{\frac{2}{\eta}}(\frac{x+2}{\eta} - i)\right)^{\frac{\eta}{2}}}$. Then $g(x)$ is a non-decreasing function with respect to $x$.
	\end{lemma}
	\begin{proof} Obviously $g(x) > 0$. Let $h(x) = \ln g(x)$, then
	
	\begin{equation}
	\frac{{\rm d}h(x)}{{\rm d}x} \geq 0 \Longleftrightarrow \frac{1}{g(x)}\frac{{\rm d}g(x)}{{\rm d}x} \geq 0 \Longleftrightarrow \frac{{\rm d}g(x)}{{\rm d}x}\geq0,
	\end{equation}
Therefore, 
	\begin{equation}
	\frac{{\rm d}h(x)}{{\rm d}x} 
	= \frac{1}{x} - \frac{\eta}{2}\sum_{i = 1}^{\frac{2}{\eta}}\frac{1}{x+2-i\eta}
	\geq \frac{1}{x} - \frac{\eta}{2}\sum_{i = 1}^{\frac{2}{\eta}}\frac{1}{x} = 0.
	\end{equation}
	
	Thus, for all $x \in \mathbb{N}_+$, we have $\frac{{\rm d}g(x)}{{\rm d}x} \geq 0$, which concludes the proof. 
	\end{proof}
	\subsection{Proof of Theorem \ref{fixbase}}\label{pfth1}
	\begin{proof}
	We let the smoothing distribution $\mathcal{P}=\mathcal{G}(\sigma, \eta, k)$ and the supplementary distribution $\mathcal{Q}=\mathcal{G}_t(\sigma, \eta, k, T)$. In addition, 
 we suppose the base classifier satisfies $(\sigma, p, 2)$-concentration property, which implies the base classifier is highly robust for restricted Gaussian noises. For the convenience of future discussions, we parameterize 0.02 as $\mu$, and the worst classifier $\tilde{f}_{x_0}^*$ is defined the same as that in Lemma \ref{lemmaforq}. 
	
	We see the condition $\eta=2$ simplifies some lemmas above. Let $\eta=2$ in Lemma \ref{lemmae1}, we get $\mathbb{P}_{z\sim\mathcal{S}(\sigma, 2)}\{\Vert z\Vert _2\leq T\}=\Lambda_{\frac{d}{2}}(\frac{T^2}{2\sigma_s^2})$, which means
	\begin{equation}\label{123}
	T = \sigma\sqrt{2\Lambda_{\frac{d}{2}}^{-1}(p)}.
	\end{equation}
	By definition of $(\sigma, p, 2)$-concentration and Lemma \ref{lemmaforq}, we have $\mathbb{P}_{z\sim \mathcal{G}_t(\sigma, \eta, k, T)}\{f(x_0+z)=y_0\} =1$, thus we can find a lower bound to estimate Problem (\ref{primaldsrs}) by Lemma \ref{lemmae2}. Let $\eta=2$ in Lemma \ref{lemmae3}, we see $\Lambda_{\frac{d}{2}} (\frac{\beta d}{2}) \leq p$.
	We have thereby 
	\begin{equation}
	T \geq \sigma\sqrt{\beta d}\label{lbt}.
	\end{equation}
	Equation (\ref{lbt}) will be used to find the lower bound of  Problem (\ref{94}). We then consider the solution of the lower-bound Problem (\ref{94}) in Lemma \ref{lemmae2}. We have 
	\begin{equation}\label{106}
	\frac{T^2-(\sigma_g(2u)^\frac{1}{\eta} - \rho)^2}{4\rho \sigma_g(2u)^\frac{1}{\eta}} \geq \tau \Longleftrightarrow \sigma_g(2u)^{\frac{2}{\eta}}-(2-4\tau)\rho\sigma_g(2u)^{\frac{1}{\eta}} + \rho^2 - T^2 \leq 0.
	\end{equation} 
	
	Notice Equation (\ref{106}) is a one-variable quadratic inequality with respect to $\sigma_g(2u)^{\frac{1}{\eta}}$. Let $\rho = \mu\sigma \sqrt{d}$ where the constant $\mu \in \mathbb{R}_+$. 
 When the discriminant $\Delta$ for Equation (\ref{106}) is positive, the solution for it is 
	\begin{equation}\label{111}
	\sigma_g(2u)^{\frac{1}{\eta}} \in \left[0, (1-2\tau)\rho + \sqrt{T^2 + (4\tau^2 - 4\tau)\rho^2}\right) \Longleftrightarrow 0\leq u < \frac{1}{2}\left(\frac{(1 -2\tau)\rho + \sqrt{T^2 + (4\tau^2 - 4\tau)\rho^2} }{\sigma_g}\right)^{\eta}.
	\end{equation} 
	Now we are ready to show the minimization for $\mathbb{E}_{z\sim \mathcal{P}}\left(\tilde{f}^{*}_{x_0}(\delta + z)\right)$ for base classifier satisfies $(\sigma, p, 2)$-concentration property. We have
\begin{align}
	&\mathbb{E}_{z\sim \mathcal{P}}\left(\tilde{f}^{*}_{x_0}(\delta + z)\right) \nonumber\\
	=& \mathbb{E}_{u\sim \Gamma(\frac{d-2k}{\eta}, 1)}\Psi_{\frac{d-1}{2}}\left(\frac{T^2-(\sigma_g(2u)^\frac{1}{\eta} - \rho)^2}{4\rho \sigma_g(2u)^\frac{1}{\eta}}\right)\nonumber\\
	\geq &\theta \mathbb{E}_{u\sim \Gamma(\frac{d-2k}{\eta}, 1)}\mathbb{I}\left(\frac{T^2-(\sigma_g(2u)^\frac{1}{\eta} - \rho)^2}{4\rho \sigma_g(2u)^\frac{1}{\eta}} \geq \tau\right) \nonumber\\
	\geq & \theta \mathbb{E}_{u\sim \Gamma(\frac{d-2k}{\eta}, 1)}\mathbb{I} \left(u < \frac{1}{2}\left(\frac{(1 -2\tau)\rho + \sqrt{T^2 + (4\tau^2 - 4\tau)\rho^2} }{\sigma_g}\right)^{\eta}\right)\nonumber\\
	 \stackrel{(a)}{\geq} &\theta \mathbb{E}_{u\sim \Gamma(\frac{d-2k}{\eta}, 1)}\mathbb{I} \left(u < \frac{1}{2}\left(\frac{(1 -2\tau)\mu\sigma\sqrt{d} + \sqrt{\sigma^2\beta d + (4\tau^2 - 4\tau)\mu^2\sigma^2d} }{\sigma_g}\right)^{\eta}\right)\nonumber\\
	  \stackrel{(b)}{=}& \theta \mathbb{E}_{u\sim \Gamma(\frac{d-2k}{\eta}, 1)}\mathbb{I} \left(u < \left(\sqrt{\frac{\Gamma(\frac{d-2k+2}{\eta})}{\Gamma(\frac{d-2k}{\eta})}}\left((1 -2\tau)\mu + \sqrt{\beta + (4\tau^2 - 4\tau)\mu^2}\right) \right)^{\eta}\right),
	\end{align}
where $\mathbb{I}(\cdot)$ is the indicator function, (a) is by $\rho = \mu\sigma\sqrt{d}$ and Equation (\ref{lbt}); (b) is because
\begin{equation}\label{sigg}
	\sigma_g=2 ^{-\frac{1}{\eta}} \sqrt{\frac{d\Gamma(\frac{d - 2k}{\eta})}{\Gamma(\frac{d-2k+2}{\eta})}}\sigma
	\end{equation}
	by definition. We write
	 \begin{equation}\label{eqform}
	m = \left(\sqrt{\frac{\Gamma(\frac{d-2k+2}{\eta})}{\Gamma(\frac{d-2k}{\eta})}}\left((1 -2\tau)\mu + \sqrt{\beta + (4\tau^2 - 4\tau)\mu^2}\right) \right)^{\eta},
	\end{equation} 
	to get 
	\begin{equation}\label{115}
	\mathbb{E}_{z\sim \mathcal{P}}\left(\tilde{f}^{*}_{x_0}(\delta + z)\right)  \geq   \theta\mathbb{E}_{u\sim \Gamma(\frac{d-2k}{\eta}, 1)} \mathds{1}_{u<m} 
	=\theta\Lambda_{\frac{d -2k}{\eta}}(m).
	\end{equation} 

	Let $\theta = 0.999, \beta=0.99, \tau=0.6,\ \mu = 0.02$~\citep{li2022}. We show the value of $\Lambda_{\frac{d -2k}{\eta}}(m)$ when $d - 2k \in [1, 30] \cap \mathbb{Z}$ and $\eta = \frac{1}{n}, n \in [1, 50] \cap \mathbb{Z}$ in Table \ref{t3}. Observing that there's no value greater than $\frac{1}{2\theta}$, we have $\mathbb{E}_{z\sim \mathcal{P}}\left(\tilde{f}^{*}_{x_0}(\delta + z)\right)\geq \frac{1}{2}$, which concludes the proof. 
	\end{proof}

\begin{algorithm}
    \KwIn{input dimension $d$, hyperparameters $k, \beta, \tau, \theta$, exponent $\eta$, error limitation $e$}
    \KwOut{tight constant $\mu_l$ for a specified EGG}
    $\mu_l, \mu_r \leftarrow 0, 1$
    
\While{$\mu_r - \mu_l > e$}{
    $\mu_m \leftarrow (\mu_r + \mu_l) / 2$ 
         
    $m_m \leftarrow \theta\Lambda_{\frac{d-2k}{\eta}}\left(\left(\sqrt{\frac{\Gamma(\frac{d-2k+2}{\eta})}{\Gamma(\frac{d-2k}{\eta})}}\left((1 -2\tau)\mu_m + \sqrt{\beta + (4\tau^2 - 4\tau)\mu_m^2}\right) \right)^{\eta}\right)$
    
  \eIf{$m_m > 1/2$}
    {$\mu_l \leftarrow \mu_m$}
    {$\mu_r \leftarrow \mu_m$}
    }
$\mu\leftarrow \mu_l$

\Return $\mu$
    \caption{Algorithm for finding tight $\mu$ for the $\Omega({\sqrt{d}})$ lower bound}\label{alg:tightmu}
\end{algorithm}

	\textbf{Remark.} We exhaust the cases for $\eta$ since the analytic solution to Problem (\ref{94}) includes intractable gamma function terms. The proof above can easily be generalized to other $\eta \in \mathbb{R}_+$. \eg, we have tried the sequence of $\eta \in [0.02, 1]$ increasing by 0.001, no value greater than $\frac{1}{2\theta}$ is observed, meaning these $\eta$s are all qualified to provide $\Omega(\sqrt{d})$ lower bounds for the $\ell_2$ certified radius under the setting of Theorem \ref{fixbase}. We have also shown results for $\eta\in[2, 10] \cap \mathbb{N}$ in Table \ref{t3}, where the boundary value for $\frac{1}{2\theta}$ is marked red. It is remarkable that $\Lambda_{\frac{d -2k}{\eta}}(m)$ decreases significantly when $\eta > 2$, which is in line with Figure \ref{main_sim} (left) though we are only considering the lower bound for the certified radius. In addition to the fixed constant 0.02, the tight constant factor $\mu$ for each EGG distribution can be computed by Algorithm \ref{alg:tightmu}. 
		
	\subsection{Theorem \ref{thcorres}}\label{apthcores}

The following theorem introduces $d^{1/\eta}$ into the lower bound using ($\sigma, p, \eta$)-concentration assumption, which can be proved to be almost equivalent to Theorem \ref{fixbase}.
\begin{theorem}
\label{thcorres}
Let $d \in \mathbb{N}_+$ be a sufficiently large input dimension, $(x_0, y_0) \in \mathbb{R}^d \times \mathcal{Y}$ be a labeled example and $f: \mathbb{R}^d\to \mathcal{Y}$ be a base classifier which satisfies ($\sigma, p, \eta$)-concentration property \wrt $(x_0, y_0)$. For the DSRS method, let $\mathcal{P} = \mathcal{G}(\sigma, \eta, k)$ be the smoothing distribution to give a smoothed classifier $\bar{f}_{\mathcal{P}}$, and $\mathcal{Q}=\mathcal{G}_t(\sigma, \eta, k, T)$ be the supplementary distribution with $T=\sigma_s\sqrt{2\Lambda_{\frac{d}{\eta}}^{-1}(p)}, d - 2k \in [1, 30] \cap \mathbb{N}$ and $\eta \in \{1, \frac{1}{2}, \frac{1}{3}, \cdots, \frac{1}{50}\}$. Then for the smoothed classifier $\bar{f}_{\mathcal{P}}(x)$ the certified $\ell_2$ radius 
\begin{equation}
r_\eta \geq 0.02\sigma_sd^{\frac{1}{\eta}}, 
\end{equation}
where $\sigma_s$ is the formal variance of $\mathcal{S}(\sigma, \eta)$. When $\sigma_s$ is converted to $\sigma$ keeping $\mathbb{E}r^2$ a constant, we still have
\begin{equation}
r_{\eta} = \Omega(\sqrt{d}).
\end{equation}
\end{theorem}
\begin{proof}
In this section, we first prove that a base classifier that satisfies a certain concentration property can certify $\Theta(d^{\frac{1}{\eta}})$ $\ell_2$ radii given the smoothing distribution $\mathcal{P}=\mathcal{G}(\sigma, \eta, k)$ and the supplementary distribution $\mathcal{Q}=\mathcal{G}_t(\sigma, \eta, k, T)$. Then, by converting $\sigma_s$ to $\sigma$,  we derive a $\Theta(\sqrt{d})$ lower bound for the certified radius. Like Appendix \ref{pfth1}, we find the lower bound for Problem (\ref{primaldsrs}) by Lemma \ref{lemmae2}. 
	
	In Lemma \ref{lemmae1}, when $\eta$ is an arbitrarily positive real number, we have 
	\begin{equation}
	p = \Lambda_{\frac{d}{\eta}}\left(\frac{T^\eta}{2\sigma_s^\eta}\right) \Longleftrightarrow T = \sigma_s(2\Lambda^{-1}_{\frac{d}{\eta}}(p))^{\frac{1}{\eta}}.
	\end{equation}
	 We thereby obtain
	\begin{equation}
	T \geq (\frac{2\beta}{\eta})^{\frac{1}{\eta}}\sigma_sd^{\frac{1}{\eta}}
	\end{equation} 
	by Lemma \ref{lemmae4}. Then we have Equation (\ref{106}) and Equation (\ref{111}) the same as in Appendix \ref{pfth1}. Let $\rho = \zeta\sigma_s d^{\frac{1}{\eta}}$ where $\zeta \in \mathbb{R}_+$. Let $\eta = \frac{1}{n}$, $n \in \mathbb{N}_+, \forall d \geq \tilde{d}$, where $\tilde{d}$ is a sufficiently large real integer which satisfies Lemma \ref{lemmae3} and Lemma \ref{lemmae4}. We then have the estimation
	 \begin{small}
 	\begin{align}\label{112}
	&\mathbb{E}_{z\sim \mathcal{P}}\left(\tilde{f}^{*}_{x_0}(\delta + z)\right) \nonumber \\ 
	\stackrel{(a)}{\geq} & \theta \mathbb{E}_{u\sim \Gamma(\frac{d-2k}{\eta}, 1)}\mathbb{I} \left(u < \frac{1}{2}\left(\frac{(1 -2\tau)\rho + \sqrt{T^2 + (4\tau^2 - 4\tau)\rho^2} }{\sigma_g}\right)^{\eta}\right) \nonumber \\
	\stackrel{(b)}{=} & \theta  \mathbb{E}_{u\sim \Gamma(\frac{d-2k}{\eta}, 1)} \mathbb{I}\left(u < \frac{1}{2}\left(\sqrt{\frac{\Gamma(\frac{d}{\eta})\Gamma(\frac{d-2k+2}{\eta})}{\Gamma(\frac{d+2}{\eta})\Gamma(\frac{d-2k}{\eta})}}\frac{(1 -2\tau)\rho + \sqrt{T^2 + (4\tau^2 - 4\tau)\rho^2} }{\sigma_s}\right)^{\eta}\right) \nonumber \\
	\stackrel{(c)}{\geq} & \theta  \mathbb{E}_{u\sim \Gamma(\frac{d-2k}{\eta}, 1)} \mathbb{I}\left(u < \frac{1}{2}\left( \sqrt{\frac{\Gamma(\frac{d}{\eta})\Gamma(\frac{d-2k+2}{\eta})}{\Gamma(\frac{d+2}{\eta})\Gamma(\frac{d-2k}{\eta})}}\frac{(1 -2\tau)\zeta\sigma_s d^{\frac{1}{\eta}} + \sqrt{(\frac{2\beta}{\eta})^{\frac{2}{\eta}}\sigma_s^2d^{\frac{2}{\eta}} + (4\tau^2 - 4\tau)(\zeta\sigma_s d^{\frac{1}{\eta}})^2} }{\sigma_s}\right)^{\eta}\right) \nonumber \\
	\stackrel{(d)}{\geq} & \theta  \mathbb{E}_{u\sim \Gamma(\frac{d-2k}{\eta}, 1)} \mathbb{I}\left(u < \frac{\tilde{d}}{2\left(\prod_{i=1}^{\frac{2}{\eta}}(\frac{\tilde{d}+2}{\eta} - i)\right)^{\frac{\eta}{2}}}\left(\sqrt{\frac{\Gamma(\frac{d-2k+2}{\eta})}{\Gamma(\frac{d-2k}{\eta})}}\left((1-2\tau)\zeta + \sqrt{(\frac{2\beta}{\eta})^{\frac{2}{\eta}}+(4\tau^2 - 4\tau)\zeta^2}\right)\right)^\eta\right). 
	\end{align}
 \end{small}
	In the equations above: (a) solve the inequality with respect to $u$ in the indicator function, whose solution is shown in Equation (\ref{111});   (b) for a constant $\mathbb{E}r^2$, we have
	\begin{equation}
	\sigma_g = \sqrt{\frac{\Gamma(\frac{d-2k}{\eta})\Gamma(\frac{d+2}{\eta})}{\Gamma(\frac{d}{\eta})\Gamma(\frac{d-2k+2}{\eta})}}\sigma_s;
	\end{equation} 
	(c) by Lemma \ref{lemmae4} and $\rho = \zeta\sigma_s d^{\frac{1}{\eta}}$; (d) by Lemma \ref{lemmae5}.
	
	We write
	 \begin{equation}
	m = \frac{\tilde{d}}{2\left(\prod_{i=1}^{\frac{2}{\eta}}(\frac{\tilde{d}+2}{\eta} - i)\right)^{\frac{\eta}{2}}}\left(\sqrt{\frac{\Gamma(\frac{d-2k+2}{\eta})}{\Gamma(\frac{d-2k}{\eta})}}\left((1-2\tau)\zeta + \sqrt{(\frac{2\beta}{\eta})^{\frac{2}{\eta}}+(4\tau^2 - 4\tau)\zeta^2}\right)\right)^\eta, 
	\end{equation} 
	and then we have 
	\begin{equation}
	\mathbb{E}_{z\sim \mathcal{P}}\left(\tilde{f}^{*}_{x_0}(\delta + z)\right) \geq  \Lambda_{\frac{d -2k}{\eta}}(m)
	\end{equation} 
	by Equation (\ref{112}). Notice the lower estimation is slightly different from Appendix \ref{pfth1} since there is a $\tilde{d}$ in the expression. Pick $p = 0.5,\ \tilde{d} = 25000, \theta = 0.999,  \beta=0.99, \tau=0.6,\ \zeta = 0.02$~\citep{li2022}. We show the value of $\Lambda_{\frac{d -2k}{\eta}}(m)$ when $d - 2k \in [1, 30] \cap \mathbb{Z}$ and $\eta = \frac{1}{n}, n \in [1, 50] \cap \mathbb{Z}$ in Table \ref{t2}. As a result, all values in Table \ref{t2} are greater than $\frac{1}{2\theta} \approx 0.5005$, which means 
	\begin{equation}
	\begin{aligned}
	\mathbb{E}_{z\sim \mathcal{P}}\left(\tilde{f}^{*}_{x_0}(\delta + z)\right)  \geq  
 \theta \cdot \frac{1}{2\theta}  = \frac{1}{2}.
	\end{aligned}
	\end{equation} 
	Recalling that our goal here is to check whether $\mathbb{E}_{z\sim \mathcal{P}}\left(\tilde{f}^{*}_{x_0}(\delta + z)\right) \geq \frac{1}{2}$ holds for some determined $\rho$, we have finished the proof that $R_{\eta} \geq \rho = \zeta\sigma_sd^{\frac{1}{\eta}}$. Superficially, we get an $\Omega(d^{\frac{1}{\eta}})$ bound for the $\ell_2$ certified radius, which seems tighter than the $\Omega(\sqrt{d})$ one. But by Lemma \ref{sigmasbound}, it is essentially equivalent to the $\Omega(\sqrt{d})$ lower bound; we thus remark that Theorem \ref{fixbase} and Theorem \ref{thcorres} are different manifestations of the same fact.
	\end{proof}
\textbf{Remark.} Different from Theorem \ref{fixbase}, this proof can not be generalized directly to $\eta>2$ due to the property of factorial. In essence, the $(\sigma, p, \eta)$-concentration assumption is slightly less strict for base classifiers than the $(\sigma, p, 2)$-concentration assumption.
		
	\begin{sidewaystable}
	\centering
	\caption{Value of $\Lambda_{\frac{d -2k}{\eta}}(m)$ for Theorem \ref{fixbase}.}\label{t3}
	\resizebox{20cm}{!}{\begin{tabular}{ccccccccccccccccccccccccccccccc}
	\toprule
	$\eta\backslash d-2k$ & 1 & 2 & 3 & 4 & 5 & 6 & 7 & 8 & 9 & 10 & 11 & 12 & 13 & 14 & 15 & 16 & 17 & 18 & 19 & 20 & 21 & 22 & 23 & 24 & 25 & 26 & 27 & 28 & 29 & 30 \\
	\midrule
	10 & 0.584 & 0.515 & \textcolor{red}{0.487} & 0.473 & 0.465 & 0.459 & 0.456 & 0.453 & 0.450 & 0.448 & 0.447 & 0.445 & 0.444 & 0.442 & 0.441 & 0.440 & 0.439 & 0.438 & 0.437 & 0.436 & 0.435 & 0.434 & 0.433 & 0.432 & 0.431 & 0.430 & 0.429 & 0.428 & 0.427 & 0.426 \\
	9 & 0.586 & 0.519 & \textcolor{red}{0.492} & 0.478 & 0.471 & 0.465 & 0.462 & 0.459 & 0.456 & 0.454 & 0.453 & 0.451 & 0.450 & 0.448 & 0.447 & 0.446 & 0.445 & 0.443 & 0.442 & 0.441 & 0.440 & 0.439 & 0.438 & 0.437 & 0.436 & 0.435 & 0.435 & 0.434 & 0.433 & 0.432 \\
	8 & 0.590 & 0.524 & \textcolor{red}{0.498} & 0.485 & 0.478 & 0.472 & 0.469 & 0.466 & 0.463 & 0.461 & 0.459 & 0.458 & 0.456 & 0.455 & 0.454 & 0.452 & 0.451 & 0.450 & 0.449 & 0.448 & 0.447 & 0.446 & 0.445 & 0.444 & 0.443 & 0.442 & 0.441 & 0.440 & 0.439 & 0.438 \\
	7 & 0.594 & 0.531 & 0.506 & \textcolor{red}{0.493} & 0.486 & 0.481 & 0.477 & 0.474 & 0.471 & 0.469 & 0.467 & 0.465 & 0.464 & 0.462 & 0.461 & 0.460 & 0.458 & 0.457 & 0.456 & 0.455 & 0.454 & 0.453 & 0.452 & 0.451 & 0.450 & 0.449 & 0.448 & 0.447 & 0.446 & 0.445 \\
	6 & 0.600 & 0.539 & 0.515 & 0.503 & \textcolor{red}{0.495} & 0.490 & 0.486 & 0.483 & 0.480 & 0.478 & 0.476 & 0.474 & 0.472 & 0.471 & 0.469 & 0.468 & 0.467 & 0.465 & 0.464 & 0.463 & 0.462 & 0.461 & 0.460 & 0.459 & 0.458 & 0.457 & 0.456 & 0.455 & 0.454 & 0.453 \\
	5 & 0.609 & 0.551 & 0.528 & 0.515 & 0.507 & 0.502 & \textcolor{red}{0.497} & 0.494 & 0.491 & 0.489 & 0.486 & 0.484 & 0.482 & 0.481 & 0.479 & 0.478 & 0.476 & 0.475 & 0.474 & 0.472 & 0.471 & 0.470 & 0.469 & 0.468 & 0.467 & 0.466 & 0.465 & 0.464 & 0.463 & 0.462 \\
	4 & 0.622 & 0.567 & 0.544 & 0.531 & 0.522 & 0.516 & 0.512 & 0.508 & 0.504 & 0.502 & \textcolor{red}{0.499} & 0.497 & 0.495 & 0.493 & 0.491 & 0.489 & 0.488 & 0.486 & 0.485 & 0.484 & 0.482 & 0.481 & 0.480 & 0.479 & 0.478 & 0.477 & 0.476 & 0.475 & 0.474 & 0.473 \\
	3 & 0.643 & 0.589 & 0.566 & 0.552 & 0.543 & 0.536 & 0.530 & 0.526 & 0.522 & 0.519 & 0.516 & 0.513 & 0.511 & 0.508 & 0.506 & 0.505 & 0.503 & 0.501 & \textcolor{red}{0.500} & 0.498 & 0.497 & 0.496 & 0.494 & 0.493 & 0.492 & 0.491 & 0.490 & 0.489 & 0.488 & 0.487 \\
	2 & 0.678 & 0.625 & 0.600 & 0.584 & 0.573 & 0.564 & 0.558 & 0.552 & 0.547 & 0.543 & 0.540 & 0.537 & 0.534 & 0.531 & 0.529 & 0.526 & 0.524 & 0.522 & 0.521 & 0.519 & 0.517 & 0.516 & 0.514 & 0.513 & 0.512 & 0.510 & 0.509 & 0.508 & 0.507 & 0.506 \\
	\hline
	1 & 0.754 & 0.697 & 0.666 & 0.646 & 0.631 & 0.619 & 0.610 & 0.602 & 0.596 & 0.590 & 0.585 & 0.581 & 0.577 & 0.573 & 0.570 & 0.567 & 0.564 & 0.561 & 0.559 & 0.557 & 0.555 & 0.553 & 0.551 & 0.549 & 0.547 & 0.546 & 0.544 & 0.543 & 0.541 & 0.540 \\
	1/2 & 0.841 & 0.782 & 0.745 & 0.720 & 0.701 & 0.685 & 0.673 & 0.662 & 0.654 & 0.646 & 0.639 & 0.633 & 0.628 & 0.623 & 0.618 & 0.614 & 0.610 & 0.607 & 0.604 & 0.601 & 0.598 & 0.595 & 0.593 & 0.590 & 0.588 & 0.586 & 0.584 & 0.582 & 0.580 & 0.578 \\
	1/3 & 0.891 & 0.835 & 0.797 & 0.769 & 0.747 & 0.730 & 0.716 & 0.704 & 0.693 & 0.684 & 0.676 & 0.669 & 0.663 & 0.657 & 0.652 & 0.647 & 0.642 & 0.638 & 0.634 & 0.631 & 0.628 & 0.624 & 0.621 & 0.619 & 0.616 & 0.613 & 0.611 & 0.609 & 0.607 & 0.605 \\
	1/4 & 0.924 & 0.872 & 0.834 & 0.805 & 0.783 & 0.764 & 0.749 & 0.736 & 0.724 & 0.714 & 0.706 & 0.698 & 0.691 & 0.684 & 0.678 & 0.673 & 0.668 & 0.663 & 0.659 & 0.655 & 0.651 & 0.648 & 0.644 & 0.641 & 0.638 & 0.635 & 0.633 & 0.630 & 0.628 & 0.625 \\
	1/5 & 0.946 & 0.899 & 0.863 & 0.834 & 0.811 & 0.792 & 0.776 & 0.762 & 0.750 & 0.739 & 0.730 & 0.721 & 0.714 & 0.707 & 0.700 & 0.695 & 0.689 & 0.684 & 0.679 & 0.675 & 0.671 & 0.667 & 0.664 & 0.660 & 0.657 & 0.654 & 0.651 & 0.648 & 0.645 & 0.643 \\
	1/6 & 0.961 & 0.920 & 0.885 & 0.858 & 0.835 & 0.815 & 0.799 & 0.784 & 0.772 & 0.761 & 0.751 & 0.742 & 0.734 & 0.727 & 0.720 & 0.714 & 0.708 & 0.702 & 0.697 & 0.693 & 0.688 & 0.684 & 0.680 & 0.677 & 0.673 & 0.670 & 0.667 & 0.664 & 0.661 & 0.658 \\
	1/7 & 0.972 & 0.936 & 0.904 & 0.877 & 0.854 & 0.835 & 0.818 & 0.804 & 0.791 & 0.779 & 0.769 & 0.760 & 0.752 & 0.744 & 0.737 & 0.730 & 0.724 & 0.719 & 0.713 & 0.708 & 0.704 & 0.699 & 0.695 & 0.691 & 0.688 & 0.684 & 0.681 & 0.678 & 0.675 & 0.672 \\
	1/8 & 0.979 & 0.948 & 0.919 & 0.893 & 0.871 & 0.852 & 0.835 & 0.821 & 0.808 & 0.796 & 0.785 & 0.776 & 0.767 & 0.759 & 0.752 & 0.745 & 0.739 & 0.733 & 0.728 & 0.723 & 0.718 & 0.713 & 0.709 & 0.705 & 0.701 & 0.697 & 0.694 & 0.691 & 0.687 & 0.684 \\
	1/9 & 0.985 & 0.958 & 0.931 & 0.907 & 0.886 & 0.867 & 0.850 & 0.836 & 0.822 & 0.811 & 0.800 & 0.790 & 0.782 & 0.773 & 0.766 & 0.759 & 0.752 & 0.746 & 0.741 & 0.735 & 0.730 & 0.726 & 0.721 & 0.717 & 0.713 & 0.709 & 0.706 & 0.702 & 0.699 & 0.696 \\
	1/10 & 0.989 & 0.966 & 0.942 & 0.919 & 0.898 & 0.880 & 0.863 & 0.849 & 0.836 & 0.824 & 0.813 & 0.803 & 0.794 & 0.786 & 0.779 & 0.771 & 0.765 & 0.759 & 0.753 & 0.747 & 0.742 & 0.737 & 0.733 & 0.728 & 0.724 & 0.720 & 0.717 & 0.713 & 0.710 & 0.706 \\
	1/11 & 0.992 & 0.972 & 0.950 & 0.929 & 0.909 & 0.891 & 0.875 & 0.861 & 0.848 & 0.836 & 0.825 & 0.815 & 0.806 & 0.798 & 0.790 & 0.783 & 0.776 & 0.770 & 0.764 & 0.758 & 0.753 & 0.748 & 0.743 & 0.739 & 0.735 & 0.731 & 0.727 & 0.723 & 0.720 & 0.716 \\
	1/12 & 0.994 & 0.977 & 0.957 & 0.937 & 0.919 & 0.902 & 0.886 & 0.872 & 0.859 & 0.847 & 0.836 & 0.826 & 0.817 & 0.809 & 0.801 & 0.794 & 0.787 & 0.780 & 0.774 & 0.769 & 0.763 & 0.758 & 0.753 & 0.749 & 0.744 & 0.740 & 0.736 & 0.733 & 0.729 & 0.725 \\
	1/13 & 0.995 & 0.982 & 0.964 & 0.945 & 0.927 & 0.911 & 0.895 & 0.882 & 0.869 & 0.857 & 0.846 & 0.836 & 0.827 & 0.819 & 0.811 & 0.803 & 0.797 & 0.790 & 0.784 & 0.778 & 0.773 & 0.768 & 0.763 & 0.758 & 0.754 & 0.749 & 0.745 & 0.741 & 0.738 & 0.734 \\
	1/14 & 0.997 & 0.985 & 0.969 & 0.952 & 0.935 & 0.919 & 0.904 & 0.890 & 0.878 & 0.866 & 0.856 & 0.846 & 0.837 & 0.828 & 0.820 & 0.813 & 0.806 & 0.799 & 0.793 & 0.787 & 0.782 & 0.776 & 0.771 & 0.767 & 0.762 & 0.758 & 0.754 & 0.750 & 0.746 & 0.742 \\
	1/15 & 0.997 & 0.988 & 0.973 & 0.957 & 0.941 & 0.926 & 0.912 & 0.899 & 0.886 & 0.875 & 0.864 & 0.854 & 0.845 & 0.837 & 0.829 & 0.821 & 0.814 & 0.808 & 0.801 & 0.796 & 0.790 & 0.785 & 0.780 & 0.775 & 0.770 & 0.766 & 0.762 & 0.758 & 0.754 & 0.750 \\
	1/16 & 0.998 & 0.990 & 0.977 & 0.962 & 0.947 & 0.933 & 0.919 & 0.906 & 0.894 & 0.883 & 0.872 & 0.862 & 0.853 & 0.845 & 0.837 & 0.829 & 0.822 & 0.816 & 0.809 & 0.803 & 0.798 & 0.792 & 0.787 & 0.783 & 0.778 & 0.773 & 0.769 & 0.765 & 0.761 & 0.758 \\
	1/17 & 0.999 & 0.992 & 0.980 & 0.967 & 0.953 & 0.939 & 0.925 & 0.913 & 0.901 & 0.890 & 0.880 & 0.870 & 0.861 & 0.852 & 0.844 & 0.837 & 0.830 & 0.823 & 0.817 & 0.811 & 0.805 & 0.800 & 0.795 & 0.790 & 0.785 & 0.781 & 0.776 & 0.772 & 0.768 & 0.765 \\
	1/18 & 0.999 & 0.993 & 0.983 & 0.970 & 0.957 & 0.944 & 0.931 & 0.919 & 0.907 & 0.897 & 0.886 & 0.877 & 0.868 & 0.860 & 0.852 & 0.844 & 0.837 & 0.830 & 0.824 & 0.818 & 0.812 & 0.807 & 0.802 & 0.797 & 0.792 & 0.788 & 0.783 & 0.779 & 0.775 & 0.771 \\
	1/19 & 0.999 & 0.994 & 0.985 & 0.974 & 0.961 & 0.949 & 0.937 & 0.925 & 0.914 & 0.903 & 0.893 & 0.883 & 0.875 & 0.866 & 0.858 & 0.851 & 0.844 & 0.837 & 0.831 & 0.825 & 0.819 & 0.814 & 0.808 & 0.803 & 0.799 & 0.794 & 0.790 & 0.786 & 0.782 & 0.778 \\
	1/20 & 0.999 & 0.995 & 0.987 & 0.977 & 0.965 & 0.953 & 0.941 & 0.930 & 0.919 & 0.909 & 0.899 & 0.890 & 0.881 & 0.873 & 0.865 & 0.857 & 0.850 & 0.844 & 0.837 & 0.831 & 0.826 & 0.820 & 0.815 & 0.810 & 0.805 & 0.800 & 0.796 & 0.792 & 0.788 & 0.784 \\
	1/21 & 1.000 & 0.996 & 0.989 & 0.979 & 0.969 & 0.957 & 0.946 & 0.935 & 0.924 & 0.914 & 0.904 & 0.895 & 0.887 & 0.878 & 0.871 & 0.863 & 0.856 & 0.850 & 0.843 & 0.837 & 0.832 & 0.826 & 0.821 & 0.816 & 0.811 & 0.806 & 0.802 & 0.798 & 0.794 & 0.790 \\
	1/22 & 1.000 & 0.997 & 0.990 & 0.982 & 0.972 & 0.961 & 0.950 & 0.939 & 0.929 & 0.919 & 0.910 & 0.901 & 0.892 & 0.884 & 0.876 & 0.869 & 0.862 & 0.855 & 0.849 & 0.843 & 0.837 & 0.832 & 0.827 & 0.822 & 0.817 & 0.812 & 0.808 & 0.804 & 0.799 & 0.795 \\
	1/23 & 1.000 & 0.997 & 0.992 & 0.984 & 0.974 & 0.964 & 0.954 & 0.944 & 0.934 & 0.924 & 0.915 & 0.906 & 0.897 & 0.889 & 0.882 & 0.874 & 0.868 & 0.861 & 0.855 & 0.849 & 0.843 & 0.837 & 0.832 & 0.827 & 0.822 & 0.818 & 0.813 & 0.809 & 0.805 & 0.801 \\
	1/24 & 1.000 & 0.998 & 0.993 & 0.986 & 0.977 & 0.967 & 0.957 & 0.947 & 0.938 & 0.928 & 0.919 & 0.911 & 0.902 & 0.894 & 0.887 & 0.880 & 0.873 & 0.866 & 0.860 & 0.854 & 0.848 & 0.843 & 0.838 & 0.833 & 0.828 & 0.823 & 0.819 & 0.814 & 0.810 & 0.806 \\
	1/25 & 1.000 & 0.998 & 0.994 & 0.987 & 0.979 & 0.970 & 0.960 & 0.951 & 0.942 & 0.932 & 0.924 & 0.915 & 0.907 & 0.899 & 0.892 & 0.885 & 0.878 & 0.871 & 0.865 & 0.859 & 0.853 & 0.848 & 0.843 & 0.838 & 0.833 & 0.828 & 0.824 & 0.819 & 0.815 & 0.811 \\
	1/26 & 1.000 & 0.998 & 0.995 & 0.989 & 0.981 & 0.972 & 0.963 & 0.954 & 0.945 & 0.936 & 0.928 & 0.919 & 0.911 & 0.904 & 0.896 & 0.889 & 0.882 & 0.876 & 0.870 & 0.864 & 0.858 & 0.853 & 0.848 & 0.842 & 0.838 & 0.833 & 0.829 & 0.824 & 0.820 & 0.816 \\
	1/27 & 1.000 & 0.999 & 0.995 & 0.990 & 0.983 & 0.975 & 0.966 & 0.957 & 0.948 & 0.940 & 0.931 & 0.923 & 0.915 & 0.908 & 0.901 & 0.894 & 0.887 & 0.881 & 0.874 & 0.869 & 0.863 & 0.857 & 0.852 & 0.847 & 0.842 & 0.838 & 0.833 & 0.829 & 0.825 & 0.821 \\
	1/28 & 1.000 & 0.999 & 0.996 & 0.991 & 0.984 & 0.977 & 0.969 & 0.960 & 0.952 & 0.943 & 0.935 & 0.927 & 0.919 & 0.912 & 0.905 & 0.898 & 0.891 & 0.885 & 0.879 & 0.873 & 0.867 & 0.862 & 0.857 & 0.852 & 0.847 & 0.842 & 0.838 & 0.833 & 0.829 & 0.825 \\
	1/29 & 1.000 & 0.999 & 0.996 & 0.992 & 0.986 & 0.979 & 0.971 & 0.963 & 0.955 & 0.946 & 0.938 & 0.931 & 0.923 & 0.916 & 0.909 & 0.902 & 0.895 & 0.889 & 0.883 & 0.877 & 0.872 & 0.866 & 0.861 & 0.856 & 0.851 & 0.847 & 0.842 & 0.838 & 0.834 & 0.830 \\
	1/30 & 1.000 & 0.999 & 0.997 & 0.993 & 0.987 & 0.980 & 0.973 & 0.965 & 0.957 & 0.949 & 0.942 & 0.934 & 0.927 & 0.919 & 0.912 & 0.906 & 0.899 & 0.893 & 0.887 & 0.881 & 0.876 & 0.870 & 0.865 & 0.860 & 0.855 & 0.851 & 0.846 & 0.842 & 0.838 & 0.834 \\
	1/31 & 1.000 & 0.999 & 0.997 & 0.994 & 0.988 & 0.982 & 0.975 & 0.967 & 0.960 & 0.952 & 0.945 & 0.937 & 0.930 & 0.923 & 0.916 & 0.909 & 0.903 & 0.897 & 0.891 & 0.885 & 0.880 & 0.874 & 0.869 & 0.864 & 0.860 & 0.855 & 0.850 & 0.846 & 0.842 & 0.838 \\
	1/32 & 1.000 & 1.000 & 0.998 & 0.994 & 0.989 & 0.983 & 0.977 & 0.970 & 0.962 & 0.955 & 0.947 & 0.940 & 0.933 & 0.926 & 0.919 & 0.913 & 0.907 & 0.901 & 0.895 & 0.889 & 0.884 & 0.878 & 0.873 & 0.868 & 0.863 & 0.859 & 0.854 & 0.850 & 0.846 & 0.842 \\
	1/33 & 1.000 & 1.000 & 0.998 & 0.995 & 0.990 & 0.985 & 0.978 & 0.972 & 0.964 & 0.957 & 0.950 & 0.943 & 0.936 & 0.929 & 0.923 & 0.916 & 0.910 & 0.904 & 0.898 & 0.893 & 0.887 & 0.882 & 0.877 & 0.872 & 0.867 & 0.863 & 0.858 & 0.854 & 0.850 & 0.846 \\
	1/34 & 1.000 & 1.000 & 0.998 & 0.995 & 0.991 & 0.986 & 0.980 & 0.973 & 0.967 & 0.960 & 0.953 & 0.946 & 0.939 & 0.932 & 0.926 & 0.920 & 0.913 & 0.907 & 0.902 & 0.896 & 0.891 & 0.885 & 0.880 & 0.876 & 0.871 & 0.866 & 0.862 & 0.858 & 0.853 & 0.849 \\
	1/35 & 1.000 & 1.000 & 0.998 & 0.996 & 0.992 & 0.987 & 0.981 & 0.975 & 0.969 & 0.962 & 0.955 & 0.948 & 0.942 & 0.935 & 0.929 & 0.923 & 0.917 & 0.911 & 0.905 & 0.899 & 0.894 & 0.889 & 0.884 & 0.879 & 0.874 & 0.870 & 0.865 & 0.861 & 0.857 & 0.853 \\
	1/36 & 1.000 & 1.000 & 0.999 & 0.996 & 0.993 & 0.988 & 0.983 & 0.977 & 0.970 & 0.964 & 0.957 & 0.951 & 0.944 & 0.938 & 0.932 & 0.926 & 0.920 & 0.914 & 0.908 & 0.903 & 0.897 & 0.892 & 0.887 & 0.882 & 0.878 & 0.873 & 0.869 & 0.865 & 0.860 & 0.856 \\
	1/37 & 1.000 & 1.000 & 0.999 & 0.997 & 0.993 & 0.989 & 0.984 & 0.978 & 0.972 & 0.966 & 0.959 & 0.953 & 0.947 & 0.940 & 0.934 & 0.928 & 0.922 & 0.917 & 0.911 & 0.906 & 0.901 & 0.895 & 0.890 & 0.886 & 0.881 & 0.877 & 0.872 & 0.868 & 0.864 & 0.860 \\
	1/38 & 1.000 & 1.000 & 0.999 & 0.997 & 0.994 & 0.990 & 0.985 & 0.980 & 0.974 & 0.968 & 0.962 & 0.955 & 0.949 & 0.943 & 0.937 & 0.931 & 0.925 & 0.920 & 0.914 & 0.909 & 0.904 & 0.899 & 0.894 & 0.889 & 0.884 & 0.880 & 0.875 & 0.871 & 0.867 & 0.863 \\
	1/39 & 1.000 & 1.000 & 0.999 & 0.997 & 0.995 & 0.991 & 0.986 & 0.981 & 0.975 & 0.969 & 0.963 & 0.957 & 0.951 & 0.945 & 0.939 & 0.934 & 0.928 & 0.922 & 0.917 & 0.912 & 0.906 & 0.901 & 0.897 & 0.892 & 0.887 & 0.883 & 0.879 & 0.874 & 0.870 & 0.866 \\
	1/40 & 1.000 & 1.000 & 0.999 & 0.998 & 0.995 & 0.991 & 0.987 & 0.982 & 0.977 & 0.971 & 0.965 & 0.959 & 0.953 & 0.948 & 0.942 & 0.936 & 0.930 & 0.925 & 0.920 & 0.914 & 0.909 & 0.904 & 0.900 & 0.895 & 0.890 & 0.886 & 0.882 & 0.877 & 0.873 & 0.869 \\
	1/41 & 1.000 & 1.000 & 0.999 & 0.998 & 0.995 & 0.992 & 0.988 & 0.983 & 0.978 & 0.973 & 0.967 & 0.961 & 0.955 & 0.950 & 0.944 & 0.938 & 0.933 & 0.928 & 0.922 & 0.917 & 0.912 & 0.907 & 0.902 & 0.898 & 0.893 & 0.889 & 0.885 & 0.880 & 0.876 & 0.872 \\
	1/42 & 1.000 & 1.000 & 0.999 & 0.998 & 0.996 & 0.993 & 0.989 & 0.984 & 0.979 & 0.974 & 0.969 & 0.963 & 0.957 & 0.952 & 0.946 & 0.941 & 0.935 & 0.930 & 0.925 & 0.920 & 0.915 & 0.910 & 0.905 & 0.900 & 0.896 & 0.892 & 0.887 & 0.883 & 0.879 & 0.875 \\
	1/43 & 1.000 & 1.000 & 0.999 & 0.998 & 0.996 & 0.993 & 0.990 & 0.985 & 0.981 & 0.975 & 0.970 & 0.965 & 0.959 & 0.954 & 0.948 & 0.943 & 0.938 & 0.932 & 0.927 & 0.922 & 0.917 & 0.912 & 0.908 & 0.903 & 0.899 & 0.894 & 0.890 & 0.886 & 0.882 & 0.878 \\
	1/44 & 1.000 & 1.000 & 1.000 & 0.998 & 0.997 & 0.994 & 0.990 & 0.986 & 0.982 & 0.977 & 0.972 & 0.966 & 0.961 & 0.956 & 0.950 & 0.945 & 0.940 & 0.935 & 0.930 & 0.925 & 0.920 & 0.915 & 0.910 & 0.906 & 0.901 & 0.897 & 0.893 & 0.889 & 0.885 & 0.881 \\
	1/45 & 1.000 & 1.000 & 1.000 & 0.999 & 0.997 & 0.994 & 0.991 & 0.987 & 0.983 & 0.978 & 0.973 & 0.968 & 0.963 & 0.957 & 0.952 & 0.947 & 0.942 & 0.937 & 0.932 & 0.927 & 0.922 & 0.917 & 0.913 & 0.908 & 0.904 & 0.900 & 0.895 & 0.891 & 0.887 & 0.884 \\
	1/46 & 1.000 & 1.000 & 1.000 & 0.999 & 0.997 & 0.995 & 0.992 & 0.988 & 0.984 & 0.979 & 0.974 & 0.969 & 0.964 & 0.959 & 0.954 & 0.949 & 0.944 & 0.939 & 0.934 & 0.929 & 0.924 & 0.920 & 0.915 & 0.911 & 0.906 & 0.902 & 0.898 & 0.894 & 0.890 & 0.886 \\
	1/47 & 1.000 & 1.000 & 1.000 & 0.999 & 0.997 & 0.995 & 0.992 & 0.989 & 0.985 & 0.980 & 0.976 & 0.971 & 0.966 & 0.961 & 0.956 & 0.951 & 0.946 & 0.941 & 0.936 & 0.931 & 0.927 & 0.922 & 0.917 & 0.913 & 0.909 & 0.905 & 0.900 & 0.896 & 0.893 & 0.889 \\
	1/48 & 1.000 & 1.000 & 1.000 & 0.999 & 0.998 & 0.996 & 0.993 & 0.989 & 0.985 & 0.981 & 0.977 & 0.972 & 0.967 & 0.962 & 0.957 & 0.953 & 0.948 & 0.943 & 0.938 & 0.933 & 0.929 & 0.924 & 0.920 & 0.915 & 0.911 & 0.907 & 0.903 & 0.899 & 0.895 & 0.891 \\
	1/49 & 1.000 & 1.000 & 1.000 & 0.999 & 0.998 & 0.996 & 0.993 & 0.990 & 0.986 & 0.982 & 0.978 & 0.973 & 0.969 & 0.964 & 0.959 & 0.954 & 0.949 & 0.945 & 0.940 & 0.935 & 0.931 & 0.926 & 0.922 & 0.918 & 0.913 & 0.909 & 0.905 & 0.901 & 0.897 & 0.894 \\
	1/50 & 1.000 & 1.000 & 1.000 & 0.999 & 0.998 & 0.996 & 0.994 & 0.991 & 0.987 & 0.983 & 0.979 & 0.975 & 0.970 & 0.965 & 0.961 & 0.956 & 0.951 & 0.947 & 0.942 & 0.937 & 0.933 & 0.928 & 0.924 & 0.920 & 0.916 & 0.911 & 0.907 & 0.904 & 0.900 & 0.896 \\
	\bottomrule
\end{tabular}
}
\end{sidewaystable}

	\begin{sidewaystable}
	\caption{Value of $\Lambda_{\frac{d -2k}{\eta}}(m)$ for Theorem \ref{thcorres}.}\label{t2}
	\centering
	\resizebox{22cm}{!}{
	\begin{tabular}{ccccccccccccccccccccccccccccccc}
	\toprule
	$\eta\backslash d-2k$ & 1 & 2 & 3 & 4 & 5 & 6 & 7 & 8 & 9 & 10 & 11 & 12 & 13 & 14 & 15 & 16 & 17 & 18 & 19 & 20 & 21 & 22 & 23 & 24 & 25 & 26 & 27 & 28 & 29 & 30 \\
	\midrule
	1 & 0.753 & 0.696 & 0.665 & 0.644 & 0.628 & 0.617 & 0.607 & 0.599 & 0.592 & 0.586 & 0.581 & 0.577 & 0.572 & 0.569 & 0.565 & 0.562 & 0.559 & 0.556 & 0.554 & 0.552 & 0.549 & 0.547 & 0.545 & 0.543 & 0.541 & 0.540 & 0.538 & 0.537 & 0.535 & 0.534 \\
	1/2 & 0.838 & 0.778 & 0.741 & 0.714 & 0.694 & 0.678 & 0.665 & 0.654 & 0.645 & 0.636 & 0.629 & 0.623 & 0.617 & 0.611 & 0.606 & 0.602 & 0.598 & 0.594 & 0.590 & 0.587 & 0.584 & 0.581 & 0.578 & 0.575 & 0.572 & 0.570 & 0.568 & 0.565 & 0.563 & 0.561 \\
	1/3 & 0.889 & 0.830 & 0.790 & 0.761 & 0.738 & 0.720 & 0.704 & 0.691 & 0.680 & 0.670 & 0.661 & 0.654 & 0.646 & 0.640 & 0.634 & 0.628 & 0.623 & 0.619 & 0.614 & 0.610 & 0.606 & 0.602 & 0.599 & 0.596 & 0.592 & 0.589 & 0.586 & 0.584 & 0.581 & 0.579 \\
	1/4 & 0.921 & 0.866 & 0.827 & 0.796 & 0.772 & 0.752 & 0.735 & 0.721 & 0.708 & 0.697 & 0.687 & 0.678 & 0.670 & 0.663 & 0.656 & 0.650 & 0.644 & 0.638 & 0.633 & 0.629 & 0.624 & 0.620 & 0.616 & 0.612 & 0.608 & 0.605 & 0.602 & 0.598 & 0.595 & 0.592 \\
	1/5 & 0.943 & 0.894 & 0.855 & 0.824 & 0.799 & 0.778 & 0.760 & 0.745 & 0.732 & 0.720 & 0.709 & 0.699 & 0.690 & 0.682 & 0.675 & 0.668 & 0.661 & 0.655 & 0.650 & 0.645 & 0.640 & 0.635 & 0.630 & 0.626 & 0.622 & 0.618 & 0.614 & 0.611 & 0.608 & 0.604 \\
	1/6 & 0.959 & 0.915 & 0.878 & 0.847 & 0.822 & 0.801 & 0.782 & 0.766 & 0.752 & 0.739 & 0.728 & 0.717 & 0.708 & 0.699 & 0.691 & 0.684 & 0.677 & 0.670 & 0.664 & 0.659 & 0.653 & 0.648 & 0.643 & 0.639 & 0.634 & 0.630 & 0.626 & 0.622 & 0.618 & 0.615 \\
	1/7 & 0.970 & 0.931 & 0.896 & 0.867 & 0.841 & 0.820 & 0.801 & 0.784 & 0.770 & 0.756 & 0.745 & 0.734 & 0.724 & 0.715 & 0.706 & 0.698 & 0.691 & 0.684 & 0.677 & 0.671 & 0.666 & 0.660 & 0.655 & 0.650 & 0.645 & 0.641 & 0.636 & 0.632 & 0.628 & 0.624 \\
	1/8 & 0.978 & 0.944 & 0.911 & 0.883 & 0.858 & 0.836 & 0.817 & 0.801 & 0.785 & 0.772 & 0.760 & 0.748 & 0.738 & 0.728 & 0.719 & 0.711 & 0.703 & 0.696 & 0.689 & 0.683 & 0.677 & 0.671 & 0.666 & 0.660 & 0.655 & 0.651 & 0.646 & 0.642 & 0.637 & 0.633 \\
	1/9 & 0.983 & 0.954 & 0.924 & 0.897 & 0.873 & 0.851 & 0.832 & 0.815 & 0.800 & 0.786 & 0.773 & 0.762 & 0.751 & 0.741 & 0.732 & 0.723 & 0.715 & 0.708 & 0.700 & 0.694 & 0.687 & 0.681 & 0.675 & 0.670 & 0.665 & 0.660 & 0.655 & 0.650 & 0.646 & 0.641 \\
	1/10 & 0.988 & 0.962 & 0.935 & 0.909 & 0.885 & 0.864 & 0.845 & 0.828 & 0.813 & 0.799 & 0.786 & 0.774 & 0.763 & 0.753 & 0.743 & 0.734 & 0.726 & 0.718 & 0.711 & 0.704 & 0.697 & 0.691 & 0.685 & 0.679 & 0.673 & 0.668 & 0.663 & 0.658 & 0.654 & 0.649 \\
	1/11 & 0.991 & 0.969 & 0.944 & 0.919 & 0.897 & 0.876 & 0.857 & 0.840 & 0.824 & 0.810 & 0.797 & 0.785 & 0.774 & 0.763 & 0.754 & 0.744 & 0.736 & 0.728 & 0.720 & 0.713 & 0.706 & 0.699 & 0.693 & 0.687 & 0.682 & 0.676 & 0.671 & 0.666 & 0.661 & 0.656 \\
	1/12 & 0.993 & 0.974 & 0.952 & 0.928 & 0.907 & 0.886 & 0.868 & 0.851 & 0.835 & 0.821 & 0.808 & 0.795 & 0.784 & 0.773 & 0.763 & 0.754 & 0.745 & 0.737 & 0.729 & 0.722 & 0.715 & 0.708 & 0.701 & 0.695 & 0.689 & 0.684 & 0.678 & 0.673 & 0.668 & 0.663 \\
	1/13 & 0.995 & 0.979 & 0.958 & 0.936 & 0.915 & 0.896 & 0.878 & 0.861 & 0.845 & 0.831 & 0.817 & 0.805 & 0.794 & 0.783 & 0.773 & 0.763 & 0.754 & 0.746 & 0.738 & 0.730 & 0.723 & 0.716 & 0.709 & 0.703 & 0.697 & 0.691 & 0.685 & 0.680 & 0.675 & 0.670 \\
	1/14 & 0.996 & 0.983 & 0.964 & 0.943 & 0.923 & 0.904 & 0.886 & 0.870 & 0.854 & 0.840 & 0.827 & 0.814 & 0.803 & 0.792 & 0.781 & 0.772 & 0.762 & 0.754 & 0.746 & 0.738 & 0.730 & 0.723 & 0.716 & 0.710 & 0.704 & 0.698 & 0.692 & 0.686 & 0.681 & 0.676 \\
	1/15 & 0.997 & 0.986 & 0.969 & 0.950 & 0.930 & 0.912 & 0.895 & 0.878 & 0.863 & 0.849 & 0.835 & 0.823 & 0.811 & 0.800 & 0.789 & 0.780 & 0.770 & 0.762 & 0.753 & 0.745 & 0.738 & 0.730 & 0.723 & 0.717 & 0.710 & 0.704 & 0.698 & 0.692 & 0.687 & 0.682 \\
	1/16 & 0.998 & 0.988 & 0.973 & 0.955 & 0.937 & 0.919 & 0.902 & 0.886 & 0.871 & 0.856 & 0.843 & 0.831 & 0.819 & 0.808 & 0.797 & 0.787 & 0.778 & 0.769 & 0.760 & 0.752 & 0.744 & 0.737 & 0.730 & 0.723 & 0.717 & 0.710 & 0.704 & 0.698 & 0.693 & 0.687 \\
	1/17 & 0.998 & 0.990 & 0.976 & 0.960 & 0.943 & 0.925 & 0.909 & 0.893 & 0.878 & 0.864 & 0.851 & 0.838 & 0.826 & 0.815 & 0.804 & 0.794 & 0.785 & 0.776 & 0.767 & 0.759 & 0.751 & 0.744 & 0.736 & 0.729 & 0.723 & 0.716 & 0.710 & 0.704 & 0.698 & 0.693 \\
	1/18 & 0.999 & 0.992 & 0.979 & 0.964 & 0.948 & 0.931 & 0.915 & 0.900 & 0.885 & 0.871 & 0.858 & 0.845 & 0.833 & 0.822 & 0.811 & 0.801 & 0.792 & 0.783 & 0.774 & 0.765 & 0.757 & 0.750 & 0.742 & 0.735 & 0.729 & 0.722 & 0.716 & 0.710 & 0.704 & 0.698 \\
	1/19 & 0.999 & 0.993 & 0.982 & 0.968 & 0.952 & 0.937 & 0.921 & 0.906 & 0.891 & 0.877 & 0.864 & 0.852 & 0.840 & 0.829 & 0.818 & 0.808 & 0.798 & 0.789 & 0.780 & 0.772 & 0.764 & 0.756 & 0.748 & 0.741 & 0.734 & 0.728 & 0.721 & 0.715 & 0.709 & 0.703 \\
	1/20 & 0.999 & 0.994 & 0.984 & 0.971 & 0.957 & 0.941 & 0.926 & 0.911 & 0.897 & 0.884 & 0.870 & 0.858 & 0.846 & 0.835 & 0.824 & 0.814 & 0.804 & 0.795 & 0.786 & 0.778 & 0.769 & 0.761 & 0.754 & 0.747 & 0.740 & 0.733 & 0.726 & 0.720 & 0.714 & 0.708 \\
	1/21 & 0.999 & 0.995 & 0.986 & 0.974 & 0.960 & 0.946 & 0.931 & 0.917 & 0.903 & 0.889 & 0.876 & 0.864 & 0.852 & 0.841 & 0.830 & 0.820 & 0.810 & 0.801 & 0.792 & 0.783 & 0.775 & 0.767 & 0.759 & 0.752 & 0.745 & 0.738 & 0.731 & 0.725 & 0.719 & 0.713 \\
	1/22 & 1.000 & 0.996 & 0.988 & 0.977 & 0.964 & 0.950 & 0.936 & 0.922 & 0.908 & 0.895 & 0.882 & 0.870 & 0.858 & 0.847 & 0.836 & 0.826 & 0.816 & 0.806 & 0.797 & 0.789 & 0.780 & 0.772 & 0.765 & 0.757 & 0.750 & 0.743 & 0.736 & 0.730 & 0.723 & 0.717 \\
	1/23 & 1.000 & 0.997 & 0.990 & 0.979 & 0.967 & 0.954 & 0.940 & 0.926 & 0.913 & 0.900 & 0.887 & 0.875 & 0.863 & 0.852 & 0.841 & 0.831 & 0.821 & 0.812 & 0.803 & 0.794 & 0.786 & 0.777 & 0.770 & 0.762 & 0.755 & 0.748 & 0.741 & 0.734 & 0.728 & 0.722 \\
	1/24 & 1.000 & 0.997 & 0.991 & 0.981 & 0.970 & 0.957 & 0.944 & 0.931 & 0.917 & 0.905 & 0.892 & 0.880 & 0.869 & 0.857 & 0.847 & 0.836 & 0.826 & 0.817 & 0.808 & 0.799 & 0.791 & 0.782 & 0.774 & 0.767 & 0.759 & 0.752 & 0.745 & 0.739 & 0.732 & 0.726 \\
	1/25 & 1.000 & 0.998 & 0.992 & 0.983 & 0.972 & 0.960 & 0.948 & 0.935 & 0.922 & 0.909 & 0.897 & 0.885 & 0.873 & 0.862 & 0.852 & 0.841 & 0.831 & 0.822 & 0.813 & 0.804 & 0.795 & 0.787 & 0.779 & 0.772 & 0.764 & 0.757 & 0.750 & 0.743 & 0.737 & 0.730 \\
	1/26 & 1.000 & 0.998 & 0.993 & 0.985 & 0.975 & 0.963 & 0.951 & 0.938 & 0.926 & 0.913 & 0.901 & 0.890 & 0.878 & 0.867 & 0.856 & 0.846 & 0.836 & 0.827 & 0.818 & 0.809 & 0.800 & 0.792 & 0.784 & 0.776 & 0.769 & 0.761 & 0.754 & 0.747 & 0.741 & 0.734 \\
	1/27 & 1.000 & 0.998 & 0.994 & 0.986 & 0.977 & 0.966 & 0.954 & 0.942 & 0.930 & 0.918 & 0.906 & 0.894 & 0.883 & 0.872 & 0.861 & 0.851 & 0.841 & 0.831 & 0.822 & 0.813 & 0.805 & 0.796 & 0.788 & 0.780 & 0.773 & 0.765 & 0.758 & 0.751 & 0.745 & 0.738 \\
	1/28 & 1.000 & 0.999 & 0.995 & 0.988 & 0.979 & 0.968 & 0.957 & 0.945 & 0.933 & 0.921 & 0.910 & 0.898 & 0.887 & 0.876 & 0.865 & 0.855 & 0.845 & 0.836 & 0.827 & 0.818 & 0.809 & 0.801 & 0.792 & 0.785 & 0.777 & 0.770 & 0.762 & 0.755 & 0.749 & 0.742 \\
	1/29 & 1.000 & 0.999 & 0.995 & 0.989 & 0.981 & 0.971 & 0.960 & 0.948 & 0.937 & 0.925 & 0.913 & 0.902 & 0.891 & 0.880 & 0.870 & 0.860 & 0.850 & 0.840 & 0.831 & 0.822 & 0.813 & 0.805 & 0.797 & 0.789 & 0.781 & 0.774 & 0.766 & 0.759 & 0.753 & 0.746 \\
	1/30 & 1.000 & 0.999 & 0.996 & 0.990 & 0.982 & 0.973 & 0.962 & 0.951 & 0.940 & 0.929 & 0.917 & 0.906 & 0.895 & 0.884 & 0.874 & 0.864 & 0.854 & 0.844 & 0.835 & 0.826 & 0.817 & 0.809 & 0.801 & 0.793 & 0.785 & 0.778 & 0.770 & 0.763 & 0.756 & 0.750 \\
	1/31 & 1.000 & 0.999 & 0.996 & 0.991 & 0.984 & 0.975 & 0.965 & 0.954 & 0.943 & 0.932 & 0.921 & 0.910 & 0.899 & 0.888 & 0.878 & 0.868 & 0.858 & 0.848 & 0.839 & 0.830 & 0.821 & 0.813 & 0.805 & 0.797 & 0.789 & 0.781 & 0.774 & 0.767 & 0.760 & 0.753 \\
	1/32 & 1.000 & 0.999 & 0.997 & 0.992 & 0.985 & 0.977 & 0.967 & 0.957 & 0.946 & 0.935 & 0.924 & 0.913 & 0.902 & 0.892 & 0.882 & 0.871 & 0.862 & 0.852 & 0.843 & 0.834 & 0.825 & 0.817 & 0.808 & 0.800 & 0.793 & 0.785 & 0.778 & 0.770 & 0.763 & 0.757 \\
	1/33 & 1.000 & 0.999 & 0.997 & 0.993 & 0.986 & 0.978 & 0.969 & 0.959 & 0.949 & 0.938 & 0.927 & 0.916 & 0.906 & 0.895 & 0.885 & 0.875 & 0.865 & 0.856 & 0.847 & 0.838 & 0.829 & 0.820 & 0.812 & 0.804 & 0.796 & 0.789 & 0.781 & 0.774 & 0.767 & 0.760 \\
	1/34 & 1.000 & 1.000 & 0.998 & 0.993 & 0.987 & 0.980 & 0.971 & 0.961 & 0.951 & 0.941 & 0.930 & 0.920 & 0.909 & 0.899 & 0.889 & 0.879 & 0.869 & 0.860 & 0.850 & 0.841 & 0.833 & 0.824 & 0.816 & 0.808 & 0.800 & 0.792 & 0.785 & 0.777 & 0.770 & 0.763 \\
	1/35 & 1.000 & 1.000 & 0.998 & 0.994 & 0.988 & 0.981 & 0.973 & 0.963 & 0.954 & 0.943 & 0.933 & 0.923 & 0.912 & 0.902 & 0.892 & 0.882 & 0.873 & 0.863 & 0.854 & 0.845 & 0.836 & 0.828 & 0.819 & 0.811 & 0.803 & 0.796 & 0.788 & 0.781 & 0.774 & 0.767 \\
	1/36 & 1.000 & 1.000 & 0.998 & 0.995 & 0.989 & 0.982 & 0.974 & 0.965 & 0.956 & 0.946 & 0.936 & 0.926 & 0.915 & 0.905 & 0.895 & 0.885 & 0.876 & 0.866 & 0.857 & 0.848 & 0.840 & 0.831 & 0.823 & 0.815 & 0.807 & 0.799 & 0.791 & 0.784 & 0.777 & 0.770 \\
	1/37 & 1.000 & 1.000 & 0.998 & 0.995 & 0.990 & 0.984 & 0.976 & 0.967 & 0.958 & 0.948 & 0.938 & 0.928 & 0.918 & 0.908 & 0.898 & 0.889 & 0.879 & 0.870 & 0.861 & 0.852 & 0.843 & 0.834 & 0.826 & 0.818 & 0.810 & 0.802 & 0.795 & 0.787 & 0.780 & 0.773 \\
	1/38 & 1.000 & 1.000 & 0.999 & 0.996 & 0.991 & 0.985 & 0.977 & 0.969 & 0.960 & 0.951 & 0.941 & 0.931 & 0.921 & 0.911 & 0.901 & 0.892 & 0.882 & 0.873 & 0.864 & 0.855 & 0.846 & 0.838 & 0.829 & 0.821 & 0.813 & 0.805 & 0.798 & 0.790 & 0.783 & 0.776 \\
	1/39 & 1.000 & 1.000 & 0.999 & 0.996 & 0.992 & 0.986 & 0.979 & 0.971 & 0.962 & 0.953 & 0.943 & 0.934 & 0.924 & 0.914 & 0.904 & 0.895 & 0.885 & 0.876 & 0.867 & 0.858 & 0.849 & 0.841 & 0.832 & 0.824 & 0.816 & 0.809 & 0.801 & 0.793 & 0.786 & 0.779 \\
	1/40 & 1.000 & 1.000 & 0.999 & 0.996 & 0.992 & 0.987 & 0.980 & 0.972 & 0.964 & 0.955 & 0.946 & 0.936 & 0.926 & 0.917 & 0.907 & 0.898 & 0.888 & 0.879 & 0.870 & 0.861 & 0.852 & 0.844 & 0.836 & 0.827 & 0.819 & 0.812 & 0.804 & 0.796 & 0.789 & 0.782 \\
	1/41 & 1.000 & 1.000 & 0.999 & 0.997 & 0.993 & 0.988 & 0.981 & 0.974 & 0.966 & 0.957 & 0.948 & 0.938 & 0.929 & 0.919 & 0.910 & 0.901 & 0.891 & 0.882 & 0.873 & 0.864 & 0.855 & 0.847 & 0.839 & 0.830 & 0.822 & 0.815 & 0.807 & 0.799 & 0.792 & 0.785 \\
	1/42 & 1.000 & 1.000 & 0.999 & 0.997 & 0.994 & 0.989 & 0.982 & 0.975 & 0.967 & 0.959 & 0.950 & 0.941 & 0.931 & 0.922 & 0.913 & 0.903 & 0.894 & 0.885 & 0.876 & 0.867 & 0.858 & 0.850 & 0.842 & 0.833 & 0.825 & 0.817 & 0.810 & 0.802 & 0.795 & 0.788 \\
	1/43 & 1.000 & 1.000 & 0.999 & 0.997 & 0.994 & 0.989 & 0.983 & 0.977 & 0.969 & 0.961 & 0.952 & 0.943 & 0.934 & 0.924 & 0.915 & 0.906 & 0.897 & 0.888 & 0.879 & 0.870 & 0.861 & 0.853 & 0.844 & 0.836 & 0.828 & 0.820 & 0.813 & 0.805 & 0.798 & 0.790 \\
	1/44 & 1.000 & 1.000 & 0.999 & 0.998 & 0.995 & 0.990 & 0.984 & 0.978 & 0.970 & 0.962 & 0.954 & 0.945 & 0.936 & 0.927 & 0.918 & 0.908 & 0.899 & 0.890 & 0.881 & 0.873 & 0.864 & 0.856 & 0.847 & 0.839 & 0.831 & 0.823 & 0.815 & 0.808 & 0.800 & 0.793 \\
	1/45 & 1.000 & 1.000 & 0.999 & 0.998 & 0.995 & 0.991 & 0.985 & 0.979 & 0.972 & 0.964 & 0.956 & 0.947 & 0.938 & 0.929 & 0.920 & 0.911 & 0.902 & 0.893 & 0.884 & 0.875 & 0.867 & 0.858 & 0.850 & 0.842 & 0.834 & 0.826 & 0.818 & 0.810 & 0.803 & 0.796 \\
	1/46 & 1.000 & 1.000 & 0.999 & 0.998 & 0.995 & 0.991 & 0.986 & 0.980 & 0.973 & 0.965 & 0.957 & 0.949 & 0.940 & 0.931 & 0.922 & 0.913 & 0.904 & 0.895 & 0.887 & 0.878 & 0.869 & 0.861 & 0.853 & 0.844 & 0.836 & 0.828 & 0.821 & 0.813 & 0.806 & 0.798 \\
	1/47 & 1.000 & 1.000 & 1.000 & 0.998 & 0.996 & 0.992 & 0.987 & 0.981 & 0.974 & 0.967 & 0.959 & 0.951 & 0.942 & 0.933 & 0.924 & 0.916 & 0.907 & 0.898 & 0.889 & 0.881 & 0.872 & 0.864 & 0.855 & 0.847 & 0.839 & 0.831 & 0.823 & 0.816 & 0.808 & 0.801 \\
	1/48 & 1.000 & 1.000 & 1.000 & 0.998 & 0.996 & 0.993 & 0.988 & 0.982 & 0.976 & 0.968 & 0.961 & 0.952 & 0.944 & 0.935 & 0.927 & 0.918 & 0.909 & 0.900 & 0.892 & 0.883 & 0.874 & 0.866 & 0.858 & 0.850 & 0.842 & 0.834 & 0.826 & 0.818 & 0.811 & 0.803 \\
	1/49 & 1.000 & 1.000 & 1.000 & 0.999 & 0.996 & 0.993 & 0.989 & 0.983 & 0.977 & 0.970 & 0.962 & 0.954 & 0.946 & 0.937 & 0.929 & 0.920 & 0.911 & 0.903 & 0.894 & 0.885 & 0.877 & 0.869 & 0.860 & 0.852 & 0.844 & 0.836 & 0.828 & 0.821 & 0.813 & 0.806 \\
	1/50 & 1.000 & 1.000 & 1.000 & 0.999 & 0.997 & 0.994 & 0.989 & 0.984 & 0.978 & 0.971 & 0.964 & 0.956 & 0.948 & 0.939 & 0.931 & 0.922 & 0.914 & 0.905 & 0.896 & 0.888 & 0.879 & 0.871 & 0.863 & 0.855 & 0.847 & 0.839 & 0.831 & 0.823 & 0.816 & 0.808 \\
	\bottomrule
	\end{tabular}
}
\end{sidewaystable}

\section{Supplementary for experimental methods}\label{supmethod}
\subsection{DSRS algorithm}\label{DSRSalg}
Taking EGG as an example, the DSRS algorithm used in this work is shown in Algorithm \ref{alg:DSRS}.
\begin{algorithm}[t!]
    \KwIn{base classifier $f$, example $x$, formal variance $\sigma$, exponent $\eta$, hyperparameter $k$, significance level $\alpha$, error bound for certified radius $e$, heuristic algorithm $H$, conservative algorithm $C$, DualBinarySearch algorithm $D$ (Algorithms $H, C, D$ from \citet{li2022}).}
    \KwOut{certified radius $r$}
    \setcounter{AlgoLine}{0}
    Initialize the smoothing distribution as EGG: $\mathcal{P}\leftarrow\mathcal{G}(\sigma, \eta, k)$
    
    $A_1\leftarrow$ SamplingUnderNoise $(f, x, \mathcal{P}, \alpha)$
    
    \tcc{$A_1$ is the Clopper-Pearson lower bound for the sampling result~\citep{clopper1934, cohen2019}}
    
    $T\leftarrow$ $H (A_1)$ \hfill\tcc{$T$ is the hyperparameter for the truncated distribution $\mathcal{Q}$}
    
    Initialize the supplementary distribution as TEGG: $\mathcal{Q}\leftarrow\mathcal{G}_t(\sigma, \eta, k, T)$
    
    $B_1\leftarrow$ SamplingUnderNoise $(f, x, \mathcal{Q}, \alpha)$ 
    
    $A, B \leftarrow C (A_1, B_1)$ \hfill\tcc{$A,B$ for Problem (\ref{dualdsrs}) are determined by algorithm $C$}
    
    $r_l \leftarrow 0, r_r \leftarrow I$
    
    \tcc{Initialization for the binary search on $r$, where $I$ is big enough}
    
    \While{$r_r - r_l > e$}{
        $r_m \leftarrow (r_r + r_l) / 2$ 
        
        $p_m \leftarrow D(A, B)$ \hfill\tcc{Problem (\ref{dualdsrs}) can be solved by $D$ with given $A, B$}
		\eIf{$p_m > 1/2$} { 
			$r_l \leftarrow r_m$ 
		}{
			$r_r \leftarrow r_m$
		}
	}
    $r\leftarrow r_l$
    
    \Return $r$
    \caption{Standard algorithm for Double Sampling Randomized Smoothing by Exponential General Gaussian (EGG) distributions on real-world datasets}\label{alg:DSRS}
\end{algorithm}

\subsection{Computational overhead}
All of our experiments on real-world datasets are composed of sampling and certification, which are finished with 4 NVIDIA RTX 3080 GPUs and CPUs. The most computationally intensive procedure is sampling. For $\sigma=0.50$ base classifiers, it takes about $5$s, $200$s to sample $50000$ times under noise distributions on CIFAR10, ImageNet with one GPU. Given that we uniformly pick $1000$ data points from each dataset, one sampling procedure takes around $1$ hour, $1$ day for $50000$ noises on CIFAR10, ImageNet respectively with one GPU. Naturally, the sampling time almost doubles for $100000$ noises. 

The computation for certification only relies on CPUs. The running time for certification includes NP certification and DSRS certification, and is basically constant for different datasets. Usually, it takes 1-2 days to complete. The overall computational time for standard DSRS certification is strictly larger than that for pure NP certification, with respect to a specific number of samples. For instance, if we compute a certificate for 100000 noises, we need one NP certification for 100000 noises, while we need one NP certification for 50000 noises and one DSRS certification for the remaining 50000 noises to finish the DSRS computation. Generally, the computational time for standard DSRS certification is one to two times that of pure NP certification.

\subsection{Transferability of the incremental effect}
We show results for real-world datasets on different base classifiers in this section. In Table \ref{transfer_inc}, we observe that $\eta=8.0$ shows better performance than General Gaussian ($\eta=2.0$) overall. This alludes to some defects in the current training method for General Gaussian because intuitively, if the models are trained by General Gaussian, then sampling with General Gaussian should have provided the best certified results among all $\eta$, but we find the incremental effect with $\eta$ still exists in General-Gaussian-augmented models.  We name the base classifiers by \{training method\}-\{distribution\} in experimental results for ACR. In all the experiments in this subsection, we set the training method $\in$ \{StdAug (standard augmentation), Consistency, SmoothMix\}, the training distribution $\in$ \{GS (Gaussian), GGS (General Gaussian)\}. We use $\sigma=0.50$ for all the training and sampling distributions.

\begin{table}[t!]
\tiny
\centering
\caption{The incremental effect with $\eta$ on different base classifiers.}\label{transfer_inc}
 \label{sim_imgnet}
 \resizebox{\textwidth}{!}{
\begin{tabular}{ccccc cccc ccccc} 
\toprule 
\multirow{2}{*}{Dataset}& \multirow{2}{*}{Model}&\multirow{2}{*}{$\eta$}&\multicolumn{10}{c}{Certified Accuracy at $r$} & \multirow{2}{*}{$ACR$} \\
\cline{4-13}
&&&0.15&0.30&0.45&0.60&0.75&0.90&1.05&1.20&1.35&1.50\\
\hline
\multirow{8}{*}{CIFAR-10}&\multirow{2}{*}{StdAug-GS}&2.0(\citep{li2022}, the SOTA)&56.9\% &        48.0\% &        40.9\% &        33.5\% &        26.4\% &        19.5\% &        13.8\% &        10.6\% &        7.4\% &         3.3\% & 0.437\\
&&8.0&58.1\% &        50.8\% &        42.8\% &        35.9\% &        30.6\% &        23.2\% &        18.1\% &        13.7\% &        10.1\% &        7.3\%  & 0.489\\
&\multirow{2}{*}{StdAug-GGS}&2.0(\citep{li2022}, the SOTA)&58.3\% &        51.7\% &        44.5\% &        38.1\% &        29.3\% &        22.7\% &        17.7\% &        13.1\% &        8.3\% &         3.8\% &0.480\\
&&8.0&58.2\% &        52.1\% &        44.4\% &        38.8\% &        30.9\% &        24.1\% &        19.1\% &        14.4\% &        10.8\% &        6.8\% &0.502\\
&\multirow{2}{*}{Consistency-GGS}&2.0(\citep{li2022}, the SOTA)&52.0\% &        48.8\% &        44.7\% &        42.1\% &        38.5\% &        36.0\% &        32.9\% &        28.6\% &        24.0\% &        19.6\% &0.618\\
&&8.0&51.7\% &        48.6\% &        44.7\% &        42.2\% &        38.8\% &        36.1\% &        34.1\% &        29.9\% &        26.8\% &        22.3\% &0.650\\
&\multirow{2}{*}{SmoothMix-GGS}&2.0(\citep{li2022}, the SOTA)&55.8\% &        52.3\% &        49.1\% &        45.3\% &        41.7\% &        37.7\% &        34.6\% &        30.0\% &        26.3\% &        21.2\% &0.662\\
&&8.0&55.5\% &        52.1\% &        49.1\% &        45.6\% &        42.0\% &        38.2\% &        35.2\% &        31.7\% &        27.8\% &        24.3\% &0.695\\
\hline
\multirow{2}{*}{ImageNet}&\multirow{2}{*}{StdAug-GGS}&2.0(\citep{li2022}, the SOTA)&56.7\% &        52.8\% &        49.4\% &        45.7\% &        41.5\% &        38.1\% &        34.0\% &        30.6\% &        25.4\% &        19.4\% &0.654\\
&&8.0&56.7\% &        52.9\% &        49.8\% &        46.3\% &        42.9\% &        39.3\% &        35.8\% &        32.7\% &        29.1\% &        24.6\% &0.703\\

\bottomrule 
\end{tabular}
}
\end{table}

\section{Numerical simulation}\label{apspns}
We also conduct a numerical simulation to explore the effects of EGG distributions on currently unattainable $A, B$ pairs from Problem (\ref{dualdsrs}).
We only show simulative experiments for the EGG distribution since we observe great monotonicity for certified robustness (and/or certified radius) \wrt $\eta$ from EGG, which does not occur in ESG.
We consider two cases: $B=1$ and $B < 1$ for $B$ in Problem (\ref{dualdsrs}). 

\textbf{Results and analyses.} (1) $B = 1$. This is actually an ideal case since it is impossible to train such a base classifier. Here we still present the results, as they demonstrate the theoretical performance of EGG distributions. From Figure \ref{main_sim} (left), we observe there is a monotonically decreasing tendency in the certified radius \wrt $\eta$, which seems contradictory to the results on real-world datasets. 
This is partly because the concentration assumption (\ie, $B=1$) is more friendly to the smaller $\eta$. 
In fact, the major mass of the EGG distribution with smaller $\eta$ gathers near $0$. This makes smaller $\eta$ more sensitive to relaxation of concentration assumption (see Figure \ref{figrelax}), which leads to worse performance on real-world datasets compared to larger $\eta$. (2) $B < 1$. As shown in Table \ref{sim_imgnet}, the monotonically increasing tendency \wrt $\eta$ continues to exist, manifesting the essential superiority of the large $\eta$ EGG distribution in the DSRS framework. We also notice that the increase is not endless, since $\eta=64.0$ shows a marginal increment to that of $\eta=32.0$ (see also Figure \ref{main_sim} (right)), which may imply some convergence due to the extremely slow growth. We do not show results for larger $\eta$ due to floating-point limitations and low necessity based on our observation. For all the distributions used for numerical simulation, we set $\sigma=1.0$ and $k=\frac{d}{2} - 5$ for fair comparison. 

\subsection{$B=1$ settings}
We need to modify Theorem \ref{thegg} slightly to get certified results. We see
\begin{equation}
\begin{aligned}
B = 1&\Longleftrightarrow C_g\mathbb{E}_{u \sim \Gamma(\frac{d-2k}{\eta}, 1)}\ \omega_1(u, \nu_1 + C_g\nu_2) \cdot \mathds{1}_{u\leq \frac{T^\eta}{2\sigma_g^\eta}} = 1\\
 &\Longleftrightarrow \nu_1 + C_g\nu_2 \to -\infty,
\end{aligned}
\end{equation}
which makes 
\begin{equation}
\omega_2(u) = \Psi_{\frac{d-1}{2}}\left(\frac{T^2- (t - \rho) ^ 2}{4\rho t}\right).
\end{equation}

Injecting $\omega_2(u)$ into Theorem \ref{thegg}, we can compute the certified radius as in general cases. In $B=1$ experiments, $T$ is set to $\sqrt{2\Lambda_{\frac{d}{2}}(0.5)}$ due to the assumption of the $(1, 0.5, 2)$-concentration property.

\subsection{$B<1$ settings}
Theorem \ref{thegg} can be directly used in computing the certified radius for this case. The values set for $A$ and $B$ are not completely random as there is an inherent constraint~\citep{li2022}:
\begin{equation}
\left\{\begin{aligned}&\frac{B}{C_g} \leq A \leq 1 - \frac{1 - B}{C_g},\\
&0 \leq B \leq 1.
\end{aligned}
\right.
\end{equation}
This constraint should also be considered when setting $A$ in case $B=1$. For each EGG distribution, we set $C_g=2$, meaning $T=\sigma_g(2\Lambda^{-1}_{\frac{d-2k}{\eta}}(0.5))^{\frac{1}{\eta}}$ for Figure \ref{main_sim} (right) and Table \ref{sim_imgnet}.

\begin{table}[!tb]
\scriptsize
\centering
\caption{Numerical simulation for EGG distributions (metric: certified radius).}
 \label{sim_imgnet}
\begin{tabular}{ccccccccccccc} 
\toprule 
\multirow{2}{*}{$\eta$}& \multicolumn{1}{c}{A} &\multicolumn{4}{c}{0.6} & \multicolumn{4}{c}{0.7} & \multicolumn{3}{c}{0.8}\\
\cmidrule(lr){2-2}\cmidrule(lr){3-6}\cmidrule(lr){7-10} \cmidrule(lr){11-13}
&\multicolumn{1}{c}{B}&0.6 & 0.7 & 0.8 & 0.9 & 0.6 & 0.7 & 0.8 & 0.9 & 0.7 & 0.8 & 0.9 \\
\hline
0.5 && 0.188& 0.194& 0.216& 0.273& 0.408& 0.391& 0.407& 0.471& 0.678& 0.632& 0.675\\
1.0 && 0.218& 0.225& 0.251& 0.320& 0.471& 0.451& 0.470& 0.546& 0.778& 0.726& 0.776\\
2.0 && 0.234& 0.242& 0.271& 0.346& 0.506& 0.485& 0.505& 0.589& 0.836& 0.779& 0.833\\
4.0 && 0.243& 0.251& 0.281& 0.360& 0.525& 0.502& 0.524& 0.611& 0.867& 0.807& 0.863\\
8.0 && 0.247& 0.255& 0.286& 0.367& 0.534& 0.511& 0.533& 0.622& 0.882& 0.821& 0.878\\
16.0 &&  0.249& 0.257& 0.288& 0.370& 0.538& 0.515& 0.537&  0.627&  0.889&  0.827&  0.885\\
32.0 && 0.250& 0.258& 0.289& 0.371& 0.540& 0.517& 0.539& 0.629& 0.893& 0.830& 0.887\\
64.0 && 0.250& 0.258& 0.290& 0.371& 0.541& 0.518& 0.539& 0.629& 0.894& 0.831& 0.888\\
\hline
Increase (64.0 to 2.0)&& 6.8\%& 6.6\%& 7.0\%& 7.2\%& 6.9\%& 6.8\%& 6.7\%& 6.8\%& 6.9\%& 6.7\%& 6.6\%\\

\bottomrule
\end{tabular}
\end{table}

\begin{figure}[t!]
	\centering
	\begin{subfigure}{0.39\linewidth}
		\centering
		\includegraphics[height=0.75\linewidth, width=1.0\linewidth]{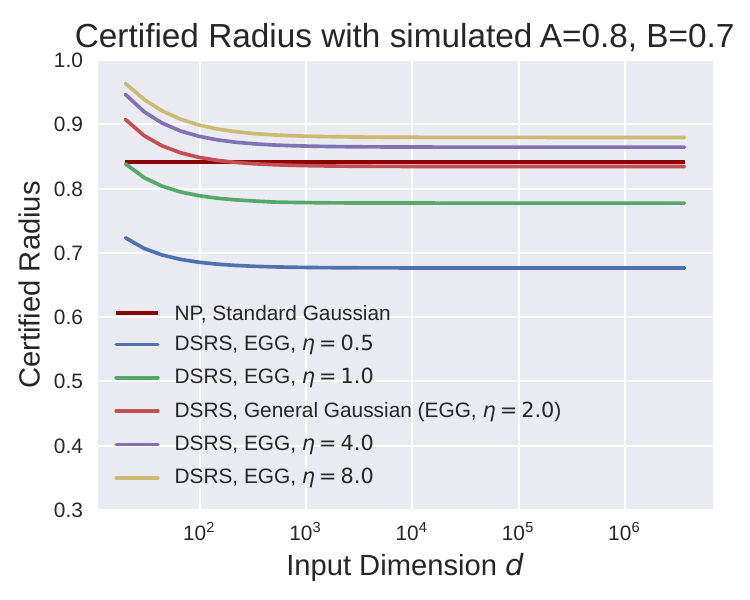}
		\caption{Certified radius \vs $d$, $A=0.8, B=0.7$.}\label{figdim}
	\end{subfigure}
	\centering
	\begin{subfigure}{0.59\linewidth}
		\centering
		\includegraphics[height=0.5\linewidth, width=1.0\linewidth]{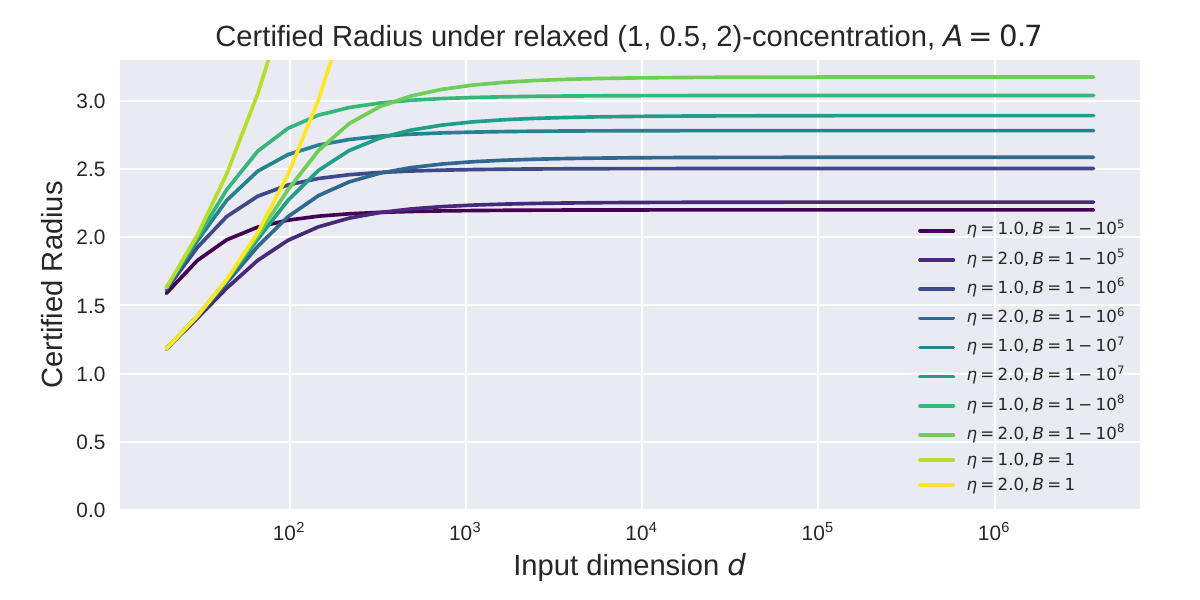}
		\caption{Certified radius with $\eta=1$ \vs $\eta=2$ under relaxed concentration assumption.}\label{figrelax}
	\end{subfigure}
	\caption{Results for numerical simulation.}
 \label{num_sim}
\end{figure}

\subsection{The effect of dimension}
The results in Table \ref{sim_imgnet} almost do not change on large $d$. That is, our results that certified radius increases monotonically with $\eta$ are general for common datasets like CIFAR-10 and ImageNet. See Figure \ref{figdim}. 

\subsection{Different sensitivity to relaxation for $\eta$} 
We observe contrary monotonicity for certified radius $\wrt$ $\eta$ under $B=1$ and $B<1$. A direct reason is that small $\eta$ in EGG is more susceptible to relaxation of $B=1$. We set $T=\sqrt{2\Lambda_{\frac{d}{2}}(0.5)}$ for $B<1$ curves in this figure for fair comparison with $B=1$ ones. Figure \ref{figrelax} demonstrates that the smaller $\eta$ performs better when $B=1$, while performs worse when $B<1$ than the larger $\eta$.  This indicates that the smaller $\eta$ suffers more from the relaxation of the concentration assumption than the larger $\eta$. 

\section{Comparison to other work}
\citet{li2022} is one of the cornerstones of this work. To the best of our knowledge, we are the first to generalize the DSRS certification systematically. For NP certification, though many attempts in the community have investigated the interrelationship between the smoothing distribution and certified radius, nobody has shown results as we do.~\citet{yang2020} trained models for each smoothing distribution, while we fixed the base classifier for the convenience of comparison. They also did not fully consider the distributions we use.~\citet{kumar2020} showed similar results to us for ESG in their Figure 5, but they only showed sampling results due to the lack of computing method, and they also trained base models for each distribution like~\citet{yang2020}. ~\citet{zhang2020} showed results for $\eta=2$ EGG distribution under a fixed-model setting, but they didn't consider $\eta\neq2$ cases. 

\section{Supplementary for experimental results}\label{apper}
\subsection{Certified radius at $r$ for Consistency~\citep{jeong2020} and SmoothMix~\citep{jeong2021} models, maximum results}\label{supex}
Table \ref{best_acc_consis} and Table \ref{best_acc_mix} show the experimental results for certified accuracy at radius $r$. Each data is the \textbf{maximum} one among base classifiers with $\sigma\in\{0.25, 0.50, 1.00\}$. All the base classifiers in both tables are trained under the General Gaussian distribution, by Consistency~\citep{jeong2020} and SmoothMix~\citep{jeong2021} respectively. From the tables, we can see the rule observed on classifiers augmented standardly continues to exist, that certified accuracy increases monotonically with the $\eta$ of EGG, and stays almost constant with the $\eta$ of ESG.
\begin{table}[t!] 
\centering
\caption{Maximum certified accuracy \wrt $\sigma$, Consistency models}\label{best_acc_consis}
\resizebox{1.0\linewidth}{!}{%
\begin{tabular}{ccccc ccccc ccccc c} 
\toprule 
\multirow{2}{*}{Dataset} & \multirow{2}{*}{Method} &\multicolumn{14}{c}{Certified accuracy at $r$}\\
\cline{3-16}
&&  0.25  &  0.50  &  0.75  &  1.00  &  1.25  &  1.50  &  1.75  &  2.00  &  2.25  &  2.50  &  2.75  &  3.00  &  3.25  &  3.50 \\
\hline
\multirow{4}{*}{CIFAR10}& EGG, $\eta=1.0$ &62.1\% &50.7\% &38.2\% &33.1\% &24.2\% &19.2\% &16.7\% &14.3\% &11.3\% &9.2\% &6.6\% &4.1\% &1.4\% &0.0\% \\
& EGG, $\eta=2.0$ &62.5\% &52.0\% &38.5\% &34.4\% &27.4\% &20.6\% &17.0\% &14.7\% &12.7\% &10.5\% &8.5\% &6.3\% &3.9\% &2.5\%\\
& EGG, $\eta=4.0$ &62.5\% &52.2\% &39.1\% &35.4\% &28.3\% &21.1\% &17.5\% &15.3\% &13.0\% &10.9\% &9.2\% &7.0\% &5.2\% &3.1\%   \\
& EGG, $\eta=8.0$ &62.5\% &52.6\% &40.4\% &35.3\% &28.6\% &22.3\% &17.6\% &15.5\% &13.2\% &11.3\% &9.6\% &7.8\% &5.8\% &3.9\%  \\
\hline
\multirow{4}{*}{CIFAR10}& ESG, $\eta=1.0$ &62.6\% &52.9\% &41.6\% &35.5\% &29.3\% &23.7\% &17.7\% &15.8\% &13.7\% &11.8\% &10.0\% &8.8\% &6.8\% &4.6\% \\
& ESG, $\eta=2.0$ &62.7\% &53.0\% &41.4\% &35.5\% &29.4\% &24.0\% &17.7\% &16.0\% &13.8\% &11.9\% &10.1\% &8.7\% &6.7\% &4.4\% \\
& ESG, $\eta=4.0$ &62.7\% &52.9\% &41.6\% &35.5\% &29.5\% &23.8\% &17.8\% &15.7\% &13.9\% &11.9\% &10.2\% &8.8\% &6.7\% &4.8\% \\
& ESG, $\eta=8.0$ &62.7\% &52.9\% &41.7\% &35.5\% &29.1\% &23.8\% &17.6\% &15.8\% &13.6\% &11.8\% &10.1\% &8.3\% &6.7\% &4.8\% \\
\bottomrule
\end{tabular}
}
\end{table}

\begin{table}[t!] 
\centering
\caption{Maximum certified accuracy \wrt $\sigma$, SmoothMix models}\label{best_acc_mix}
\resizebox{1.0\linewidth}{!}{%
\begin{tabular}{ccccc ccccc ccccc c} 
\toprule 
\multirow{2}{*}{Dataset} & \multirow{2}{*}{Method} &\multicolumn{14}{c}{Certified accuracy at $r$}\\
\cline{3-16}
&&  0.25  &  0.50  &  0.75  &  1.00  &  1.25  &  1.50  &  1.75  &  2.00  &  2.25  &  2.50  &  2.75  &  3.00  &  3.25  &  3.50 \\
\hline
\multirow{4}{*}{CIFAR10}& EGG, $\eta=1.0$ &63.7\% &53.8\% &40.9\% &34.3\% &26.6\% &21.1\% &17.0\% &14.2\% &10.5\% &7.7\% &4.0\% &1.5\% &0.1\% &0.0\% \\
& EGG, $\eta=2.0$ &64.5\% &55.0\% &41.7\% &35.6\% &28.9\% &21.3\% &18.0\% &15.2\% &12.3\% &9.7\% &6.4\% &3.7\% &1.4\% &0.4\%\\
& EGG, $\eta=4.0$ &64.4\% &55.5\% &43.0\% &35.9\% &29.5\% &23.4\% &18.3\% &15.8\% &12.8\% &10.2\% &7.7\% &4.6\% &2.1\% &0.9\%   \\
& EGG, $\eta=8.0$ &64.7\% &55.7\% &43.9\% &36.2\% &30.1\% &24.3\% &18.6\% &15.8\% &13.2\% &10.6\% &8.0\% &5.4\% &2.7\% &1.3\%  \\
\hline
\multirow{4}{*}{CIFAR10}& ESG, $\eta=1.0$ &64.6\% &56.5\% &46.5\% &36.6\% &31.3\% &25.6\% &19.1\% &16.2\% &13.4\% &11.4\% &8.9\% &6.3\% &4.0\% &1.8\% \\
& ESG, $\eta=2.0$ &64.7\% &56.3\% &45.8\% &36.6\% &31.2\% &25.6\% &18.9\% &16.3\% &13.4\% &11.6\% &8.8\% &6.2\% &4.0\% &1.7\% \\
& ESG, $\eta=4.0$ &64.6\% &56.3\% &46.0\% &36.5\% &31.0\% &25.6\% &19.0\% &16.2\% &13.6\% &11.5\% &9.2\% &6.0\% &3.9\% &1.8\% \\
& ESG, $\eta=8.0$ &64.6\% &56.1\% &46.5\% &36.4\% &31.1\% &25.5\% &18.9\% &16.3\% &13.5\% &11.3\% &9.0\% &6.2\% &4.1\% &1.6\% \\
\bottomrule 
\end{tabular}
}
\end{table}

\begin{figure}[t!]
	\centering
	\begin{subfigure}{0.24\linewidth}
		\centering
		\includegraphics[height=0.8\linewidth, width=1.0\linewidth]{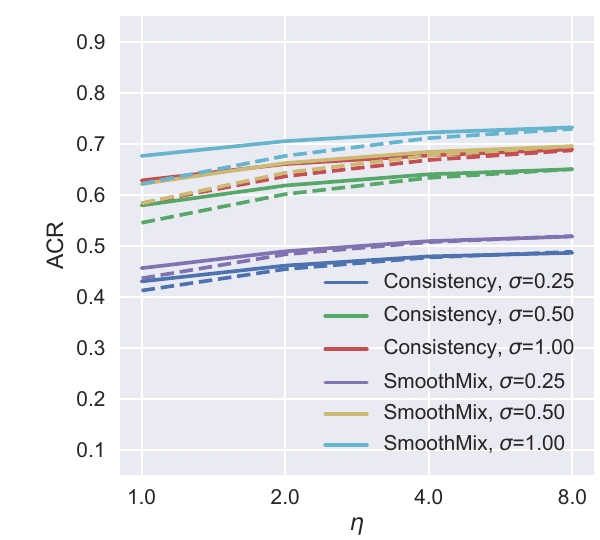}
		\caption{}\label{f11a}
	\end{subfigure}
	\centering
	\begin{subfigure}{0.24\linewidth}
		\centering
		\includegraphics[height=0.8\linewidth, width=1.0\linewidth]{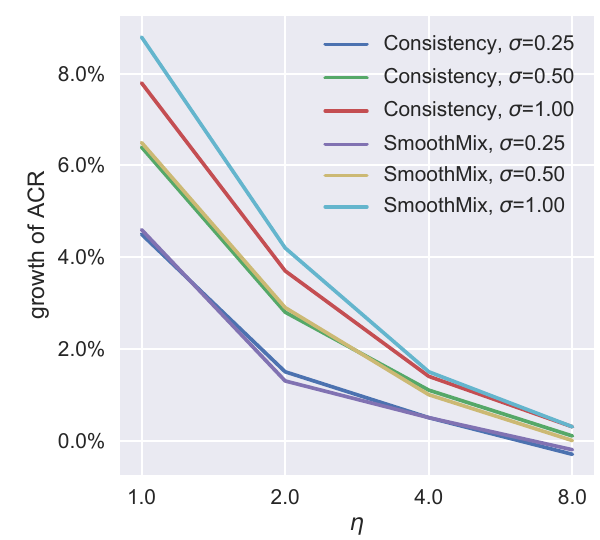}
		\caption{}\label{f11b}
	\end{subfigure}
    \centering
	\begin{subfigure}{0.24\linewidth}
		\centering
		\includegraphics[height=0.8\linewidth, width=1.0\linewidth]{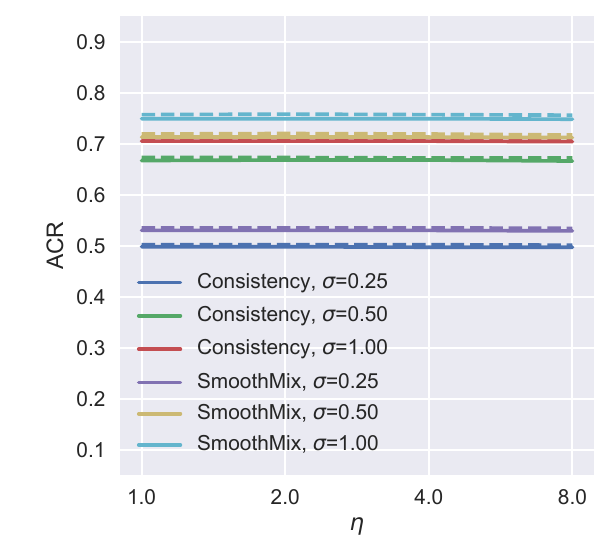}
		\caption{}\label{f11c}
	\end{subfigure}
    \centering
	\begin{subfigure}{0.24\linewidth}
		\centering
		\includegraphics[height=0.8\linewidth, width=1.0\linewidth]{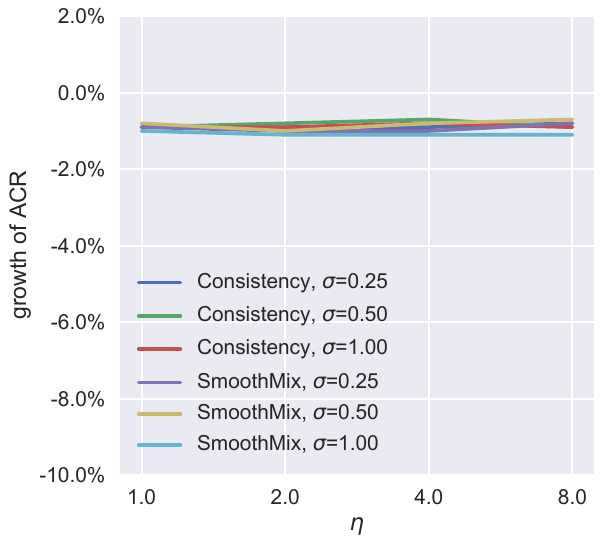}
		\caption{}\label{f11d}
	\end{subfigure}
 \vspace{-3mm}
	\caption{ACR results on Consistency and SmoothMix models. \textbf{(a).} ACR monotonically increases with $\eta$ in EGG. \textbf{(b).} The ACR growth gain from DSRS relative to NP shrinks with $\eta$ in EGG. \textbf{(c).} ACR stays almost constant in ESG. \textbf{(d).} The ACR growth gain from DSRS remains almost constant in ESG. For (a) and (c), solid lines represent results from DSRS, and dotted lines represent results from NP.}
	\label{demo_supp}
 \vspace{-3mm}
\end{figure}

\subsection{Full experimental results for certifications}\label{full_data}
We show full experimental results in this section, where the maximum results for certified accuracy in Table \ref{best_ESG_std}, Table \ref{best_EGG_std}, Table \ref{best_acc_consis} and Table \ref{best_acc_mix} originated from. All the base classifiers in this section are trained under the EGG distribution with $\eta=2$ (the General Gaussian distribution used in DSRS~\citep{li2022}). Our base classifiers are only affected by the dataset and the formal variance. For example, in Table \ref{full_EGG_std_CIFAR10}, all data under $\sigma=0.25$ use the same base classifier, no matter what $\eta$ is. In addition, for $\sigma=0.25$, the base classifiers for EGG and ESG are the same, which guarantees fair comparisons between distinctive distributions for certification. 
\begin{table}[htbp!] 
\centering
\caption{Full experimental results for certified accuracy, standard augmentation, EGG, CIFAR10}\label{full_EGG_std_CIFAR10}
\resizebox{1.0\linewidth}{!}{%
\begin{tabular}{ccccc cccc ccccc cccc} 
\toprule 
\multirow{2}{*}{$\sigma$} & \multirow{2}{*}{$\eta$} & Certification &\multicolumn{14}{c}{Certified accuracy at $r$}\\
\cline{4-17}
& &method &0.25  &  0.50  &  0.75  &  1.00  &  1.25  &  1.50  &  1.75  &  2.00  &  2.25  &  2.50  &  2.75  &  3.00  &  3.25  &  3.50 \\
\hline
\multirow{18}{*}{0.25}& \multirow{3}{*}{0.25} & NP &41.4\%	&0.5\%\\
&&DSRS&52.6\%	&7.9\%\\
&&(Growth)&11.2\% & 7.4\%\\
\cline{2-17}
&\multirow{3}{*}{0.5}&NP&51.6\%	&14.5\%	&0.1\%\\
&&DSRS&55.5\%	&29.5\%\\
&&(Growth)&3.9\%& 15.0\% \\
\cline{2-17}
&\multirow{3}{*}{1.0}&NP&55.6\%	&29.1\%	&5.5\% \\
&&DSRS&56.3\%	&36.5\%	&9.3\%\\
&&(Growth)&0.7\%& 7.4\%& 3.8\% \\
\cline{2-17}
&\multirow{3}{*}{2.0}&NP&56.2\%	&35.7\%	&13.4\%\\
&&DSRS&56.7\%	&38.4\%	&16.9\%\\
&&(Growth)&0.5\%& 2.7\%& 3.5\% \\
\cline{2-17}
&\multirow{3}{*}{4.0}&NP&57.3\%	&38.5\%	&18.5\%\\
&&DSRS&57.5\%	&39.3\%	&20.0\%\\
&&(Growth)&0.2\%& 0.8\%& 1.5\% \\
\cline{2-17}
&\multirow{3}{*}{8.0}&NP&57.5\%	&39.2\%	&21.2\%\\
&&DSRS&57.6\%	&40.1\%	&22.0\%\\
&&(Growth)&0.1\%& 0.9\%& 0.8\% \\
\hline
\multirow{18}{*}{0.50}& \multirow{3}{*}{0.25} & NP &50.9\%	&23.7\%	&3.9\%	&0.1\%	\\
&&DSRS&54.2\%	&37.6\%	&16.4\%	&2.0\%\\
&&(Growth)&3.3\%& 13.9\%& 12.5\%& 1.9\%\\
\cline{2-17}
&\multirow{3}{*}{0.5}&NP&53.0\%	&34.6\%	&16.6\%	&4.5\%	&0.3\%\\
&&DSRS&53.6\%	&40.4\%	&25.2\%	&12.9\%	&2.5\%\\
&&(Growth)&0.6\%& 5.8\%& 8.6\%& 8.4\%& 2.2\% \\
\cline{2-17}
&\multirow{3}{*}{1.0}&NP&53.7\%	&39.8\%	&23.1\%	&12.4\%	&3.6\%	&0.7\% \\
&&DSRS&54.1\%	&41.7\%	&28.2\%	&17.3\%	&9.2\%	&1.8\%\\
&&(Growth)&0.4\%& 1.9\%& 5.1\%& 4.9\%& 5.6\%& 1.1\% \\
\cline{2-17}
&\multirow{3}{*}{2.0}&NP&53.8\%	&41.2\%	&27.9\%	&17.0\%	&8.9\%	&2.9\%	&0.1\%\\
&&DSRS&54.0\%	&42.4\%	&29.3\%	&19.4\%	&11.6\%	&3.8\%	&0.6\%\\
&&(Growth)&0.2\%& 1.2\%& 1.4\%& 2.4\%& 2.7\%& 0.9\%& 0.5\% \\
\cline{2-17}
&\multirow{3}{*}{4.0}&NP&53.9\%	&42.2\%	&29.7\%	&19.0\%	&11.3\%	&4.7\%	&1.5\%\\
&&DSRS&54.0\%	&42.5\%	&30.0\%	&20.2\%	&12.8\%	&5.8\%	&1.5\%\\
&&(Growth)&0.1\%& 0.3\%& 0.3\%& 1.2\%& 1.5\%& 1.1\%& 0.0\%& \\
\cline{2-17}
&\multirow{3}{*}{8.0}&NP&54.0\%	&42.6\%	&30.0\%	&20.0\%	&12.7\%	&6.8\%	&2.3\%\\
&&DSRS&54.2\%	&42.5\%	&30.9\%	&20.6\%	&13.4\%	&6.8\%	&1.8\%\\
&&(Growth)&0.2\%&-0.1\%& 0.9\%& 0.6\%& 0.7\%& 0.0\%&-0.5\% \\
\hline
\multirow{18}{*}{1.00}& \multirow{3}{*}{0.25} & NP &40.6\%	&28.7\%	&18.0\%	&9.0\%	&3.1\%	&0.4\%	\\
&&DSRS&41.8\%	&32.6\%	&23.5\%	&16.5\%	&9.4\%	&4.5\%	&0.5\%	&0.1\%\\
&&(Growth)&1.2\%& 3.9\%& 5.5\%& 7.5\%& 6.3\%& 4.1\%\\
\cline{2-17}
&\multirow{3}{*}{0.5}&NP&40.9\%	&31.8\%	&21.8\%	&15.7\%	&9.1\%	&4.9\%	&1.6\%	&0.2\%	&0.1\%\\
&&DSRS&41.1\%	&33.6\%	&24.7\%	&19.1\%	&13.4\%	&8.5\%	&5.5\%	&2.0\%	&0.4\%	&0.1\%\\
&&(Growth)&0.2\%& 1.8\%& 2.9\%& 3.4\%& 4.3\%& 3.6\%& 3.9\%& 1.8\%& 0.3\%& \\
\cline{2-17}
&\multirow{3}{*}{1.0}&NP&40.4\%	&32.4\%	&24.0\%	&17.7\%	&12.6\%	&8.5\%	&5.0\%	&2.2\%	&0.5\%	&0.1\%	&0.1\% \\
&&DSRS&40.2\%	&33.2\%	&25.5\%	&20.0\%	&15.1\%	&10.5\%	&7.1\%	&4.2\%	&1.9\%	&0.9\%	&0.1\%\\
&&(Growth)&-0.2\%& 0.8\%& 1.5\%& 2.3\%& 2.5\%& 2.0\%& 2.1\%& 2.0\%& 1.4\%& 0.8\%& 0.0\% \\
\cline{2-17}
&\multirow{3}{*}{2.0}&NP&40.2\%	&32.6\%	&24.6\%	&18.9\%	&15.0\%	&10.1\%	&7.4\%	&4.0\%	&2.0\%	&0.7\%	&0.1\%	&0.1\%\\
&&DSRS&40.2\%	&32.8\%	&25.5\%	&20.2\%	&15.7\%	&11.5\%	&8.0\%	&5.5\%	&2.6\%	&1.5\%	&0.6\%	&0.1\%\\
&&(Growth)&0.0\%& 0.2\%& 0.9\%& 1.3\%& 0.7\%& 1.4\%& 0.6\%& 1.5\%& 0.6\%& 0.8\%& 0.5\%& 0.0\% \\
\cline{2-17}
&\multirow{3}{*}{4.0}&NP&40.0\%	&32.7\%	&25.2\%	&19.7\%	&15.5\%	&11.4\%	&8.2\%	&5.3\%	&3.0\%	&1.5\%	&0.6\%	&0.1\%	&0.1\%\\
&&DSRS&39.6\%	&32.7\%	&25.7\%	&20.2\%	&15.9\%	&12.2\%	&8.5\%	&6.5\%	&3.4\%	&1.8\%	&0.9\%	&0.4\%\\
&&(Growth)&-0.4\%& 0.0\%& 0.5\%& 0.5\%& 0.4\%& 0.8\%& 0.3\%& 1.2\%& 0.4\%& 0.3\%& 0.3\%& 0.3\% \\
\cline{2-17}
&\multirow{3}{*}{8.0}&NP&39.7\%	&32.5\%	&25.5\%	&20.2\%	&15.7\%	&12.1\%	&8.5\%	&6.3\%	&3.4\%	&2.0\%	&0.9\%	&0.5\%	&0.2\%\\
&&DSRS&39.5\%	&32.6\%	&25.5\%	&20.2\%	&15.8\%	&12.3\%	&8.6\%	&6.6\%	&3.7\%	&2.1\%	&1.1\%	&0.5\%	&0.2\%\\
&&(Growth)&-0.2\%& 0.1\%& 0.0\%& 0.0\%& 0.1\%& 0.2\%& 0.1\%& 0.3\%& 0.3\%& 0.1\%& 0.2\%& 0.0\%& 0.0\% \\
\bottomrule
\end{tabular}
}
\end{table}

\begin{table}[htbp!] 
\centering
\caption{Full experimental results for certified accuracy, standard augmentation, EGG, ImageNet}\label{full_EGG_std_ImageNet}
\resizebox{1.0\linewidth}{!}{%
\begin{tabular}{ccccc cccc ccccc cccc} 
\toprule
\multirow{2}{*}{$\sigma$} & \multirow{2}{*}{$\eta$} & Certification &\multicolumn{14}{c}{Certified accuracy at $r$}\\
\cline{4-17}
& &method &0.25  &  0.50  &  0.75  &  1.00  &  1.25  &  1.50  &  1.75  &  2.00  &  2.25  &  2.50  &  2.75  &  3.00  &  3.25  &  3.50 \\
\hline
\multirow{18}{*}{0.25}&\multirow{3}{*}{0.25}&NP	&9.9\%	\\
&&DSRS	&45.8\%	\\
&&(Growth)	& 35.9\%\\
\cline{2-17}
&\multirow{3}{*}{0.5}&NP	&44.6\%	&0.6\%	\\
&&DSRS	&54.9\%	&9.7\%	\\
&&(Growth)	& 10.3\%& 9.1\%\\
\cline{2-17}
&\multirow{3}{*}{1.0}&NP	&54.0\%	&21.1\%	&0.6\%	\\
&&DSRS	&57.0\%	&39.3\%	&1.4\%	\\
&&(Growth)	& 3.0\%& 18.2\%& 0.8\%\\
\cline{2-17}
&\multirow{3}{*}{2.0}&NP	&57.1\%	&41.7\%	&17.6\%	\\
&&DSRS	&58.4\%	&47.9\%	&24.1\%	\\
&&(Growth)	& 1.3\%& 6.2\%& 6.5\%\\
\cline{2-17}
&\multirow{3}{*}{4.0}&NP	&58.4\%	&48.0\%	&33.6\%	\\
&&DSRS	&58.7\%	&49.9\%	&36.2\%	\\
&&(Growth)	& 0.3\%& 1.9\%& 2.6\%\\
\cline{2-17}
&\multirow{3}{*}{8.0}&NP	&58.7\%	&50.0\%	&37.8\%	\\
&&DSRS	&59.1\%	&50.8\%	&38.7\%	\\
&&(Growth)	& 0.4\%& 0.8\%& 0.9\%\\
\cline{1-17}
\multirow{18}{*}{0.50}&\multirow{3}{*}{0.25}&NP	&49.7\%	&26.7\%	&0.1\%	\\
&&DSRS	&53.8\%	&41.4\%	&14.8\%	\\
&&(Growth)	& 4.1\%& 14.7\%& 14.7\%\\
\cline{2-17}
&\multirow{3}{*}{0.5}&NP	&52.1\%	&40.6\%	&26.4\%	&8.3\%	&0.1\%	\\
&&DSRS	&54.4\%	&46.3\%	&36.4\%	&22.5\%	&2.9\%	\\
&&(Growth)	& 2.3\%& 5.7\%& 10.0\%& 14.2\%& 2.8\%\\
\cline{2-17}
&\multirow{3}{*}{1.0}&NP	&52.8\%	&44.8\%	&35.7\%	&26.2\%	&16.1\%	&6.3\%	\\
&&DSRS	&54.2\%	&47.8\%	&39.9\%	&32.8\%	&22.9\%	&8.9\%	\\
&&(Growth)	& 1.4\%& 3.0\%& 4.2\%& 6.6\%& 6.8\%& 2.6\%\\
\cline{2-17}
&\multirow{3}{*}{2.0}&NP	&53.4\%	&47.0\%	&39.4\%	&33.3\%	&24.5\%	&17.4\%	&8.4\%	\\
&&DSRS	&53.8\%	&48.5\%	&41.5\%	&35.2\%	&28.9\%	&19.4\%	&11.3\%	\\
&&(Growth)	& 0.4\%& 1.5\%& 2.1\%& 1.9\%& 4.4\%& 2.0\%& 2.9\%\\
\cline{2-17}
&\multirow{3}{*}{4.0}&NP	&53.3\%	&47.7\%	&41.3\%	&35.2\%	&29.3\%	&21.3\%	&14.0\%	\\
&&DSRS	&53.6\%	&48.6\%	&42.6\%	&36.4\%	&31.0\%	&22.8\%	&14.4\%	\\
&&(Growth)	& 0.3\%& 0.9\%& 1.3\%& 1.2\%& 1.7\%& 1.5\%& 0.4\%\\
\cline{2-17}
&\multirow{3}{*}{8.0}&NP	&53.3\%	&48.3\%	&42.2\%	&36.5\%	&31.1\%	&24.0\%	&16.9\%	\\
&&DSRS	&53.6\%	&48.8\%	&42.9\%	&36.8\%	&31.8\%	&24.6\%	&16.5\%	\\
&&(Growth)	& 0.3\%& 0.5\%& 0.7\%& 0.3\%& 0.7\%& 0.6\%&-0.4\%\\
\cline{1-17}
\multirow{18}{*}{1.00}&\multirow{3}{*}{0.25}&NP	&38.8\%	&29.8\%	&15.7\%	&3.3\%	\\
&&DSRS	&41.0\%	&35.3\%	&28.4\%	&20.1\%	&7.1\%	&0.8\%	\\
&&(Growth)	& 2.2\%& 5.5\%& 12.7\%& 16.8\%\\
\cline{2-17}
&\multirow{3}{*}{0.5}&NP	&40.3\%	&34.1\%	&26.4\%	&19.0\%	&10.7\%	&4.1\%	&1.4\%	&0.1\%	\\
&&DSRS	&40.9\%	&37.0\%	&31.6\%	&26.3\%	&22.1\%	&15.2\%	&8.7\%	&3.1\%	&0.8\%	\\
&&(Growth)	& 0.6\%& 2.9\%& 5.2\%& 7.3\%& 11.4\%& 11.1\%& 7.3\%& 3.0\%\\
\cline{2-17}
&\multirow{3}{*}{1.0}&NP	&41.7\%	&36.4\%	&30.7\%	&25.5\%	&20.9\%	&15.3\%	&10.5\%	&6.3\%	&3.3\%	&1.7\%	&0.7\%	&0.2\%	\\
&&DSRS	&42.0\%	&37.9\%	&33.4\%	&29.3\%	&24.9\%	&22.0\%	&18.5\%	&13.1\%	&9.2\%	&5.0\%	&2.1\%	&0.5\%	\\
&&(Growth)	& 0.3\%& 1.5\%& 2.7\%& 3.8\%& 4.0\%& 6.7\%& 8.0\%& 6.8\%& 5.9\%& 3.3\%& 1.4\%& 0.3\%\\
\cline{2-17}
&\multirow{3}{*}{2.0}&NP	&42.5\%	&37.2\%	&32.8\%	&29.2\%	&24.7\%	&21.4\%	&17.4\%	&13.8\%	&10.1\%	&7.8\%	&5.5\%	&3.3\%	&2.2\%	&1.1\%	\\
&&DSRS	&42.9\%	&38.2\%	&34.3\%	&30.2\%	&26.8\%	&23.3\%	&21.3\%	&18.8\%	&14.1\%	&11.1\%	&8.9\%	&6.1\%	&2.2\%	&1.4\%	\\
&&(Growth)	& 0.4\%& 1.0\%& 1.5\%& 1.0\%& 2.1\%& 1.9\%& 3.9\%& 5.0\%& 4.0\%& 3.3\%& 3.4\%& 2.8\%& 0.0\%& 0.3\%\\
\cline{2-17}
&\multirow{3}{*}{4.0}&NP	&42.2\%	&38.0\%	&33.8\%	&30.8\%	&26.1\%	&23.3\%	&21.1\%	&18.2\%	&13.9\%	&11.2\%	&9.3\%	&8.1\%	&6.0\%	&4.2\%	\\
&&DSRS	&42.5\%	&38.5\%	&34.6\%	&31.3\%	&27.5\%	&23.9\%	&22.3\%	&20.2\%	&17.3\%	&13.2\%	&10.7\%	&9.2\%	&6.8\%	&4.0\%	\\
&&(Growth)	& 0.3\%& 0.5\%& 0.8\%& 0.5\%& 1.4\%& 0.6\%& 1.2\%& 2.0\%& 3.4\%& 2.0\%& 1.4\%& 1.1\%& 0.8\%&-0.2\%\\
\cline{2-17}
&\multirow{3}{*}{8.0}&NP	&42.4\%	&38.1\%	&34.5\%	&31.1\%	&26.9\%	&24.0\%	&22.2\%	&19.9\%	&16.8\%	&12.8\%	&11.0\%	&9.5\%	&8.0\%	&6.1\%	\\
&&DSRS	&42.4\%	&38.4\%	&35.0\%	&31.4\%	&28.0\%	&24.4\%	&22.6\%	&20.7\%	&18.9\%	&14.5\%	&11.7\%	&10.1\%	&8.6\%	&5.2\%	\\
&&(Growth)	& 0.0\%& 0.3\%& 0.5\%& 0.3\%& 1.1\%& 0.4\%& 0.4\%& 0.8\%& 2.1\%& 1.7\%& 0.7\%& 0.6\%& 0.6\%&-0.9\%\\
\bottomrule
\end{tabular}
}
\end{table}

\begin{table}[htbp!] 
\centering
\caption{Full experimental results for certified accuracy, standard augmentation, ESG, CIFAR10}\label{full_ESG_std_CIFAR10}
\resizebox{1.0\linewidth}{!}{%
\begin{tabular}{ccccc cccc ccccc cccc} 
\toprule 
\multirow{2}{*}{$\sigma$} & \multirow{2}{*}{$\eta$} & Certification &\multicolumn{14}{c}{Certified accuracy at $r$}\\
\cline{4-17}
& &method &0.25  &  0.50  &  0.75  &  1.00  &  1.25  &  1.50  &  1.75  &  2.00  &  2.25  &  2.50  &  2.75  &  3.00  &  3.25  &  3.50 \\
\hline
\multirow{12}{*}{0.25}&\multirow{3}{*}{1.0}&NP	&57.8\%	&40.7\%	&25.6\%	\\
&&DSRS	&57.6\%	&40.7\%	&25.1\%	\\
&&(Growth)	&-0.2\%& 0.0\%&-0.5\%\\
\cline{2-17}
&\multirow{3}{*}{2.0}&NP	&57.8\%	&40.8\%	&25.7\%	\\
&&DSRS	&57.6\%	&40.6\%	&25.1\%	\\
&&(Growth)	&-0.2\%&-0.2\%&-0.6\%\\
\cline{2-17}
&\multirow{3}{*}{4.0}&NP	&57.9\%	&40.9\%	&26.0\%	\\
&&DSRS	&57.6\%	&40.6\%	&25.0\%	\\
&&(Growth)	&-0.3\%&-0.3\%&-1.0\%\\
\cline{2-17}
&\multirow{3}{*}{8.0}&NP	&57.9\%	&40.8\%	&25.6\%	\\
&&DSRS	&57.8\%	&40.6\%	&24.9\%	\\
&&(Growth)	&-0.1\%&-0.2\%&-0.7\%\\
\cline{1-17}
\multirow{12}{*}{0.50}&\multirow{3}{*}{1.0}&NP	&54.3\%	&42.7\%	&31.7\%	&21.8\%	&14.0\%	&8.7\%	&3.5\%	\\
&&DSRS	&54.2\%	&42.6\%	&31.3\%	&21.5\%	&14.1\%	&8.3\%	&2.7\%	\\
&&(Growth)	&-0.1\%&-0.1\%&-0.4\%&-0.3\%& 0.1\%&-0.4\%&-0.8\%\\
\cline{2-17}
&\multirow{3}{*}{2.0}&NP	&54.3\%	&42.6\%	&31.6\%	&21.7\%	&14.0\%	&8.5\%	&3.6\%	\\
&&DSRS	&54.3\%	&42.6\%	&31.6\%	&21.5\%	&13.9\%	&8.7\%	&2.8\%	\\
&&(Growth)	& 0.0\%& 0.0\%& 0.0\%&-0.2\%&-0.1\%& 0.2\%&-0.8\%\\
\cline{2-17}
&\multirow{3}{*}{4.0}&NP	&54.3\%	&42.7\%	&31.5\%	&21.6\%	&14.3\%	&8.6\%	&3.6\%	\\
&&DSRS	&54.3\%	&42.6\%	&31.3\%	&21.5\%	&13.9\%	&8.2\%	&3.0\%	\\
&&(Growth)	& 0.0\%&-0.1\%&-0.2\%&-0.1\%&-0.4\%&-0.4\%&-0.6\%\\
\cline{2-17}
&\multirow{3}{*}{8.0}&NP	&54.3\%	&42.6\%	&31.7\%	&21.7\%	&14.1\%	&8.7\%	&3.5\%	\\
&&DSRS	&54.4\%	&42.6\%	&31.6\%	&21.6\%	&14.0\%	&8.1\%	&2.4\%	\\
&&(Growth)	& 0.1\%& 0.0\%&-0.1\%&-0.1\%&-0.1\%&-0.6\%&-1.1\%\\
\cline{1-17}
\multirow{12}{*}{1.00}&\multirow{3}{*}{1.0}&NP	&39.6\%	&32.6\%	&26.0\%	&20.5\%	&15.9\%	&13.0\%	&9.2\%	&7.0\%	&4.5\%	&2.5\%	&1.5\%	&0.9\%	&0.5\%	&0.2\%	\\
&&DSRS	&39.6\%	&32.5\%	&25.7\%	&20.4\%	&15.8\%	&12.8\%	&8.6\%	&6.8\%	&4.3\%	&2.3\%	&1.3\%	&0.8\%	&0.3\%	&0.1\%	\\
&&(Growth)	& 0.0\%&-0.1\%&-0.3\%&-0.1\%&-0.1\%&-0.2\%&-0.6\%&-0.2\%&-0.2\%&-0.2\%&-0.2\%&-0.1\%&-0.2\%&-0.1\%\\
\cline{2-17}
&\multirow{3}{*}{2.0}&NP	&39.6\%	&32.6\%	&25.9\%	&20.4\%	&15.9\%	&13.0\%	&9.2\%	&7.0\%	&4.6\%	&2.5\%	&1.3\%	&0.8\%	&0.3\%	&0.2\%	\\
&&DSRS	&39.5\%	&32.6\%	&25.8\%	&20.4\%	&15.8\%	&12.7\%	&8.8\%	&6.8\%	&4.5\%	&2.4\%	&1.3\%	&0.7\%	&0.2\%	&0.2\%	\\
&&(Growth)	&-0.1\%& 0.0\%&-0.1\%& 0.0\%&-0.1\%&-0.3\%&-0.4\%&-0.2\%&-0.1\%&-0.1\%& 0.0\%&-0.1\%&-0.1\%& 0.0\%\\
\cline{2-17}
&\multirow{3}{*}{4.0}&NP	&39.7\%	&32.6\%	&25.9\%	&20.4\%	&15.9\%	&12.9\%	&9.0\%	&7.0\%	&4.6\%	&2.5\%	&1.5\%	&0.8\%	&0.3\%	&0.2\%	\\
&&DSRS	&39.5\%	&32.6\%	&25.8\%	&20.3\%	&15.9\%	&12.9\%	&8.6\%	&6.9\%	&4.3\%	&2.4\%	&1.3\%	&0.8\%	&0.2\%	&0.1\%	\\
&&(Growth)	&-0.2\%& 0.0\%&-0.1\%&-0.1\%& 0.0\%& 0.0\%&-0.4\%&-0.1\%&-0.3\%&-0.1\%&-0.2\%& 0.0\%&-0.1\%&-0.1\%\\
\cline{2-17}
&\multirow{3}{*}{8.0}&NP	&39.6\%	&32.6\%	&25.8\%	&20.3\%	&15.9\%	&12.9\%	&9.2\%	&7.0\%	&4.5\%	&2.5\%	&1.3\%	&0.8\%	&0.3\%	&0.2\%	\\
&&DSRS	&39.7\%	&32.4\%	&25.8\%	&20.4\%	&15.9\%	&12.9\%	&8.9\%	&6.7\%	&4.2\%	&2.4\%	&1.3\%	&0.9\%	&0.2\%	&0.1\%	\\
&&(Growth)	& 0.1\%&-0.2\%& 0.0\%& 0.1\%& 0.0\%& 0.0\%&-0.3\%&-0.3\%&-0.3\%&-0.1\%& 0.0\%& 0.1\%&-0.1\%&-0.1\%\\

\bottomrule 
\end{tabular}
}
\end{table}

\begin{table}[htbp!] 
\centering
\caption{Full experimental results for certified accuracy, standard augmentation, ESG, ImageNet}\label{full_ESG_std_ImageNet}
\resizebox{1.0\linewidth}{!}{%
\begin{tabular}{ccccc cccc ccccc cccc} 
\toprule 
\multirow{2}{*}{$\sigma$} & \multirow{2}{*}{$\eta$} & Certification &\multicolumn{14}{c}{Certified accuracy at $r$}\\
\cline{4-17}
& &method &0.25  &  0.50  &  0.75  &  1.00  &  1.25  &  1.50  &  1.75  &  2.00  &  2.25  &  2.50  &  2.75  &  3.00  &  3.25  &  3.50 \\
\hline
\multirow{12}{*}{0.25}&\multirow{3}{*}{1.0}&NP	&59.6\%	&51.6\%	&42.2\%	\\
&&DSRS	&59.6\%	&51.5\%	&41.6\%	\\
&&(Growth)	& 0.0\%&-0.1\%&-0.6\%\\
\cline{2-17}
&\multirow{3}{*}{2.0}&NP	&59.6\%	&51.7\%	&41.9\%	\\
&&DSRS	&59.6\%	&51.6\%	&41.8\%	\\
&&(Growth)	& 0.0\%&-0.1\%&-0.1\%\\
\cline{2-17}
&\multirow{3}{*}{4.0}&NP	&59.6\%	&51.7\%	&42.0\%	\\
&&DSRS	&59.6\%	&51.5\%	&41.9\%	\\
&&(Growth)	& 0.0\%&-0.2\%&-0.1\%\\
\cline{2-17}
&\multirow{3}{*}{8.0}&NP	&59.6\%	&51.6\%	&42.2\%	\\
&&DSRS	&59.6\%	&51.5\%	&41.5\%	\\
&&(Growth)	& 0.0\%&-0.1\%&-0.7\%\\
\cline{1-17}
\multirow{12}{*}{0.50}&\multirow{3}{*}{1.0}&NP	&53.6\%	&49.3\%	&43.2\%	&38.2\%	&33.0\%	&27.2\%	&21.0\%	\\
&&DSRS	&53.6\%	&49.1\%	&43.2\%	&37.9\%	&33.0\%	&26.8\%	&19.1\%	\\
&&(Growth)	& 0.0\%&-0.2\%& 0.0\%&-0.3\%& 0.0\%&-0.4\%&-1.9\%\\
\cline{2-17}
&\multirow{3}{*}{2.0}&NP	&53.7\%	&49.2\%	&43.2\%	&38.1\%	&33.1\%	&27.3\%	&20.7\%	\\
&&DSRS	&53.6\%	&49.2\%	&43.1\%	&38.0\%	&32.9\%	&26.9\%	&19.2\%	\\
&&(Growth)	&-0.1\%& 0.0\%&-0.1\%&-0.1\%&-0.2\%&-0.4\%&-1.5\%\\
\cline{2-17}
&\multirow{3}{*}{4.0}&NP	&53.6\%	&49.2\%	&43.2\%	&38.1\%	&33.0\%	&27.5\%	&20.8\%	\\
&&DSRS	&53.6\%	&49.2\%	&43.2\%	&38.0\%	&32.9\%	&27.2\%	&19.4\%	\\
&&(Growth)	& 0.0\%& 0.0\%& 0.0\%&-0.1\%&-0.1\%&-0.3\%&-1.4\%\\
\cline{2-17}
&\multirow{3}{*}{8.0}&NP	&53.7\%	&49.2\%	&43.2\%	&38.1\%	&33.1\%	&27.4\%	&21.0\%	\\
&&DSRS	&53.6\%	&49.1\%	&43.2\%	&38.0\%	&33.0\%	&26.8\%	&19.4\%	\\
&&(Growth)	&-0.1\%&-0.1\%& 0.0\%&-0.1\%&-0.1\%&-0.6\%&-1.6\%\\
\cline{1-17}
\multirow{12}{*}{1.00}&\multirow{3}{*}{1.0}&NP	&42.7\%	&39.1\%	&35.4\%	&32.0\%	&29.5\%	&25.3\%	&23.1\%	&21.6\%	&20.0\%	&17.6\%	&13.9\%	&11.7\%	&10.5\%	&9.1\%	\\
&&DSRS	&42.6\%	&38.8\%	&35.3\%	&31.9\%	&28.9\%	&25.3\%	&23.1\%	&21.5\%	&19.9\%	&17.4\%	&13.8\%	&11.5\%	&10.3\%	&7.7\%	\\
&&(Growth)	&-0.1\%&-0.3\%&-0.1\%&-0.1\%&-0.6\%& 0.0\%& 0.0\%&-0.1\%&-0.1\%&-0.2\%&-0.1\%&-0.2\%&-0.2\%&-1.4\%\\
\cline{2-17}
&\multirow{3}{*}{2.0}&NP	&42.6\%	&39.1\%	&35.3\%	&32.1\%	&29.2\%	&25.3\%	&23.2\%	&21.6\%	&20.0\%	&17.6\%	&14.1\%	&11.7\%	&10.5\%	&9.1\%	\\
&&DSRS	&42.5\%	&39.0\%	&35.2\%	&31.7\%	&29.0\%	&25.2\%	&23.1\%	&21.5\%	&19.7\%	&17.4\%	&13.6\%	&11.4\%	&10.1\%	&8.3\%	\\
&&(Growth)	&-0.1\%&-0.1\%&-0.1\%&-0.4\%&-0.2\%&-0.1\%&-0.1\%&-0.1\%&-0.3\%&-0.2\%&-0.5\%&-0.3\%&-0.4\%&-0.8\%\\
\cline{2-17}
&\multirow{3}{*}{4.0}&NP	&42.6\%	&39.1\%	&35.4\%	&32.0\%	&29.1\%	&25.3\%	&23.1\%	&21.6\%	&19.9\%	&17.9\%	&14.5\%	&11.7\%	&10.5\%	&9.1\%	\\
&&DSRS	&42.5\%	&38.8\%	&35.2\%	&31.9\%	&29.3\%	&25.3\%	&23.1\%	&21.6\%	&19.9\%	&17.2\%	&13.6\%	&11.4\%	&10.2\%	&8.0\%	\\
&&(Growth)	&-0.1\%&-0.3\%&-0.2\%&-0.1\%& 0.2\%& 0.0\%& 0.0\%& 0.0\%& 0.0\%&-0.7\%&-0.9\%&-0.3\%&-0.3\%&-1.1\%\\
\cline{2-17}
&\multirow{3}{*}{8.0}&NP	&42.7\%	&39.0\%	&35.3\%	&31.9\%	&29.4\%	&25.3\%	&23.1\%	&21.7\%	&20.0\%	&17.5\%	&13.9\%	&11.5\%	&10.4\%	&9.3\%	\\
&&DSRS	&42.5\%	&38.8\%	&35.3\%	&32.0\%	&29.0\%	&25.2\%	&23.1\%	&21.6\%	&19.7\%	&17.3\%	&13.6\%	&11.5\%	&10.1\%	&8.4\%	\\
&&(Growth)	&-0.2\%&-0.2\%& 0.0\%& 0.1\%&-0.4\%&-0.1\%& 0.0\%&-0.1\%&-0.3\%&-0.2\%&-0.3\%& 0.0\%&-0.3\%&-0.9\%\\

\bottomrule 
\end{tabular}
}
\end{table}

\begin{table}[htbp!] 
\centering
\caption{Full experimental results for certified accuracy, Consistency, EGG, CIFAR10}\label{full_EGG_consis_CIFAR10}
\resizebox{1.0\linewidth}{!}{%
\begin{tabular}{ccccc cccc ccccc cccc} 
\toprule 
\multirow{2}{*}{$\sigma$} & \multirow{2}{*}{$\eta$} & Certification &\multicolumn{14}{c}{Certified accuracy at $r$}\\
\cline{4-17}
& &method &0.25  &  0.50  &  0.75  &  1.00  &  1.25  &  1.50  &  1.75  &  2.00  &  2.25  &  2.50  &  2.75  &  3.00  &  3.25  &  3.50 \\
\hline
\multirow{12}{*}{0.25}&\multirow{3}{*}{1.0}&NP	&61.4\%	&47.0\%	&25.2\%	\\
&&DSRS	&62.1\%	&50.7\%	&30.1\%	\\
&&(Growth)	& 0.7\%& 3.7\%& 4.9\%\\
\cline{2-17}
&\multirow{3}{*}{2.0}&NP	&61.8\%	&50.7\%	&35.1\%	\\
&&DSRS	&62.5\%	&52.0\%	&37.2\%	\\
&&(Growth)	& 0.7\%& 1.3\%& 2.1\%\\
\cline{2-17}
&\multirow{3}{*}{4.0}&NP	&62.3\%	&51.7\%	&38.2\%	\\
&&DSRS	&62.5\%	&52.2\%	&39.1\%	\\
&&(Growth)	& 0.2\%& 0.5\%& 0.9\%\\
\cline{2-17}
&\multirow{3}{*}{8.0}&NP	&62.5\%	&52.2\%	&40.2\%	\\
&&DSRS	&62.5\%	&52.6\%	&40.4\%	\\
&&(Growth)	& 0.0\%& 0.4\%& 0.2\%\\
\cline{1-17}
\multirow{12}{*}{0.50}&\multirow{3}{*}{1.0}&NP	&49.2\%	&43.1\%	&36.3\%	&28.7\%	&19.6\%	&11.6\%	\\
&&DSRS	&49.5\%	&43.7\%	&38.2\%	&33.1\%	&24.2\%	&15.4\%	\\
&&(Growth)	& 0.3\%& 0.6\%& 1.9\%& 4.4\%& 4.6\%& 3.8\%\\
\cline{2-17}
&\multirow{3}{*}{2.0}&NP	&49.3\%	&43.8\%	&37.8\%	&32.3\%	&23.5\%	&18.0\%	&9.6\%	\\
&&DSRS	&49.4\%	&44.1\%	&38.5\%	&34.4\%	&27.4\%	&19.6\%	&11.5\%	\\
&&(Growth)	& 0.1\%& 0.3\%& 0.7\%& 2.1\%& 3.9\%& 1.6\%& 1.9\%\\
\cline{2-17}
&\multirow{3}{*}{4.0}&NP	&49.3\%	&44.0\%	&38.3\%	&34.1\%	&27.4\%	&20.4\%	&14.6\%	\\
&&DSRS	&49.3\%	&44.0\%	&38.7\%	&35.4\%	&28.3\%	&21.1\%	&14.1\%	\\
&&(Growth)	& 0.0\%& 0.0\%& 0.4\%& 1.3\%& 0.9\%& 0.7\%&-0.5\%\\
\cline{2-17}
&\multirow{3}{*}{8.0}&NP	&49.3\%	&44.1\%	&38.7\%	&35.0\%	&28.4\%	&21.9\%	&15.8\%	\\
&&DSRS	&49.3\%	&44.1\%	&38.8\%	&35.3\%	&28.6\%	&22.3\%	&15.2\%	\\
&&(Growth)	& 0.0\%& 0.0\%& 0.1\%& 0.3\%& 0.2\%& 0.4\%&-0.6\%\\
\cline{1-17}
\multirow{12}{*}{1.00}&\multirow{3}{*}{1.0}&NP	&37.2\%	&32.5\%	&29.4\%	&25.2\%	&21.8\%	&17.6\%	&14.4\%	&11.7\%	&8.8\%	&6.3\%	&4.5\%	&2.6\%	&1.8\%	\\
&&DSRS	&37.4\%	&33.0\%	&29.5\%	&26.2\%	&22.8\%	&19.2\%	&16.7\%	&14.3\%	&11.3\%	&9.2\%	&6.6\%	&4.1\%	&1.4\%	\\
&&(Growth)	& 0.2\%& 0.5\%& 0.1\%& 1.0\%& 1.0\%& 1.6\%& 2.3\%& 2.6\%& 2.5\%& 2.9\%& 2.1\%& 1.5\%&-0.4\%\\
\cline{2-17}
&\multirow{3}{*}{2.0}&NP	&37.2\%	&32.6\%	&29.7\%	&25.9\%	&22.4\%	&19.0\%	&16.3\%	&13.9\%	&11.3\%	&8.9\%	&7.2\%	&5.1\%	&3.5\%	&2.2\%	\\
&&DSRS	&37.1\%	&32.3\%	&29.8\%	&26.5\%	&23.0\%	&20.6\%	&17.0\%	&14.7\%	&12.7\%	&10.5\%	&8.5\%	&6.3\%	&3.9\%	&2.5\%	\\
&&(Growth)	&-0.1\%&-0.3\%& 0.1\%& 0.6\%& 0.6\%& 1.6\%& 0.7\%& 0.8\%& 1.4\%& 1.6\%& 1.3\%& 1.2\%& 0.4\%& 0.3\%\\
\cline{2-17}
&\multirow{3}{*}{4.0}&NP	&37.1\%	&32.5\%	&29.8\%	&26.2\%	&22.7\%	&20.3\%	&16.9\%	&14.9\%	&12.4\%	&10.6\%	&8.7\%	&6.7\%	&5.1\%	&3.2\%	\\
&&DSRS	&37.0\%	&32.4\%	&29.8\%	&26.7\%	&23.1\%	&20.9\%	&17.5\%	&15.3\%	&13.0\%	&10.9\%	&9.2\%	&7.0\%	&5.2\%	&3.1\%	\\
&&(Growth)	&-0.1\%&-0.1\%& 0.0\%& 0.5\%& 0.4\%& 0.6\%& 0.6\%& 0.4\%& 0.6\%& 0.3\%& 0.5\%& 0.3\%& 0.1\%&-0.1\%\\
\cline{2-17}
&\multirow{3}{*}{8.0}&NP	&37.1\%	&32.5\%	&29.9\%	&26.4\%	&23.0\%	&20.7\%	&17.2\%	&15.2\%	&13.2\%	&10.9\%	&9.6\%	&7.5\%	&5.9\%	&4.2\%	\\
&&DSRS	&36.7\%	&32.5\%	&29.8\%	&26.6\%	&23.2\%	&20.9\%	&17.6\%	&15.5\%	&13.2\%	&11.3\%	&9.6\%	&7.8\%	&5.8\%	&3.9\%	\\
&&(Growth)	&-0.4\%& 0.0\%&-0.1\%& 0.2\%& 0.2\%& 0.2\%& 0.4\%& 0.3\%& 0.0\%& 0.4\%& 0.0\%& 0.3\%&-0.1\%&-0.3\%\\

\bottomrule 
\end{tabular}
}
\end{table}

\begin{table}[htbp!] 
\centering
\caption{Full experimental results for certified accuracy, Consistency, ESG, CIFAR10}\label{full_ESG_consis_CIFAR10}
\resizebox{1.0\linewidth}{!}{%
\begin{tabular}{ccccc cccc ccccc cccc}
\toprule 
\multirow{2}{*}{$\sigma$} & \multirow{2}{*}{$\eta$} & Certification &\multicolumn{14}{c}{Certified accuracy at $r$}\\
\cline{4-17}
& &method &0.25  &  0.50  &  0.75  &  1.00  &  1.25  &  1.50  &  1.75  &  2.00  &  2.25  &  2.50  &  2.75  &  3.00  &  3.25  &  3.50 \\
\hline
\multirow{12}{*}{0.25}&\multirow{3}{*}{1.0}&NP	&62.7\%	&52.9\%	&41.8\%	\\
&&DSRS	&62.6\%	&52.9\%	&41.6\%	\\
&&(Growth)	&-0.1\%& 0.0\%&-0.2\%\\
\cline{2-17}
&\multirow{3}{*}{2.0}&NP	&62.7\%	&53.0\%	&42.1\%	\\
&&DSRS	&62.7\%	&53.0\%	&41.4\%	\\
&&(Growth)	& 0.0\%& 0.0\%&-0.7\%\\
\cline{2-17}
&\multirow{3}{*}{4.0}&NP	&62.7\%	&53.0\%	&42.0\%	\\
&&DSRS	&62.7\%	&52.9\%	&41.6\%	\\
&&(Growth)	& 0.0\%&-0.1\%&-0.4\%\\
\cline{2-17}
&\multirow{3}{*}{8.0}&NP	&62.7\%	&53.0\%	&42.0\%	\\
&&DSRS	&62.7\%	&52.9\%	&41.7\%	\\
&&(Growth)	& 0.0\%&-0.1\%&-0.3\%\\
\cline{1-17}
\multirow{12}{*}{0.50}&\multirow{3}{*}{1.0}&NP	&49.3\%	&44.1\%	&39.2\%	&35.5\%	&29.7\%	&24.3\%	&18.7\%	\\
&&DSRS	&49.3\%	&44.1\%	&38.9\%	&35.5\%	&29.3\%	&23.7\%	&16.9\%	\\
&&(Growth)	& 0.0\%& 0.0\%&-0.3\%& 0.0\%&-0.4\%&-0.6\%&-1.8\%\\
\cline{2-17}
&\multirow{3}{*}{2.0}&NP	&49.3\%	&44.1\%	&39.2\%	&35.5\%	&29.7\%	&24.1\%	&18.7\%	\\
&&DSRS	&49.3\%	&44.1\%	&38.9\%	&35.5\%	&29.4\%	&24.0\%	&17.1\%	\\
&&(Growth)	& 0.0\%& 0.0\%&-0.3\%& 0.0\%&-0.3\%&-0.1\%&-1.6\%\\
\cline{2-17}
&\multirow{3}{*}{4.0}&NP	&49.3\%	&44.1\%	&39.1\%	&35.5\%	&29.9\%	&24.1\%	&18.6\%	\\
&&DSRS	&49.3\%	&44.0\%	&39.0\%	&35.5\%	&29.5\%	&23.8\%	&17.6\%	\\
&&(Growth)	& 0.0\%&-0.1\%&-0.1\%& 0.0\%&-0.4\%&-0.3\%&-1.0\%\\
\cline{2-17}
&\multirow{3}{*}{8.0}&NP	&49.3\%	&44.1\%	&39.3\%	&35.5\%	&29.6\%	&24.2\%	&19.0\%	\\
&&DSRS	&49.3\%	&44.1\%	&38.9\%	&35.5\%	&29.1\%	&23.8\%	&17.3\%	\\
&&(Growth)	& 0.0\%& 0.0\%&-0.4\%& 0.0\%&-0.5\%&-0.4\%&-1.7\%\\
\cline{1-17}
\multirow{12}{*}{1.00}&\multirow{3}{*}{1.0}&NP	&36.7\%	&32.6\%	&30.0\%	&26.9\%	&23.4\%	&21.0\%	&18.0\%	&16.0\%	&14.0\%	&12.2\%	&10.4\%	&8.8\%	&7.0\%	&5.4\%	\\
&&DSRS	&36.5\%	&32.3\%	&29.8\%	&27.0\%	&23.2\%	&20.9\%	&17.7\%	&15.8\%	&13.7\%	&11.8\%	&10.0\%	&8.8\%	&6.8\%	&4.6\%	\\
&&(Growth)	&-0.2\%&-0.3\%&-0.2\%& 0.1\%&-0.2\%&-0.1\%&-0.3\%&-0.2\%&-0.3\%&-0.4\%&-0.4\%& 0.0\%&-0.2\%&-0.8\%\\
\cline{2-17}
&\multirow{3}{*}{2.0}&NP	&36.8\%	&32.5\%	&29.9\%	&27.1\%	&23.3\%	&21.0\%	&18.1\%	&16.0\%	&14.0\%	&12.2\%	&10.3\%	&8.8\%	&7.1\%	&5.5\%	\\
&&DSRS	&36.6\%	&32.3\%	&29.9\%	&26.8\%	&23.3\%	&21.0\%	&17.7\%	&16.0\%	&13.8\%	&11.9\%	&10.1\%	&8.7\%	&6.7\%	&4.4\%	\\
&&(Growth)	&-0.2\%&-0.2\%& 0.0\%&-0.3\%& 0.0\%& 0.0\%&-0.4\%& 0.0\%&-0.2\%&-0.3\%&-0.2\%&-0.1\%&-0.4\%&-1.1\%\\
\cline{2-17}
&\multirow{3}{*}{4.0}&NP	&36.6\%	&32.5\%	&29.9\%	&27.0\%	&23.2\%	&21.0\%	&18.0\%	&16.1\%	&14.0\%	&12.2\%	&10.3\%	&9.0\%	&7.2\%	&5.3\%	\\
&&DSRS	&36.6\%	&32.3\%	&29.8\%	&26.9\%	&23.3\%	&20.9\%	&17.8\%	&15.7\%	&13.9\%	&11.9\%	&10.2\%	&8.8\%	&6.7\%	&4.8\%	\\
&&(Growth)	& 0.0\%&-0.2\%&-0.1\%&-0.1\%& 0.1\%&-0.1\%&-0.2\%&-0.4\%&-0.1\%&-0.3\%&-0.1\%&-0.2\%&-0.5\%&-0.5\%\\
\cline{2-17}
&\multirow{3}{*}{8.0}&NP	&36.7\%	&32.6\%	&29.8\%	&27.0\%	&23.3\%	&21.0\%	&17.9\%	&16.1\%	&13.8\%	&12.2\%	&10.3\%	&9.0\%	&7.1\%	&5.1\%	\\
&&DSRS	&36.6\%	&32.1\%	&29.8\%	&26.9\%	&23.2\%	&20.9\%	&17.6\%	&15.8\%	&13.6\%	&11.8\%	&10.1\%	&8.3\%	&6.7\%	&4.8\%	\\
&&(Growth)	&-0.1\%&-0.5\%& 0.0\%&-0.1\%&-0.1\%&-0.1\%&-0.3\%&-0.3\%&-0.2\%&-0.4\%&-0.2\%&-0.7\%&-0.4\%&-0.3\%\\

\bottomrule 
\end{tabular}
}
\end{table}

\begin{table}[htbp!] 
\centering
\caption{Full experimental results for certified accuracy, SmoothMix, EGG, CIFAR10}\label{full_EGG_mix_CIFAR10}
\resizebox{1.0\linewidth}{!}{%
\begin{tabular}{ccccc cccc ccccc cccc} 
\toprule
\multirow{2}{*}{$\sigma$} & \multirow{2}{*}{$\eta$} & Certification &\multicolumn{14}{c}{Certified accuracy at $r$}\\
\cline{4-17}
& &method &0.25  &  0.50  &  0.75  &  1.00  &  1.25  &  1.50  &  1.75  &  2.00  &  2.25  &  2.50  &  2.75  &  3.00  &  3.25  &  3.50 \\
\hline
\multirow{12}{*}{0.25}&\multirow{3}{*}{1.0}&NP	&63.3\%	&49.8\%	&27.4\%	\\
&&DSRS	&63.7\%	&53.8\%	&32.4\%	\\
&&(Growth)	& 0.4\%& 4.0\%& 5.0\%\\
\cline{2-17}
&\multirow{3}{*}{2.0}&NP	&63.8\%	&53.3\%	&38.1\%	\\
&&DSRS	&64.5\%	&55.0\%	&40.8\%	\\
&&(Growth)	& 0.7\%& 1.7\%& 2.7\%\\
\cline{2-17}
&\multirow{3}{*}{4.0}&NP	&64.2\%	&55.1\%	&42.0\%	\\
&&DSRS	&64.4\%	&55.5\%	&43.0\%	\\
&&(Growth)	& 0.2\%& 0.4\%& 1.0\%\\
\cline{2-17}
&\multirow{3}{*}{8.0}&NP	&64.6\%	&55.5\%	&43.7\%	\\
&&DSRS	&64.7\%	&55.7\%	&43.9\%	\\
&&(Growth)	& 0.1\%& 0.2\%& 0.2\%\\
\cline{1-17}
\multirow{12}{*}{0.50}&\multirow{3}{*}{1.0}&NP	&53.0\%	&46.7\%	&38.8\%	&30.1\%	&22.2\%	&12.7\%	\\
&&DSRS	&53.3\%	&47.7\%	&40.9\%	&34.3\%	&26.6\%	&16.5\%	\\
&&(Growth)	& 0.3\%& 1.0\%& 2.1\%& 4.2\%& 4.4\%& 3.8\%\\
\cline{2-17}
&\multirow{3}{*}{2.0}&NP	&53.2\%	&47.5\%	&40.3\%	&34.0\%	&26.6\%	&19.4\%	&9.6\%	\\
&&DSRS	&53.3\%	&48.1\%	&41.7\%	&35.6\%	&28.9\%	&21.2\%	&11.4\%	\\
&&(Growth)	& 0.1\%& 0.6\%& 1.4\%& 1.6\%& 2.3\%& 1.8\%& 1.8\%\\
\cline{2-17}
&\multirow{3}{*}{4.0}&NP	&53.3\%	&48.0\%	&41.3\%	&35.1\%	&29.0\%	&22.6\%	&15.5\%	\\
&&DSRS	&53.3\%	&48.1\%	&41.9\%	&35.9\%	&29.5\%	&23.4\%	&14.3\%	\\
&&(Growth)	& 0.0\%& 0.1\%& 0.6\%& 0.8\%& 0.5\%& 0.8\%&-1.2\%\\
\cline{2-17}
&\multirow{3}{*}{8.0}&NP	&53.3\%	&48.2\%	&41.6\%	&35.8\%	&29.5\%	&23.9\%	&17.5\%	\\
&&DSRS	&53.3\%	&48.3\%	&42.0\%	&36.2\%	&30.1\%	&24.3\%	&15.6\%	\\
&&(Growth)	& 0.0\%& 0.1\%& 0.4\%& 0.4\%& 0.6\%& 0.4\%&-1.9\%\\
\cline{1-17}
\multirow{12}{*}{1.00}&\multirow{3}{*}{1.0}&NP	&43.3\%	&39.3\%	&33.2\%	&27.9\%	&22.7\%	&18.2\%	&15.1\%	&10.9\%	&7.4\%	&3.5\%	&1.6\%	&0.8\%	&0.1\%	\\
&&DSRS	&43.5\%	&39.6\%	&34.5\%	&29.1\%	&24.8\%	&21.1\%	&17.0\%	&14.2\%	&10.5\%	&7.7\%	&4.0\%	&1.5\%	&0.1\%	\\
&&(Growth)	& 0.2\%& 0.3\%& 1.3\%& 1.2\%& 2.1\%& 2.9\%& 1.9\%& 3.3\%& 3.1\%& 4.2\%& 2.4\%& 0.7\%& 0.0\%\\
\cline{2-17}
&\multirow{3}{*}{2.0}&NP	&43.2\%	&39.4\%	&33.9\%	&29.2\%	&24.1\%	&20.5\%	&16.8\%	&14.1\%	&10.4\%	&7.9\%	&4.8\%	&2.0\%	&1.3\%	&0.3\%	\\
&&DSRS	&43.1\%	&39.6\%	&34.3\%	&29.3\%	&24.8\%	&21.3\%	&18.0\%	&15.2\%	&12.3\%	&9.7\%	&6.4\%	&3.7\%	&1.4\%	&0.4\%	\\
&&(Growth)	&-0.1\%& 0.2\%& 0.4\%& 0.1\%& 0.7\%& 0.8\%& 1.2\%& 1.1\%& 1.9\%& 1.8\%& 1.6\%& 1.7\%& 0.1\%& 0.1\%\\
\cline{2-17}
&\multirow{3}{*}{4.0}&NP	&43.1\%	&39.5\%	&34.3\%	&29.6\%	&24.5\%	&21.3\%	&17.7\%	&15.0\%	&12.2\%	&9.6\%	&6.7\%	&4.0\%	&2.1\%	&1.3\%	\\
&&DSRS	&42.8\%	&39.6\%	&34.4\%	&29.4\%	&25.2\%	&21.8\%	&18.3\%	&15.8\%	&12.8\%	&10.2\%	&7.7\%	&4.6\%	&2.1\%	&0.9\%	\\
&&(Growth)	&-0.3\%& 0.1\%& 0.1\%&-0.2\%& 0.7\%& 0.5\%& 0.6\%& 0.8\%& 0.6\%& 0.6\%& 1.0\%& 0.6\%& 0.0\%&-0.4\%\\
\cline{2-17}
&\multirow{3}{*}{8.0}&NP	&43.1\%	&39.5\%	&34.1\%	&29.8\%	&24.9\%	&21.7\%	&18.4\%	&15.6\%	&12.9\%	&10.1\%	&8.0\%	&5.6\%	&3.2\%	&1.6\%	\\
&&DSRS	&42.9\%	&39.5\%	&34.3\%	&29.8\%	&25.2\%	&21.9\%	&18.6\%	&15.8\%	&13.2\%	&10.6\%	&8.0\%	&5.4\%	&2.7\%	&1.3\%	\\
&&(Growth)	&-0.2\%& 0.0\%& 0.2\%& 0.0\%& 0.3\%& 0.2\%& 0.2\%& 0.2\%& 0.3\%& 0.5\%& 0.0\%&-0.2\%&-0.5\%&-0.3\%\\

\bottomrule 
\end{tabular}
}
\end{table}

\begin{table}[htbp!] 
\centering
\caption{Full experimental results for certified accuracy, SmoothMix, ESG, CIFAR10}\label{full_ESG_mix_CIFAR10}
\resizebox{1.0\linewidth}{!}{%
\begin{tabular}{ccccc cccc ccccc cccc} 
\toprule 
\multirow{2}{*}{$\sigma$} & \multirow{2}{*}{$\eta$} & Certification &\multicolumn{14}{c}{Certified accuracy at $r$}\\
\cline{4-17}
& &method &0.25  &  0.50  &  0.75  &  1.00  &  1.25  &  1.50  &  1.75  &  2.00  &  2.25  &  2.50  &  2.75  &  3.00  &  3.25  &  3.50 \\
\hline
\multirow{12}{*}{0.25}&\multirow{3}{*}{1.0}&NP	&64.8\%	&56.5\%	&46.7\%	\\
&&DSRS	&64.6\%	&56.5\%	&46.5\%	\\
&&(Growth)	&-0.2\%& 0.0\%&-0.2\%\\
\cline{2-17}
&\multirow{3}{*}{2.0}&NP	&64.8\%	&56.5\%	&46.7\%	\\
&&DSRS	&64.7\%	&56.3\%	&45.8\%	\\
&&(Growth)	&-0.1\%&-0.2\%&-0.9\%\\
\cline{2-17}
&\multirow{3}{*}{4.0}&NP	&64.8\%	&56.4\%	&46.8\%	\\
&&DSRS	&64.6\%	&56.3\%	&46.0\%	\\
&&(Growth)	&-0.2\%&-0.1\%&-0.8\%\\
\cline{2-17}
&\multirow{3}{*}{8.0}&NP	&64.8\%	&56.5\%	&46.9\%	\\
&&DSRS	&64.6\%	&56.1\%	&46.5\%	\\
&&(Growth)	&-0.2\%&-0.4\%&-0.4\%\\
\cline{1-17}
\multirow{12}{*}{0.50}&\multirow{3}{*}{1.0}&NP	&53.3\%	&48.3\%	&42.1\%	&36.7\%	&31.6\%	&26.2\%	&20.1\%	\\
&&DSRS	&53.3\%	&48.3\%	&42.0\%	&36.6\%	&31.3\%	&25.6\%	&18.1\%	\\
&&(Growth)	& 0.0\%& 0.0\%&-0.1\%&-0.1\%&-0.3\%&-0.6\%&-2.0\%\\
\cline{2-17}
&\multirow{3}{*}{2.0}&NP	&53.2\%	&48.3\%	&42.2\%	&36.7\%	&31.7\%	&26.0\%	&20.4\%	\\
&&DSRS	&53.3\%	&48.3\%	&42.0\%	&36.6\%	&31.2\%	&25.6\%	&17.6\%	\\
&&(Growth)	& 0.1\%& 0.0\%&-0.2\%&-0.1\%&-0.5\%&-0.4\%&-2.8\%\\
\cline{2-17}
&\multirow{3}{*}{4.0}&NP	&53.3\%	&48.3\%	&42.1\%	&36.6\%	&31.6\%	&26.1\%	&20.3\%	\\
&&DSRS	&53.1\%	&48.3\%	&42.0\%	&36.5\%	&31.0\%	&25.6\%	&18.3\%	\\
&&(Growth)	&-0.2\%& 0.0\%&-0.1\%&-0.1\%&-0.6\%&-0.5\%&-2.0\%\\
\cline{2-17}
&\multirow{3}{*}{8.0}&NP	&53.3\%	&48.3\%	&42.0\%	&36.5\%	&31.6\%	&26.3\%	&20.2\%	\\
&&DSRS	&53.3\%	&48.3\%	&42.0\%	&36.4\%	&31.1\%	&25.5\%	&18.5\%	\\
&&(Growth)	& 0.0\%& 0.0\%& 0.0\%&-0.1\%&-0.5\%&-0.8\%&-1.7\%\\
\cline{1-17}
\multirow{12}{*}{1.00}&\multirow{3}{*}{1.0}&NP	&43.1\%	&39.6\%	&34.4\%	&30.2\%	&25.4\%	&22.2\%	&19.0\%	&16.3\%	&13.7\%	&11.6\%	&9.1\%	&6.7\%	&4.9\%	&2.3\%	\\
&&DSRS	&42.8\%	&39.5\%	&34.4\%	&29.9\%	&25.3\%	&22.0\%	&19.1\%	&16.2\%	&13.4\%	&11.4\%	&8.9\%	&6.3\%	&4.0\%	&1.8\%	\\
&&(Growth)	&-0.3\%&-0.1\%& 0.0\%&-0.3\%&-0.1\%&-0.2\%& 0.1\%&-0.1\%&-0.3\%&-0.2\%&-0.2\%&-0.4\%&-0.9\%&-0.5\%\\
\cline{2-17}
&\multirow{3}{*}{2.0}&NP	&43.0\%	&39.7\%	&34.4\%	&30.0\%	&25.5\%	&22.5\%	&19.0\%	&16.2\%	&13.6\%	&11.5\%	&9.3\%	&6.5\%	&4.6\%	&2.1\%	\\
&&DSRS	&42.9\%	&39.4\%	&34.3\%	&29.8\%	&25.3\%	&22.1\%	&18.9\%	&16.3\%	&13.4\%	&11.6\%	&8.8\%	&6.2\%	&4.0\%	&1.7\%	\\
&&(Growth)	&-0.1\%&-0.3\%&-0.1\%&-0.2\%&-0.2\%&-0.4\%&-0.1\%& 0.1\%&-0.2\%& 0.1\%&-0.5\%&-0.3\%&-0.6\%&-0.4\%\\
\cline{2-17}
&\multirow{3}{*}{4.0}&NP	&43.0\%	&39.7\%	&34.6\%	&30.1\%	&25.5\%	&22.4\%	&19.0\%	&16.3\%	&13.7\%	&11.6\%	&9.4\%	&6.6\%	&4.8\%	&2.4\%	\\
&&DSRS	&42.9\%	&39.5\%	&34.3\%	&29.8\%	&25.4\%	&22.1\%	&19.0\%	&16.2\%	&13.6\%	&11.5\%	&9.2\%	&6.0\%	&3.9\%	&1.8\%	\\
&&(Growth)	&-0.1\%&-0.2\%&-0.3\%&-0.3\%&-0.1\%&-0.3\%& 0.0\%&-0.1\%&-0.1\%&-0.1\%&-0.2\%&-0.6\%&-0.9\%&-0.6\%\\
\cline{2-17}
&\multirow{3}{*}{8.0}&NP	&43.0\%	&39.7\%	&34.3\%	&30.0\%	&25.5\%	&22.3\%	&19.1\%	&16.2\%	&13.5\%	&11.5\%	&9.2\%	&6.5\%	&4.7\%	&2.5\%	\\
&&DSRS	&42.7\%	&39.4\%	&34.3\%	&29.8\%	&25.3\%	&22.3\%	&18.9\%	&16.3\%	&13.5\%	&11.3\%	&9.0\%	&6.2\%	&4.1\%	&1.6\%	\\
&&(Growth)	&-0.3\%&-0.3\%& 0.0\%&-0.2\%&-0.2\%& 0.0\%&-0.2\%& 0.1\%& 0.0\%&-0.2\%&-0.2\%&-0.3\%&-0.6\%&-0.9\%\\

\bottomrule 
\end{tabular}
}
\end{table}

\clearpage
\subsection{Supplemental figures for StdAug-GGS models}\label{appfigstdggs}
\begin{figure}[htbp!]
  \centering
  \includegraphics[height=7in, width=5.5in]{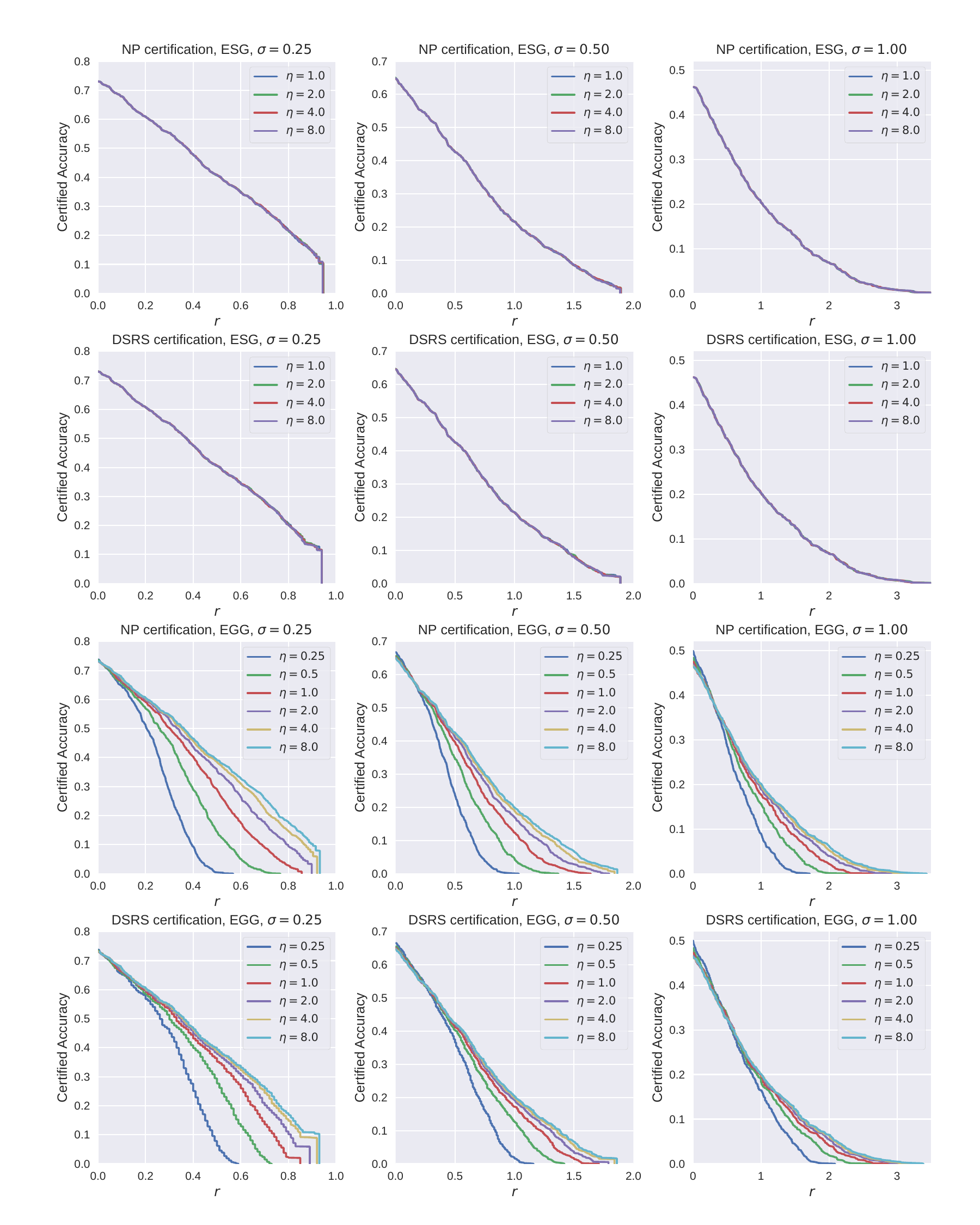}
  \caption{Certified accuracy of $\ell_2$ for standardly augmented models, on CIFAR-10 by General Gaussian, $k=1530$.}
 \label{figcohenCIFAR}
\end{figure}

\clearpage
\begin{figure}
  \centering
  \includegraphics[height=7in, width=5.5in]{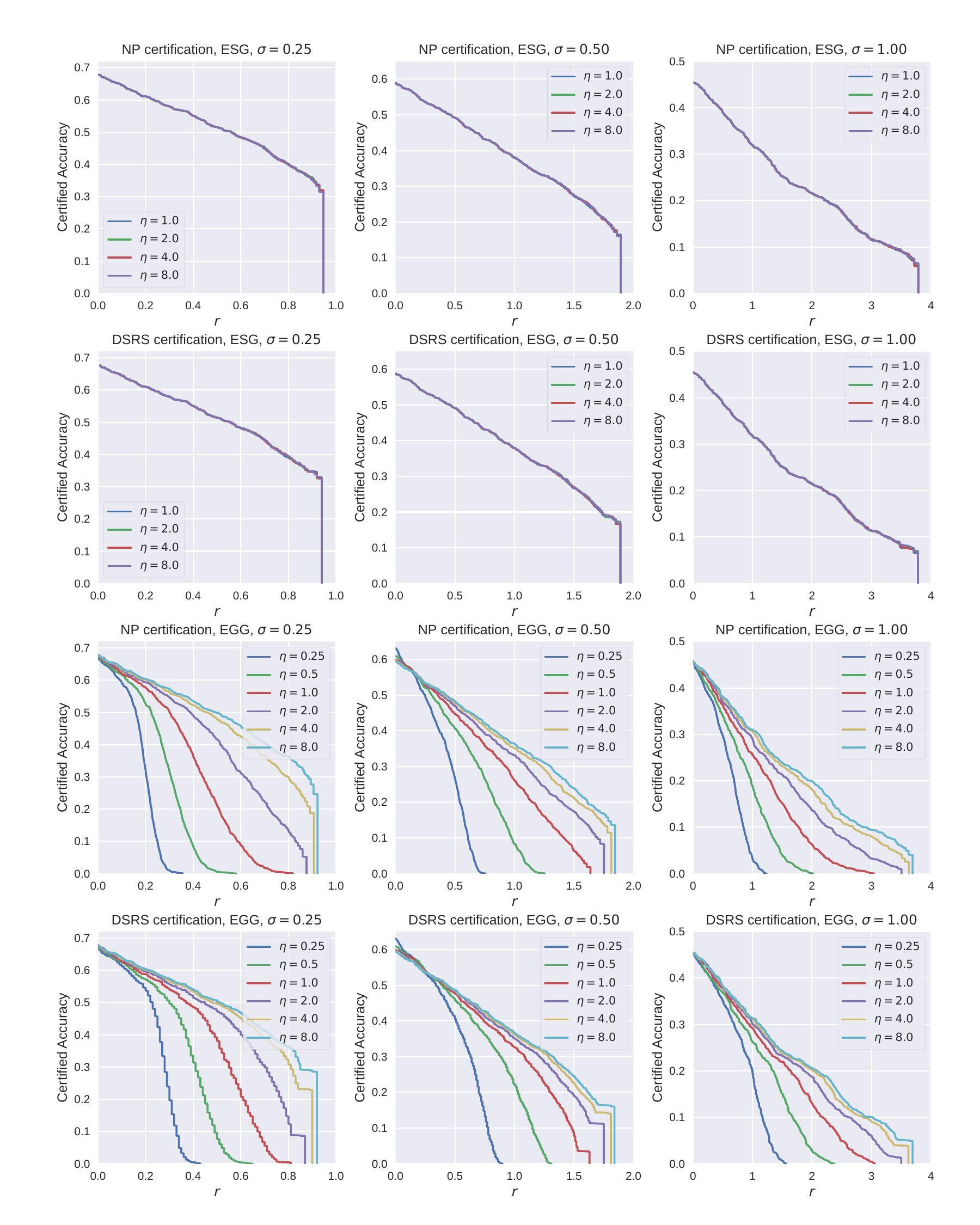}
  \caption{Certified accuracy of $\ell_2$ for standardly augmented models, on ImageNet by General Gaussian, $k=75260$.}
  \label{figcohenimgnet}
\end{figure}

\clearpage
\subsection{Supplemental figures for Consistency-GGS and SmoothMix-GGS models}\label{appfigcs}
\begin{figure}[htbp!]
 \centering
  \includegraphics[height=7in, width=5.5in]{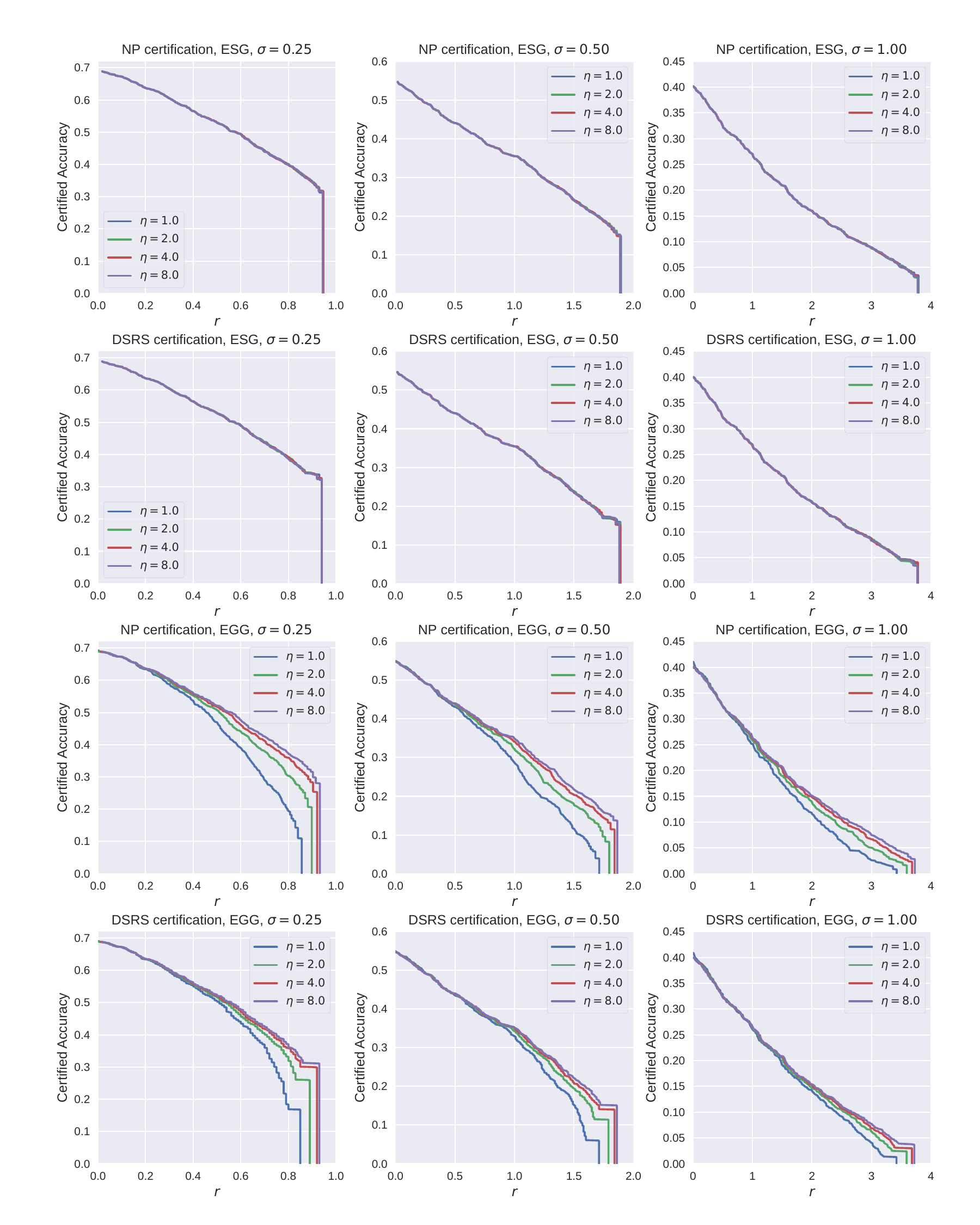}
  \caption{Certified accuracy of $\ell_2$ for Consistency models, augmented on CIFAR-10 by General Gaussian, $k=1530$.}
  \label{figconsis}
\end{figure}

\begin{figure}
  \centering
  \includegraphics[height=7in, width=5.5in]{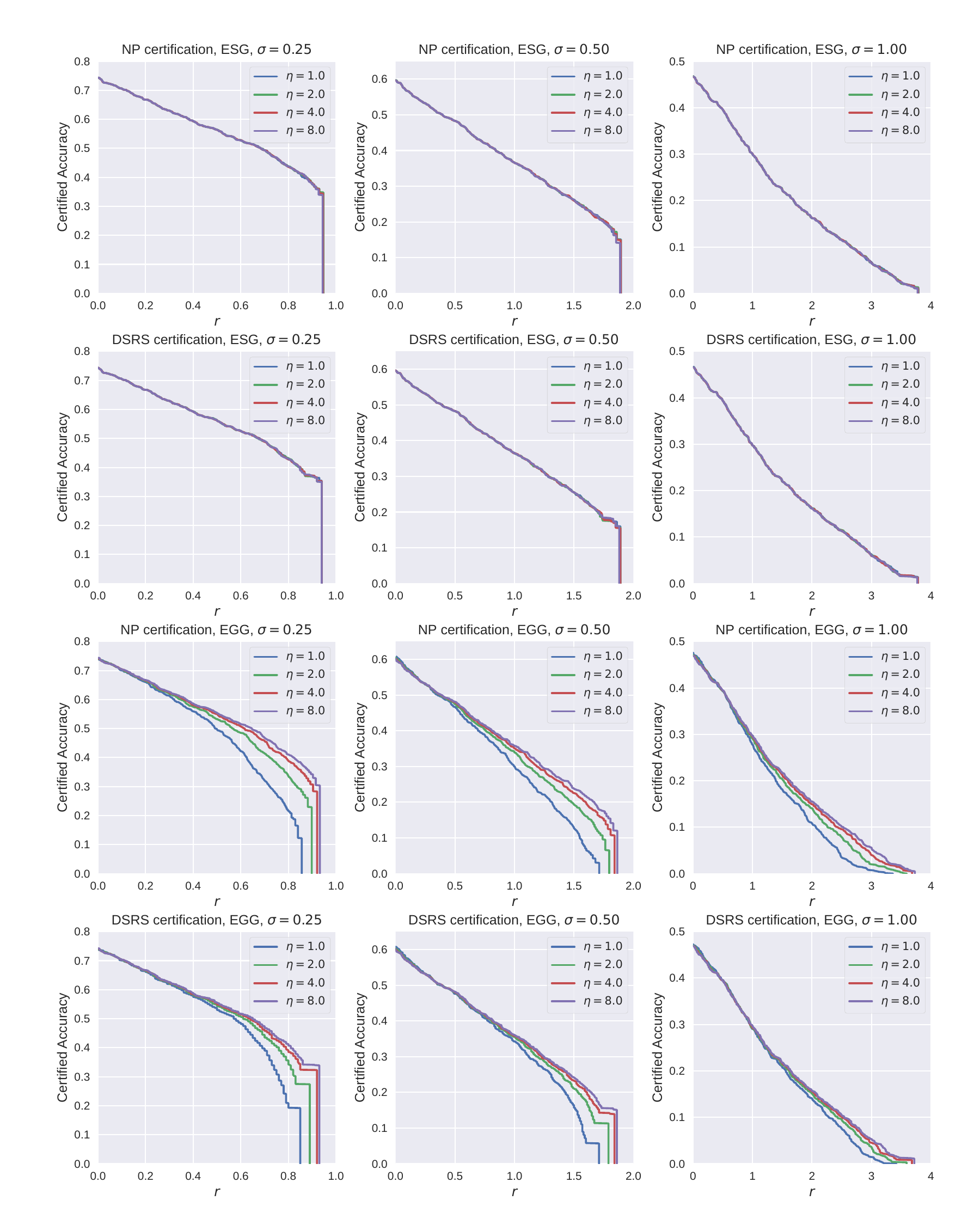}
  \caption{Certified accuracy of $\ell_2$ for SmoothMix models, augmented on CIFAR-10 by General Gaussian, $k=1530$.}
 \label{figmix}
\end{figure}

\end{document}